\documentclass{article}

\PassOptionsToPackage{numbers, compress}{natbib}
\usepackage[final]{include/neurips_2021}

\usepackage[utf8]{inputenc} %
\usepackage[T1]{fontenc}    %
\usepackage{include/titletoc}
\usepackage{hyperref}       %
\usepackage{url}            %
\usepackage{booktabs}       %
\usepackage{amsfonts}       %
\usepackage{nicefrac}       %
\usepackage{microtype}      %
\usepackage{xcolor}         %

\usepackage[title]{appendix}

\usepackage{microtype}
\usepackage{graphicx}
\graphicspath{ {../figures/} }
\usepackage{booktabs} %

\usepackage{amsmath}
\usepackage{amsthm}
\usepackage{amssymb}
\usepackage{bbm}
\usepackage{bm}
\usepackage{verbatim}
\usepackage{float}
\usepackage{color,soul}
\usepackage{enumitem}
\usepackage{mathtools}
\usepackage{hhline}
\usepackage[nameinlink]{cleveref} %
\usepackage[font=small,labelfont=bf,
   justification=justified,
   format=plain]{caption}
\usepackage{subcaption}
\usepackage{tabularx}
\usepackage{tikz}
\usepackage{wrapfig,booktabs}
\usepackage{xspace}
\usepackage{cancel}
\usepackage{sidecap}

\usetikzlibrary{shapes.geometric, arrows, bayesnet, calc, positioning}

\newtheorem{definition}{Definition}

\newtheorem{remark}{Remark}
\newtheorem{theorem}{Theorem}
\newtheorem{proposition}{Proposition}
\newtheorem{lemma}{Lemma}
\newtheorem{corollary}{Corollary}

\newtheorem{innercustomtheorem}{Theorem}
\newenvironment{customtheorem}[2]
  {\renewcommand\theinnercustomtheorem{#1 \textnormal{#2}}\innercustomtheorem}
  {\endinnercustomtheorem}

\newtheorem{innercustomproposition}{Proposition}
\newenvironment{customproposition}[2]
  {\renewcommand\theinnercustomproposition{#1 \textnormal{#2}}\innercustomproposition}
  {\endinnercustomproposition}

\newtheorem{innercustomcorollary}{Corollary}
\newenvironment{customcorollary}[2]
  {\renewcommand\theinnercustomcorollary{#1 \textnormal{#2}}\innercustomcorollary}
  {\endinnercustomcorollary}

\DeclarePairedDelimiterX{\infdivx}[2]{(}{)}{%
  #1\;\delimsize\|\;#2%
}

\definecolor{darkgreen}{rgb}{0.0, 0.7, 0.0}

\newcommand{\kld}{KL divergence\xspace}

\crefname{appsec}{appendix}{appendices}
\Crefname{appsec}{Appendix}{Appendices}

\definecolor{mydarkblue}{rgb}{0,0.08,0.45}

\newcommand{\ba}{\mathbf{a}}
\newcommand{\bA}{\mathbf{A}}
\newcommand{\bg}{\mathbf{g}}

\newcommand{\bs}{\mathbf s}
\newcommand{\bS}{\mathbf S}

\newcommand{\calA}{\mathcal{A}}

\newcommand{\calD}{\mathcal{D}}

\newcommand{\calF}{\mathcal{F}}

\newcommand{\calH}{\mathcal{H}}

\newcommand{\calJ}{\mathcal{J}}

\newcommand{\calQ}{\mathcal{Q}}
\newcommand{\calR}{\mathcal{R}}
\newcommand{\calS}{\mathcal{S}}
\newcommand{\calT}{\mathcal{T}}

\newcommand{\R}{\mathbb{R}}

\newcommand{\btau}{\bm{\tau}}

\newcommand{\closer}[3]{{\kern-#1ex{#2}\kern-#3ex}}

\newcommand{\DKL}{\mathbb{D}_{\text{KL}}\infdivx}

\DeclareMathOperator*{\argmin}{arg\,min}
\DeclareMathOperator*{\argmax}{arg\,max}

\mathchardef\mhyphen="2D

\DeclareMathOperator{\E}{\mathbb{E}}

\usepackage{algorithm}
\usepackage{algorithmic}

\newcommand{\vbar}{\,|\,}

\newcommand{\rewardt}[1]{r(\bs_{#1}, \ba_{#1}, \bg; q_{\Delta})}

\newcommand{\dyn}{p_d}
\newcommand{\transitiont}[1]{\dyn(\bs_{#1+1} \vbar \bs_{#1} , \ba_{#1} )}
\newcommand{\goaltransitiont}[1]{\dyn(\bg \vbar \bs_{#1} , \ba_{#1} )}

\newcommand{\policyt}[1]{\pi(\ba_{#1} \vbar \bs_{#1})}

\newcommand{\likelihoodt}[1]{
    \dyn(\bg \vbar \bs_{#1}, \ba_{#1})
}

\newcommand{\policytdot}[1]{
    \pi( \cdot \vbar \bs_{#1})
}
\newcommand{\policypriort}[1]{
    p( \ba_{#1} \vbar \bs_{#1})
}
\newcommand{\policypriortdot}[1]{
    p( \cdot \vbar \bs_{#1})
}

\newcommand{\qtrajnosodot}{
    q_{\Trajnoso_{0:T} | \bS_{0}}( \cdot \vbar \bs_{0})
}

\newcommand{\qtrajnosotvar}{
    q_{\Trajnoso_{0:T} | T, \bS_{0}}( \trajnoso_{0:t} \vbar t, \bs_{0})
}
\newcommand{\qtrajnosotdot}{
    q_{\Trajnoso_{0:T} | T, \bS_{0}}( \cdot \vbar t, \bs_{0})
}

\newcommand{\ptrajnosotvar}{
    p_{\Trajnoso_{0:T} | T, \bS_{0}}( \trajnoso_{0:t} \vbar t, \bs_{0})
}
\newcommand{\ptrajnosotdot}{
    p_{\Trajnoso_{0:T} | T, \bS_{0}}( \cdot \vbar t, \bs_{0})
}

\newcommand{\qtvar}{
    q_{T}( t )
}
\newcommand{\qtdot}{
    q_{T}
}

\newcommand{\ptvar}{
    p_{T}( t )
}
\newcommand{\ptdot}{
    p_{T}
}

\newcommand{\qdeltatdot}[1]{
    q_{\Delta_{#1}}
}

\newcommand{\qdeltatzero}[1]{
    q_{\Delta_{#1}}( \Delta_{#1} = 0 )
}
\newcommand{\qdeltatone}[1]{
    q_{\Delta_{#1}}( \Delta_{#1} = 1 )
}

\newcommand{\pdeltatdot}[1]{
    p_{\Delta_{#1}}
}

\newcommand{\pdeltatzero}[1]{
    p_{\Delta_{#1}}( \Delta_{#1} = 0 )
}
\newcommand{\pdeltatone}[1]{
    p_{\Delta_{#1}}( \Delta_{#1} = 1 )
}

\newcommand{\odac}{\textsc{odac}\xspace}
\newcommand{\sigmoid}{\sigma}

\newcommand{\Trajnoso}{\tilde{\bTau}}
\newcommand{\trajnoso}{\tilde{\btau}}

\newcommand{\pq}{\phi}
\newcommand{\pqtarget}{{\bar{\pq}}}
\newcommand{\pdyn}{\psi}
\newcommand{\ppi}{\theta}

\newcommand{\bTau}{\bm{\mathcal{T}}}

\newcommand{\Qpgc}{Q}

\newcommand\Bstrut{\rule[-0.9ex]{0pt}{0pt}}   %

\newcommand\defines{\,\,\dot{=}\,\,}

\hypersetup{
    colorlinks,
    citecolor=mydarkblue,
    linkcolor=mydarkblue,
    urlcolor=mydarkblue,
    }

\title{
    Outcome-Driven Reinforcement Learning via Variational Inference
}

\author{%
    \hspace*{8pt}Tim G. J. Rudner\thanks{Equal contribution. $^\dagger$\,Corresponding authors: \href{mailto:tim.rudner@cs.ox.ac.uk}{\texttt{tim.rudner@cs.ox.ac.uk}} and \href{mailto:tim.rudner@cs.ox.ac.uk}{\texttt{vitchyr@berkeley.edu}}.}~~$^\dagger$ \\
    University of Oxford \\
    \And
    \hspace*{4pt}Vitchyr H. Pong$^\ast$$^\dagger$ \\
    University of California, Berkeley \\
    \AND
    Rowan McAllister \\
    University of California, Berkeley \\
    \And
    Yarin Gal \\
    University of Oxford \\
    \And
    Sergey Levine \\
    University of California, Berkeley \\
}

\begin{document}

\maketitle

\begin{abstract}
While reinforcement learning algorithms provide automated acquisition of optimal policies, practical application of such methods requires a number of design decisions, such as manually designing reward functions that not only define the task, but also provide sufficient shaping to accomplish it.
In this paper, we view reinforcement learning as inferring policies that achieve desired outcomes, rather than as a problem of maximizing rewards.
To solve this inference problem, we establish a novel variational inference formulation that allows us to derive a well-shaped reward function which can be learned directly from environment interactions.
From the corresponding variational objective, we also derive a new probabilistic Bellman backup operator and use it to develop an off-policy algorithm to solve goal-directed tasks.
We empirically demonstrate that this method eliminates the need to hand-craft reward functions for a suite of diverse manipulation and locomotion tasks and leads to effective goal-directed behaviors.
\end{abstract}

\section{Introduction}
\label{sec:introduction}

Reinforcement learning (RL) provides an appealing formalism for automated learning of behavioral skills, but requires considerable care and manual design to use in practice.
One particularly delicate decision is the design of the reward function, which has a significant impact on the resulting policy but is largely heuristic in practice, often lacks theoretical grounding, can make effective learning difficult, and may lead to reward mis-specification.

To avoid these shortcomings, we propose to circumvent the process of manually specifying a reward function altogether:
Instead of framing the reinforcement learning problem as finding a policy that maximizes a heuristically-defined reward function, we express it probabilistically, as inferring a state--action trajectory distribution conditioned on a desired future outcome.
By building off of prior work on probabilistic perspectives on RL~\citep{Fellows2019Virel,kappen2012pgm,rawlik2013soc,toussaint2009soc,toussaint2006probabilistic,ziebart2008maxent} and goal-directed RL in particular~\citep{attias2003planning,choi2021variational,hoffman2009expectation,toussaint06probabilisticinference}, we derive a tractable variational objective, an temporal-difference algorithm that provides a shaping-like effect for effective learning, as well as a reward function that captures the semantics of the underlying decision problem and facilitates effective learning.

We demonstrate that unlike prior works that proposed inference methods for finding policies that achieve desired outcomes~\citep{attias2003planning,fu2018vice,hoffman2009expectation,toussaint06probabilisticinference}, 
the resulting algorithm, Outcome-Driven Actor--Critic (\odac), is amenable to off-policy learning and applicable to complex, high-dimensional continuous control tasks over finite and infinite horizons.
The resulting variational algorithm can be interpreted as an automatic shaping method, where each iteration learns a reward function that automatically provides dense rewards, as we visualize in \Cref{fig:heatmaps}.
In tabular settings, \odac is guaranteed to converge to an optimal policy, and in non-tabular settings with linear Gaussian transition dynamics, the derived optimization objective is convex in the policy, facilitating easier learning.
In high-dimensional and non-linear domains, our method can be combined with deep neural network function approximators to yield a deep reinforcement learning method that does not require manual specification of rewards, and leads to good performance on a range of benchmark tasks.

\paragraph{Contributions.}
The core contributions of this paper are the probabilistic formulation of a general framework for inferring policies that lead to desired outcomes and the derivation of a variational objective from which we obtain a novel outcome-driven Bellman backup operator.
We show that this Bellman backup operator induces a shaping-like effect which ensures a clear and dense learning signal even in the early stages of training.
Crucially, unlike heuristic approaches for incorporating shaping often used in standard RL, this ``shaping'' emerges automatically from variational inference.
We demonstrate that the resulting variational objective is a lower bound on the log-marginal likelihood of achieving the outcome given an initial state and that it leads to an off-policy temporal-difference algorithm.
We evaluate this algorithm---Outcome-Driven Variational Inference (\odac)---on a range of reinforcement learning tasks without having to manually specify task-specific reward functions.
In our experiments, we find that our method results in significantly faster learning across a variety of robot manipulation and locomotion tasks than alternative approaches.

\begin{figure}[t!]
\vspace*{-2pt}
~\,
\centering
\begin{subfigure}[t]{0.2\columnwidth}
    \includegraphics[trim={0 -35pt 0 0},clip, width=\columnwidth]{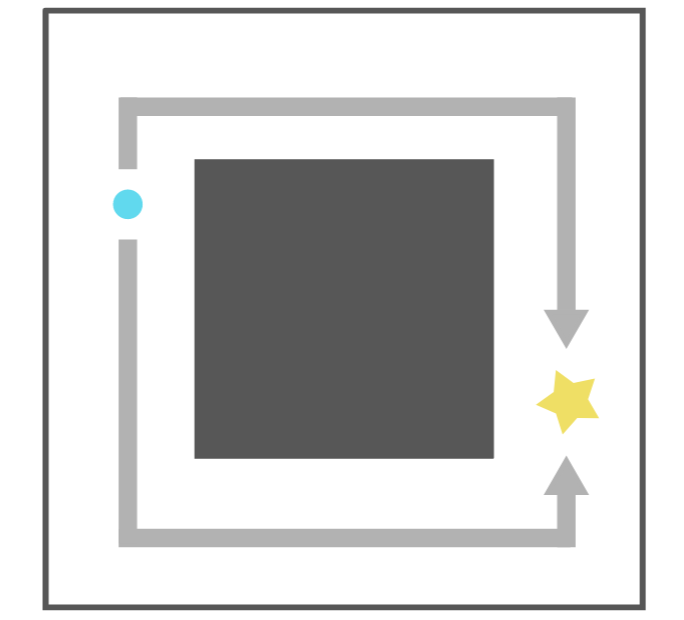}
\end{subfigure}
\,
\begin{subfigure}[t]{0.22\columnwidth}
    \includegraphics[width=\columnwidth]{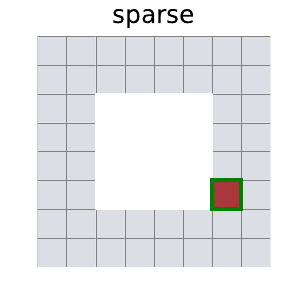}
\end{subfigure}
\begin{subfigure}[t]{0.22\columnwidth}
    \includegraphics[width=\columnwidth]{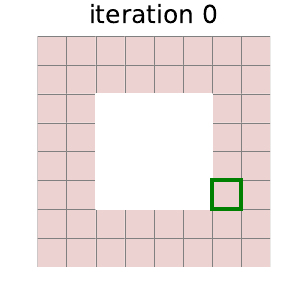}
\end{subfigure}
\begin{subfigure}[t]{0.22\columnwidth}
    \includegraphics[width=\columnwidth]{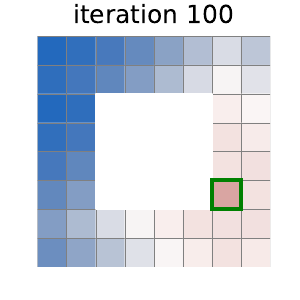}
\end{subfigure}
\vspace*{-8pt}
\caption{
    Illustration of the shaping effect of the reward function derived from the goal-directed variational inference objective.
    \textbf{Left:} A 2-dimensional grid world with a desired outcome marked by a star.
    \textbf{Center-left} The corresponding sparse reward function provides little shaping.
    \textbf{Center-right:} The reward function derived from our variational inference formulation at initialization.
    \textbf{Right:} The derived reward function after training.
    We see that the derived reward learns to provide a dense reward signal everywhere in the state space.
    }
\label{fig:heatmaps}
\vspace{-12pt}
\end{figure}

\section{Preliminaries}
\label{sec:preliminaries}

Standard reinforcement learning (RL) addresses reward maximization in a Markov decision process (MDP) defined by the tuple $( \mathcal{S}, \mathcal{A}, p_{\bS_{0}}, \dyn, r, \gamma )$~\citep{sutton1998rl,szepesvari2010}, where $\mathcal{S}$ and $\calA$ denote the state and action space, respectively, $p_{0}$ denotes the initial state distribution, $\dyn$ is a state transition distribution, $r$ is an immediate reward function, and $\gamma$ is a discount factor.
To sample trajectories, an initial state is sampled according to $p_{\bS_{0}}$, and successive states are sampled from the state transition distribution \mbox{$\bS_{t+1} \sim \dyn(\cdot \vbar \bs_{t}, \ba_{t})$} and actions from a policy \mbox{$\bA_t \sim \pi(\cdot \vbar \bs_t)$}.
We will write \mbox{$\bTau_{0:t} = \{ \bS_0, \bA_0, \bS_1, \dots, \bS_t, \bA_t \}$} to represent a finite-horizon and \mbox{$\bTau_{0} \defines \{\bA_{t}, \bS_{t+1} \}_{t=0}^\infty$} to represent an infinite-horizon stochastic state--action trajectory, and write \mbox{$\btau_{0:t} = \{ \bs_0, \ba_0, \bs_1, \dots, \bs_t, \ba_t \}$} and \mbox{$\btau_{0} \defines \{\ba_{t}, \bs_{t+1} \}_{t=0}^\infty$} for the respective trajectory realizations.
Given a reward function $r:\calS \times \calA \to \mathbb{R}$ and discount factor $\gamma \in (0,1)$, the objective in reinforcement learning is to find a policy $\pi$ that maximizes the returns, defined as
$
    \E_{p_\pi} \left[
    \sum_{t=0}^\infty \gamma^t r(\bs_t, \ba_t)
    \right],
$
where $p_{\pi}$ denotes the distribution of states induced by a policy $\pi$.

\paragraph{Goal-Conditioned Reinforcement Learning.}
In goal-conditioned reinforcement learning~\citep{kaelbling1993goals}, which can be considered a special case of the broader class of stochastic boundary value problems~\citep{aly74boundaryvalue,goebel90boundaryvalue}, the objective is for an agent to reach some pre-specified goal state, $\bg \in \calS$, so that the policy and reward function introduced above become dependent on the goal and are expressed as $\pi(\ba \vbar \bs, \bg)$ and $r(\bs, \ba, \bg)$, respectively.
Typically, such a reward function needs to be defined manually, with a common choice being to use a sparse indicator reward \mbox{$r(\bs, \bg) = \mathbb{I} \{ \bs = \bg \}$}.
However, this approach presents a number of challenges both in theory and in practice.
From a theoretical perspective, the indicator reward will almost surely equal zero for environments with continuous goal spaces and non-trivial stochastic dynamics.
From a practical perspective, such sparse rewards can be slow to learn from, as most transitions provide no reward supervision, while manually designing dense reward functions that provide a better learning signal is time-consuming and often based on heuristics.
In~\Cref{sec:theory}, we will present a framework that addresses these practical and theoretical considerations by casting goal-conditioned RL as probabilistic inference.

\paragraph{$\bm{Q}$-Learning.}
Off-policy $Q$-learning algorithms~\citep{watkins1992q} allow training  policies from data collected under alternate decision rules by estimating the expected return $Q^\pi$ for a given state--action pair:
\begin{align*}
    Q^\pi(\bs, \ba)
    &\defines
    \E_{p_\pi} \Big[
    \sum\nolimits_{t=0}^\infty \gamma^t r(\bs_t, \ba_t)
    \Big| \bS_0 = \bs, \bA_0 = \ba
    \Big].
\end{align*}
Crucially, the expected return given a state--action pair can be expressed recursively as
\begin{align}
\label{eq:recursive-q-def}
    Q^\pi(\bs, \ba)
    &=
    r(\bs, \ba)
    \hspace*{-1pt}+\hspace*{-1pt} 
    \gamma
    \E_{p_\pi} [
    Q^\pi(\bs_1, \ba_1)
    \vbar \bS_0 = \bs, \bA_0 = \ba
    ],
\end{align}
which makes it possible to estimate the expectation on the right-hand side from single-step transitions.
The resulting estimates can then be used to find a policy that results in actions which maximize the expected return $Q^\pi(\bs, \ba)$ for all available state--action pairs.

\section{Outcome-Driven Reinforcement Learning}
\label{sec:theory}

In this section, we derive a variational inference objective to infer an approximate posterior policy for achieving desired outcomes.
Instead of using the heuristic goal-reaching rewards discussed in~\Cref{sec:preliminaries}, we will derive a general framework for \emph{inferring actions that lead to desired outcomes} by formulating a probabilistic objective, using the tools of variational inference.
As we will show in the following sections, we use this formulation to translate the problem of inferring a policy that leads to a desired outcome into a tractable variational optimization problem, which we show corresponds to an RL problem with a well-shaped, dense reward signal from which the agent can learn more easily.

We start with a warm-up problem that demonstrates how to frame the task of achieving a desired outcome as an inference problem in a simplified setting.
We then describe how to extend this approach to more general settings.
Finally, we show that the resulting variational objective can be expressed as a recurrence relation, which allows us to derive an outcome-driven variational Bellman operator and prove an outcome-driven probabilistic policy iteration theorem.

\begin{wrapfigure}{r}{0.57\linewidth}
\centering
    \vspace*{-12pt}
    \begin{tikzpicture}
    \small
    \tikzstyle{unobserved}=[draw, circle, minimum size=1.125cm]
    \tikzstyle{observed}=[draw, circle, fill=lightgray, fill opacity=0.2, text opacity=1, minimum size=1.125cm]
    \tikzstyle{arrow}=[-latex, thick]

    \node[observed] (s0) at (-3.0, 0) {$\bS_{0}$};
    \node[unobserved] (a0) at (-3.0, -1.5) {$\bA_{0}$};

    \node[unobserved] (s1) at (-1.5, 0) {$\bS_1$};
    \node[unobserved] (a1) at (-1.5, -1.5) {$\bA_1$};

    \node[unobserved] (s2) at (0, 0) {$\dots$};
    \node[unobserved] (a2) at (0, -1.5) {$\dots$};

    \node[unobserved] (s3) at (1.5, 0) {$\bS_{t^\star\hspace*{-2pt}-1}$};
    \node[unobserved] (a3) at (1.5, -1.5) {$\bA_{t^\star\hspace*{-2pt}-1}$};

    \node[observed] (s4) at (3, 0) {$\bS_{t^\star}$};

    \path[arrow] (s0) edge (s1);
    \path[arrow] (s0) edge (a0);
    \path[arrow] (a0) edge (s1);

    \path[arrow] (s1) edge (s2);
    \path[arrow] (s1) edge (a1);
    \path[arrow] (a1) edge (s2);

    \path[arrow] (s2) edge (s3);
    \path[arrow] (s2) edge (a2);
    \path[arrow] (a2) edge (s3);

    \path[arrow] (s2) edge (s3);
    \path[arrow] (s2) edge (a2);
    \path[arrow] (a2) edge (s3);

    \path[arrow] (s3) edge (s4);
    \path[arrow] (s3) edge (a3);
    \path[arrow] (a3) edge (s4);

    \end{tikzpicture}
    \caption{
    A Probabilistic graphical model of a state--action trajectory with observed random variables $\bS_{0}\hspace*{-1pt}=\hspace*{-1pt}\bs_{0}$ and $\bS_{t^\star}\hspace*{-3pt}=\hspace*{-1pt}\bg$\hspace*{-5pt}.
    \newline\newline
    }
    \label{fig:mdp-time-known}
    \vspace*{-30pt}
\end{wrapfigure}

\subsection{Warm-up: Achieving a Desired Outcome at a Fixed Time Step}
\label{sec:gc_probabilistic_framework_finite}

We first consider a simplified problem, where the desired outcome is to reach a specific state $\bg \in \calS$ at a specific time step $t^\star$ when starting from initial state $\bs_{0}$.
We can think of the starting state $\bs_{0}$ and the desired outcome $\bg$ as boundary conditions, and the goal is to learn a stochastic policy that induces a trajectory from $\bs_{0}$ to $\bg$.
To derive a control law that solves this stochastic boundary value problem, we frame the problem probabilistically, as inferring a state--action trajectory posterior distribution \emph{conditioned on the desired outcome and the initial state}.
We will show that, by framing the learning problem this way, we obtain an algorithm for learning outcome-driven policies without needing to manually specify a reward function.
We consider a model of the state--action trajectory up to and including the desired outcome $\bg$,
\begin{align}
    p_{\Trajnoso_{0:t}, \bS_{t^\star}}(\trajnoso_{0:t}, \bg \vbar \bs_{0}) \defines
    \goaltransitiont{t} \policypriort{t} 
    \nonumber
    \prod_{t'=0}^{t-1} \transitiont{t'} \policypriort{t'},
\end{align}
where $t^\star \defines t + 1$, $\Trajnoso_{0:t}$ is the state--action trajectory up to an including $t$ but excluding $\bS_{0}$, $\policypriort{t}$ is a conditional action prior, and $\transitiont{t}$ is the environment's state transition distribution.
If the dynamics are simple (e.g., tabular or Gaussian), the posterior over actions can be computed in closed form~\citep{attias2003planning}, but we would like to be able to infer outcome-driven posterior policies in any environments, including those where exact inference may be intractable.
To do so, we start by expressing posterior inference as the variational minimization problem
\begin{align}
\label{eq:kl_trajectories_fixed_t}
    \min_{q_{\Trajnoso_{0:t} | \bS_{0}} \in \hat{\mathcal{Q}}} \DKL{q_{\Trajnoso_{0:t} | \bS_{0}}( \cdot \vbar \bs_{0})}{p_{\Trajnoso_{0:t} | \bS_{0}, \bS_{t^\star}}(\cdot \vbar \bs_{0}, \bg )} ,
\end{align}
where $\mathbb{D}_{\textrm{KL}}(\cdot \,\|\, \cdot)$ is the \kld, and $\hat{\calQ}$ denotes the variational family over which to optimize.
We consider a family of distributions parameterized by a policy $\pi$ and defined by
\begin{align}
\SwapAboveDisplaySkip
\label{eq:warm-up-vi-family}
    q_{\Trajnoso_{0:t} | \bS_{0}}(\trajnoso_{0:t} \vbar \bs_{0})
    \hspace*{-2pt}\defines\hspace*{-2pt} \policyt{t} \hspace*{-2pt}\prod_{t'=0}^{t - 1}\hspace*{-2pt} \transitiont{t'} \policyt{t'} ,
\end{align}
where $\pi \in \Pi$, a family of policy distributions, and where $\prod_{t=0}^{t-1} \transitiont{t}$ is the true action-conditional state transition distribution up to but excluding the state transition at \mbox{$t^\star-1$},
since \mbox{$\bS_{t^\star} = \bg$} is observed.
Under this variational family, the inference problem in~\Cref{eq:kl_trajectories_fixed_t} can be equivalently stated as the problem of maximizing the following objective with respect to the policy $\pi$:
\begin{proposition}%
\label{prop:maxent_objective_fixed_time}
Given $q_{\Trajnoso_{0:t} | \bS_{0}}(\trajnoso_{0:t} \vbar \bs_{0})$ from~\Cref{eq:warm-up-vi-family}, any state $\bs_{0}$, termination time $t^\star$, and outcome $\bg$,
solving~\Cref{eq:kl_trajectories_fixed_t} is equivalent to maximizing this objective with respect to $\pi \in \Pi$:
\begin{align}
\begin{split}
\label{eq:maxent_objective_fixed_time}
    \bar{\calF}(\pi, \bs_{0}, \bg)
    \defines
    \mathbb{E}_{q_{\Trajnoso_{0:t} | \bS_{0}}(\trajnoso_{0:t} \vbar \bs_{0})} \bigg[
    \log \goaltransitiont{t} 
    - \sum_{t'=0}^{t-1}
    \mathbb{D}_{\emph{\textrm{KL}}}(\policytdot{t'} \,||\, \policypriortdot{t'}) \bigg].
\end{split}
\end{align}
\end{proposition}
\vspace{-10pt}
\begin{proof}
    See~\Cref{appsec:posterior_trajectory_inference}.
\end{proof}
\vspace{-5pt}
A variational problem of this form---which corresponds to finding a posterior distribution over state--action trajectories---can equivalently be viewed as a reinforcement learning problem:
\begin{corollary}
\label{cor:maxent_finite_time}
The objective in~\Cref{eq:maxent_objective_fixed_time} corresponds to KL-regularized reinforcement learning with a time-varying reward function given by
$
    r(\bs_{t'}, \ba_{t'}, \bg, t')
    \defines
    \mathbb{I}\{ t' = t \} \log \goaltransitiont{t'} .
$
\end{corollary}
\Cref{cor:maxent_finite_time} illustrates how a reward function \emph{emerges automatically} from a probabilistic framing of outcome-driven reinforcement learning problems where the sole specification is which variable ($\bS_{t^\star}$) should attain which value ($\bg$).
In particular,~\Cref{cor:maxent_finite_time} suggests that we ought to learn the environment's state-transition distribution, and view the log-likelihood of achieving the desired outcome given a state--action pair as a ``reward'' that can be used for off-policy learning as described in~\Cref{sec:preliminaries}.
Importantly---and unlike in model-based RL---such a transition model would not have to be accurate beyond single-step predictions, as it would not be used for planning (see~\Cref{appsec:additional-exps}).
Instead, \mbox{$\log \goaltransitiont{t}$} only needs to be well shaped, which we expect to happen for commonly used model classes.
For example, when the dynamics are linear-Gaussian, using a conditional Gaussian model~\citep{nagaband2018mbmf} yields a reward function that is quadratic in $\bS_{t+1}$, making the objective convex and thus more amenable to optimization.

\subsection{Outcome-Driven Reinforcement Learning as Variational Inference}
\label{sec:gc_approximate_inference}

Thus far, we assumed that the time at which the outcome is achieved is given. In many problem settings, we do not know (or care) when an outcome is achieved.
In this section, we present a variational inference perspective on achieving desired outcomes in settings where \emph{no reward function} and \emph{no termination time} are given, but only a desired outcome is provided.
As in the previous section, we derive a variational objective and show that a principled algorithm and reward function emerge automatically when framing the problem as variational inference.

To derive such an objective, we modify the probabilistic model used in the previous section to model that the time at which the outcome is achieved is not known.
As before, we define an ``outcome'' as a point in the state space, but instead of defining the event of ``achieving a desired outcome'' as a realization $\bS_{t^\star} = \bg$ for a known $t^\star$, we define it as a realization $\bS_{T^\star} = \bg$ for an \emph{unknown} termination time $T^\star$ at which the outcome is achieved.
Specifically, we model the distribution over the trajectory and the unknown termination time with
\begin{align}
\SwapAboveDisplaySkip
    p_{\Trajnoso_{0:T}, \bS_{T^\star}, T | \bS_{0}}(\trajnoso_{0:t}, \bg, t \vbar \bs_0)
    =
    p_{T}(t) \goaltransitiont{t} \policypriort{t}
    \prod_{t'=0}^{t-1} \transitiont{t'} \policypriort{t'},
    \label{eq:trajectory_mixture}
\end{align}
where $p_{T}(t)$ is the probability of reaching the outcome at $t+1$.
Since the trajectory length is itself a random variable, the joint distribution in~\Cref{eq:trajectory_mixture} is a \textit{transdimensional} distribution defined on \mbox{$\biguplus_{t=0}^\infty \{ t \} \times \calS^{t+1} \times \calA^{t+1}$}~\citep{hoffman2009expectation}.

Unlike in the warm-up, the problem of finding an outcome-driven policy that eventually achieves the desired outcome corresponds to finding the posterior distribution over state--action trajectories \emph{and} the termination time $T$ conditioned on the desired outcome $\bS_{T^\star}$ and a starting state.
Analogously to~\Cref{sec:gc_probabilistic_framework_finite}, we can express this inference problem variationally as
\begin{align}
\label{eq:gc_kl_minimization}
    \min_{q_{\Trajnoso_{0:T}, T | \bS_{0}} \in \calQ} \DKL{
    q_{\Trajnoso_{0:T}, T | \bS_{0}}( \cdot \vbar \bs_{0})
    }{
    p_{\Trajnoso_{0:T}, T | \bS_{0}, \bS_{T^{\star}}}( \cdot \vbar \bs_{0}, \bg)
    } ,
\end{align}
where $t$ denotes the time immediately \emph{before} the outcome is achieved, $\calQ$ denotes the variational family.
In general, solving this variational problem in closed form is challenging, but by choosing a variational family
\mbox{$q_{\Trajnoso_{0:T}, T | \bS_{0}}(\trajnoso_{0:t}, t \vbar \bs_0) = q_{\Trajnoso_{0:T} | T, \bS_{0}}(\trajnoso_{0:t} \vbar t, \bs_0) q_{T}(t)$}, where $q_{T}$ is a distribution over $T$ in some variational family $\calQ_{T}$
parameterized by
\begin{align}
\SwapAboveDisplaySkip
\label{eq:qT-delta-form}
    \qtvar = \qdeltatone{t+1}  \prod_{t'=1}^{t} \qdeltatzero{t'},
\end{align}
with Bernoulli random variables $\Delta_{t}$ denoting the event of ``reaching $\bg$ at time $t$ given that $\bg$ has not yet been reached by time \mbox{$t-1$},''
we can equivalently express the variational problem in~\Cref{eq:gc_kl_minimization} in a way that is tractable and amenable to off-policy optimization:
\begin{theorem}%
\label{thm:gc_variational_problem}
    Let $q_{T}(t)$ and \mbox{$q_{\Trajnoso_{0:T} | T, \bS_{0}}(\trajnoso_{0:t} \vbar t, \bs_{0})$} be as defined before,
    and define
    \begin{align}
    \label{eq:odac_variational_objective}
        V_{\text{{}}}^\pi(\bs_{t}, \bg; q_{T})
        &
        \defines
        \E_{\policyt{t}} \hspace{-1pt} \left[ Q_{\text{{}}}^{\pi}(\bs_{t}, \ba_{t}, \bg ; q_{T}) \right]
        - \mathbb{D}_{\emph{\textrm{KL}}}(\policytdot{t} \,\|\,\policypriortdot{t})
        \\
        Q_{\text{{}}}^{\pi}(\bs_t, \ba_t, \bg; q_{T})
        &
        \defines
        \rewardt{t}
        +
        \qdeltatzero{t+1}
        \E_{\dyn(\bs_{t+1} \vbar \bs_t, \ba_t)} \hspace{-1pt}\left[ V^\pi(\bs_{t+1}, \bg ; \pi, q_{T})
        \right]
        \\
        \label{eq:derived-reward}
        \rewardt{t}
        &
        \defines
        \qdeltatone{t+1} \log \goaltransitiont{t}
        - \mathbb{D}_{\emph{\textrm{KL}}}(\qdeltatdot{t+1} \,\|\, \pdeltatdot{t+1}).
    \end{align}
    Then given any initial state $\bs_{0}$ and outcome $\bg$,
    \begin{align*}
        \mathbb{D}_{\emph{\textrm{KL}}}( q_{\Trajnoso_{0:T}, T | \bS_{0}}( \cdot \vbar \bs_{0})
        \,\|\,
        p_{\Trajnoso_{0:T}, T | \bS_{0}, \bS_{T^\ast}}( \cdot \vbar \bs_{0}, \bg)
        )
        =
        -V^\pi(\bs_0, \bg; q_{T}) + \log p(\bg \vbar \bs_{0}),
    \end{align*}
    where $ \log p(\bg \vbar \bs_{0})$ is independent of $\pi$ and $q_{T}$
    and hence maximizing $V^\pi(\bs_{0}, \bg ; \pi, q_{T})$ is equivalent to
    minimizing~\Cref{eq:gc_kl_minimization}.
\end{theorem}
\vspace{-10pt}
\begin{proof}
    See~\Cref{appsec:variational_approximation}.
\end{proof}
\vspace{-5pt}
This theorem tells us that the maximizer of $V^\pi(\bs_{t}, \bg; q_{T})$ is equal to the minimizer of~\Cref{eq:gc_kl_minimization}.
In other words,~\Cref{thm:gc_variational_problem} presents a variational objective with dense reward functions defined solely in terms of the desired outcome and the environment dynamics, which we can learn directly from environment interactions.
It further makes precise that the variational objective, \mbox{$V_{\text{{}}}^\pi(\bs_{0}, \bg; q_{T})$}, is a lower bound on the log-marginal likelihood, that is, $\log p(\bg \vbar \bs_{0}) \geq V_{\text{{}}}^\pi(\bs_{0}, \bg; q_{T})$, where
\begin{align*}
    V_{\text{{}}}^\pi(\bs_{0}, \bg; q_{T})
    =
    \E \Bigg[ \sum_{{t} = 0}^\infty \left( \prod_{t'=1}^{t} \qdeltatzero{t'} \right)\hspace*{-3pt}\Big(\rewardt{t} - \DKL{\policytdot{t}}{ \policypriortdot{t}} \Big) \Bigg],
\end{align*}
with the expectation taken with respect to the infinite-horizon trajectory distribution $q_{\Trajnoso_{0} | \bS_{0}}(\trajnoso_{0} \vbar \bs_{0})$.
Thanks to the recursive expression of the variational objective, we can find the optimal variational over $T$ as a function of the current policy and $Q$-function analytically:
\begin{proposition}%
\label{prop:gc_optimal_variational_distribution_T}
    The optimal distribution $q_{T}^\star$ with respect to~\Cref{eq:odac_variational_objective} is %
    \begin{align}
    \begin{split}
    \label{eq:optimal-qt}
        q_{\Delta_{t+1}}^\star(\Delta_{t+1} = 0; \pi, Q^\pi)
        =
        \sigmoid\left(\hspace*{-1pt} \Lambda(\bs_{t}, \pi, q_{T}, Q^\pi) + \sigmoid^{-1}\hspace{-1pt}\left(\pdeltatzero{t+1} \right)
        \hspace*{-1pt}\right)\hspace*{-2pt},
    \end{split}
    \end{align}
    where
    \begin{align*}
    \Lambda(\bs_{t}, \pi, q_{T}, Q^\pi)
    \defines
    \E_{\policyt{t+1} \transitiont{t} \policyt{t}}[ Q_{\text{{}}}^{\pi}(\bs_{t+1}, \ba_{t+1}, \bg; q_{T}) - \log \goaltransitiont{t}]
    \end{align*}
    and $\sigmoid(\cdot)$ is the sigmoid function, that is, \mbox{$\sigmoid(x) = \frac{1}{e^{-x}+1}$} and \mbox{$\sigmoid^{-1}(x) = \log \frac{x}{1-x}$}.
\end{proposition}
\vspace{-10pt}
\begin{proof}
See~\Cref{appsec:alternative-posterior}
\end{proof}
\vspace{-5pt}

Alternatively, if instead of learning $q_{T}$ variationally, we fix $q_{T}$ to the prior $p_{T}$, we recover the more conventional fixed-discount factor objective~\citep{galashov2019information,haarnoja2018sac,rudner2020pathologies}:
\begin{corollary}
\label{cor:gc_variational_problem_naive}
    Let \mbox{$q_{T} = p_{T}$}, assume that $p_{T}$ is a Geometric distribution with parameter $\gamma \in (0,1)$.
    Then the inference problem in~\Cref{eq:gc_kl_minimization} of finding a goal-directed variational trajectory distribution simplifies to maximizing the following recursively defined variational objective with respect to $\pi$:
    \begin{align}
    \begin{split}
        \bar{V}^\pi_{\text{{}}}(\bs_{0} , \bg ; \gamma)
        &\defines
        \E_{\policyt{0}} \left[ Q_{\text{{}}}(\bs_{0}, \ba_{0}, \bg ; \gamma ) \right] - \mathbb{D}_{\emph{\textrm{KL}}}(\policytdot{0} \,\|\, \policypriortdot{0})),
    \end{split}
    \end{align}
    where
    \begin{align}
    \begin{split}
        &\bar{Q}^\pi_{\text{{}}}(\bs_{0}, \ba_{0}, \bg ; \gamma) \defines\,
        (1 - \gamma) \log \goaltransitiont{0} + \gamma \E_{\dyn(\bs_{1} | \bs_{0}, \ba_{0})} \big[ V_{\text{{}}}(\bs_{1}, \bg ; \gamma) \big].
    \end{split}
    \end{align}
\end{corollary}

In the next section, we derive a temporal-difference algorithm and discuss how we can learn the $Q$-function in~\Cref{thm:gc_variational_problem} using off-policy transitions.

\section{Outcome-Driven Reinforcement Learning}
\label{sec:gc_algorithms}

In this section, we show that the variational objective in~\Cref{thm:gc_variational_problem} is amenable to off-policy learning and that it can be estimated efficiently from single-step transitions.
We then describe how to instantiate the resulting outcome-driven algorithm in large environments where function approximation is necessary.

\subsection{Outcome-Driven Policy Iteration}
\label{sec:gc_policy_iteration}

To develop an outcome-directed off-policy algorithm, we define the following Bellman operator:
\begin{definition}%
\label{def:gc_bellman_operator}
    Given a function $Q_{\text{{}}}: \calS \times \calA \times \calS \to \R$, define the operator $\calT^{\pi}$ as
    \begin{align}
    \begin{split}
    \label{eq:variational-q-t-bellman-update}
        \calT^{\pi} Q_{\text{{}}}(\bs_{t}, \ba_{t}, \bg ; q_{T})
        \defines
        \rewardt{t}
        + \qdeltatzero{t+1} \E_{\transitiont{t}} \big[ V_{\text{{}}}(\bs_{t+1}, \bg; q_{T}) \big],
    \end{split}
    \end{align}
    where $\rewardt{t}$ is from~\Cref{thm:gc_variational_problem} and
    \begin{align}
    \begin{split}
        V_{\text{{}}}(\bs_{t}, \bg; q_{T})
        \defines \E_{\policyt{t}} \left[ Q_{\text{{}}}(\bs_{t}, \ba_{t}, \bg ; q_{T} ) \right]
        + \mathbb{D}_{\emph{\textrm{KL}}}(\policytdot{t} \,\|\, \policypriortdot{t}).
    \end{split}
    \end{align}
\end{definition}
Unlike the standard Bellman operator, the above operator has a varying weight factor \mbox{$\qdeltatzero{t+1}$}, with the optimal weight factor given by~\Cref{eq:optimal-qt}.
From~\Cref{eq:optimal-qt}, we see that as the 
outcome likelihood $\dyn(\bg \vbar \bs, \ba)$ 
becomes large relative to the $Q$-function, the weight factor automatically adjusts the target to rely more on the rewards.

Below, we show that repeatedly applying the operator $\calT^\pi$ (policy evaluation) and optimizing $\pi$ with respect to $Q^{\pi}$ (policy improvement) converges to a policy that maximizes the objective in~\Cref{thm:gc_variational_problem}.
\begin{theorem}%
    Assume MDP is ergodic and $|\calA| <\infty$.
    \begin{enumerate}[leftmargin=15pt]
    \item Outcome-Driven Policy Evaluation (ODPE): Given policy $\pi$ and a function $Q^{0}: \calS \times \calA \times \calS \rightarrow \R$, define $Q^{i+1}_{\text{{}}} = \calT^{\pi} Q^{i}_{\text{{}}}$. Then the sequence $Q^{i}_{\text{{}}}$ converges to the lower bound in~\Cref{thm:gc_variational_problem}.
    \item Outcome-Driven Policy Improvement (ODPI): The policy
    \begin{align}
    \begin{split}
        \pi^+
        =
        \argmax_{\pi' \in \Pi} \{ \E_{\pi'(\ba_{t} \vbar \bs_{t})} \left[ Q^{\pi}(\bs_{t}, \ba_{t}, \bg; q_{T}) \right]
        - \mathbb{D}_{\emph{\textrm{KL}}}(\pi'( \cdot \vbar \bs_{t}) \,||\, \policypriortdot{t} \}
    \end{split}
    \label{eq:odpi}
    \end{align}
    and the variational distribution over $T$ defined in \Cref{eq:optimal-qt}
    improve the variational objective, that is, $\calF(\pi^+, q_T, \bs_0) \geq \calF(\pi, q_T, \bs_0)$ and $\calF(\pi, q_T^+, \bs_0) \geq \calF(\pi, q_T, \bs_0)$ for all $\bs_0, \pi, q_T$.%
    \item Alternating between ODPE and ODPI converges to a policy $\pi^\star$ and a variational distribution over $T$, $q_{T}$, such that $\Qpgc^{\pi^\star}(\bs, \ba, \bg; q_{T}^\star) \geq \Qpgc^{\pi}(\bs, \ba, \bg; q_{T})$ for all $(\pi, q_{T}) \in \Pi \times \calQ_{T}$ and any \mbox{$(\bs, \ba) \in \calS \times \calA$}.
    \end{enumerate}
\end{theorem}
\vspace{-10pt}
\begin{proof}
    See~\Cref{appsec:proof_policy_iteration_theorem}.
\end{proof}
\vspace{-5pt}
This result tells us that alternating between applying the outcome-driven Bellman operator in~\Cref{def:gc_bellman_operator} and optimizing the bound in~\Cref{thm:gc_variational_problem} using the resulting $Q$-function, which can equivalently be viewed as expectation maximization, will lead to a policy that induces an outcome-driven trajectory and solves the inference problem in~\Cref{eq:gc_kl_minimization}.
As we discuss in~\Cref{appsec:proof_policy_iteration_theorem}, this implies that Variational Outcome-Driven Policy Iteration is theoretically at least as good as or better than standard policy iteration for KL-regularized objectives.

\subsection{Outcome-Driven Actor--Critic (ODAC)}
\label{sec:gc_actor_critic}

We now build on previous sections to develop a practical algorithm that handles large and continuous domains.
In such domains, the expectation in the Bellman operator in~\Cref{def:gc_bellman_operator} is intractable, and so we approximate the policy $\pi_\ppi$ and $Q$-function $Q_\pq$ with neural networks parameterized by parameters $\ppi$ and $\pq$, respectively.
We train the $Q$-function to minimize
\begin{align}
\begin{split}
    \calF_Q(\pq)
    &=
    \E
    \bigg[
    \Big(
        Q_\pq (\bs, \ba, \bg) -
        (
            r(\bs, \ba, \bg; q_{\Delta})
            \label{eq:q-loss}
            + \qdeltatzero{t} \,            \hat{V}(\bs', \bg)
        )
    \Big)^2
    \bigg],
\end{split}
\end{align}
where the expectation is taken with respect to \mbox{$(\bs, \ba, \bg, \bs')$} sampled from a replay buffer, $\calD$, of data collected by a policy. 
We approximate the $\hat{V}$-function using a target $Q$-function $Q_\pqtarget$:
\mbox{$
    \hat{V}(\bs', \bg)
    \approx
    Q_{\pqtarget}(\bs', \ba', \bg) - \log \pi(\ba' \vbar \bs' ; \bg)
$},
where $ \ba' $ are samples from the amortized variational policy $\pi_\ppi(\cdot \vbar \bs'; \bg)$.
We further assume a uniform prior policy $\policypriortdot{t}$ in all of our experiments.
The parameters $\bar{\phi}$ slowly track the parameters of $\pq$ at each time step via the standard update \mbox{$\pqtarget \leftarrow \tau \pqtarget + (1-\tau) \pq$}~\citep{lillicrap2015continuous}.
We then train the policy to maximize the approximate $Q$-function
by performing gradient descent on
\begin{align}
\label{eq:pi-loss}
    \calF_\pi(\ppi) = - \E_{\bs \sim \calD, \ba \sim \pi_\ppi(\cdot \vbar \bs; \bg)} \left[
        Q_\pq(\bs, \ba, \bg) - \log \pi_{\ppi}(\ba \vbar \bs; \bg)
    \right].
\end{align}

We estimate \mbox{$\hat{q}_{\Delta_{t+1}}(\Delta_{t+1} = 0)$} with a Monte Carlo estimate of~\Cref{eq:optimal-qt} obtained via a single Monte Carlo sample \mbox{$(\bs, \ba, \bs', \ba', \bg)$} from the replay buffer.
In practice, a value of \mbox{$\qdeltatzero{t+1}=1$} can lead to numerical instabilities with bootstrapping, and so we also upper bound the estimated \mbox{$\qdeltatzero{t+1}$} by the prior distribution \mbox{$\pdeltatzero{t+1}$}.

To compute the rewards, we need to compute the likelihood of achieving the desired outcome.
If the transition dynamics are unknown, we learn a dynamics model from environment interactions by training a neural network $p_\pdyn$ that parameterizes the mean and scale of a factorized Laplace distribution.
We train this model by maximizing the log-likelihood of the data collected by the policy,
\begin{align}
\label{eq:dyn-loss}
    \calF_p(\pdyn) = \E_{(\bs, \ba, \bs') \sim \calD}[ \log p_\pdyn(\bs' \vbar \bs, \ba)],
\end{align}
and use it to compute the rewards
\begin{align}
\begin{split}
\label{eq:approx-reward}
        \hat{r}(\bs_{t}, \ba_{t}, \bg; q_{\Delta})
        \defines
        \hat{q}_{\Delta_{t+1}}(\Delta_{t+1} = 1) \log p_\pdyn(\bg \vbar \bs_{t}, \ba_{t})
        - \DKL{\qdeltatdot{t}}{\pdeltatdot{t}}.
\end{split}
\end{align}

The complete algorithm is presented in~\Cref{alg:odac-alg} and consists of alternating between collecting data via policy $\pi$ and minimizing Equations \ref{eq:q-loss}, \ref{eq:pi-loss}, and \ref{eq:dyn-loss} via gradient descent.
This algorithm alternates between approximating the lower bound in~\Cref{eq:odac_variational_objective} by repeatedly applying the outcome-driven Bellman operator to an approximate $Q$-function, and maximizing this lower bound by performing approximate policy optimization on~\Cref{eq:pi-loss}.

\begin{figure}[t!]
\vspace{-8pt}
    \begin{algorithm}[H]
      	\caption{\odac: Outcome-Driven Actor--Critic}
      	\label{alg:odac-alg}
      	\begin{algorithmic}[1]
      	 \STATE Initialize policy $\pi_\theta$, replay buffer $\mathcal{R}$, $Q$-function $Q_\phi$, and dynamics model $p_\psi$.
      	\FOR{iteration $i=1, 2, ...$}
      	    \STATE Collect on-policy samples to add to $\mathcal{R}$ by sampling $\bg$ from environment and executing $\pi$.
            \STATE Sample batch $(\bs, \ba, \bs', \bg)$ from $\mathcal R$.
            \STATE Compute approximate reward and optimal weights with~\Cref{eq:approx-reward} and~\Cref{eq:optimal-qt}.
            \STATE Update $Q_\phi$ with~\Cref{eq:q-loss}, $\pi_\theta$ with~\Cref{eq:pi-loss}, and $p_\psi$ with~\Cref{eq:dyn-loss}.
      	\ENDFOR
      	\end{algorithmic}
    \end{algorithm}
\vspace{-20pt}
\end{figure}

\section{Related Work}

\label{sec:related_work}
Our problem definition is related to goal-conditioned RL, where desired outcomes are often defined in terms of an exact goal-equality indicator~\citep{kaelbling1993goals,schaul2015uva}. Unlike in our work, goal-conditioned RL typically requires specifying a reward function reflecting the desired outcome.
The most natural choice for a reward function in this setting is an exact goal--equality indicator, which gives non-zero reward whenever an outcome is reached.
However, this type of reward function makes learning difficult, since the associated reward signal is sparse at best and impossible to attain at worst (in continuous state spaces).

To overcome this limitation, prior work has proposed heuristics for creating dense reward functions, such as the Euclidean distance or a thresholded distance to a goal~\citep{andrychowicz2017her,levy2017learning,nachum2018hiro,Plappert2018,pong2018tdm,schroecker2020universal}, or estimating auxiliary metrics to encourage learning, such as the mutual information or time between states and goal~\citep{eysenbach2019search,hartikainen2019dynamical,pong2019skewfit,venkattaramanujam2019self,wardefarley2019discern}.
In contrast, a dense, generally-applicable reward function results automatically from our variational objective.
In~\Cref{sec:experiments}, we demonstrate that this reward function is substantially easier to optimize than sparse rewards, and that it removes the need to choose arbitrary thresholds or distance metrics needed in alternative approaches.

Several prior works cast RL and control as probabilistic inference~\citep{Fellows2019Virel,fu2018vice,hoffman2009expectation,levine2018tutorial,rawlik2013soc,singh2019end,todorov2006lmdp,ziebart2008maxent} or \kld minimization~\citep{karny1996towards,peters2010reps}, but with the aim of reformulating standard reward-based RL, assuming that the reward function is given.
In other words, these prior works similarly study how optimizing $r_\text{manual}(s, a) + \mathcal{H}(\pi)$ corresponds to solving a probabilistic inference problem where $r_\text{manual}$ is used to define the outcome's log-likelihood function.
However, these prior methods assume that $r_\text{manual}$ is provided ex ante.
In contrast, our work removes the need to manually specify a task-specific reward or likelihood function, and instead derives both an objective for learning an environment-specific likelihood function and a learning algorithm from the same inference problem.

Past work has also studied control as probabilistic inference in the context of
reaching a goal or desired outcome~\citep{attias2003planning,fu2018vice, hoffman2009expectation,toussaint2006probabilistic}.
\citet{toussaint06probabilisticinference} and~\citet{hoffman2009expectation} focus on exact inference methods that require time-varying tabular or time-varying Gaussian value functions.
In contrast, we propose a variational inference method that eliminates the need to train a time-varying value function, and enables us to use expressive neural networks to represent an approximate value function, making our method applicable to high-dimensional, continuous, and non-linear domains.
Unlike the approach in~\citet{attias2003planning}, our formulation is not constrained to fixed-horizon settings, obtains a closed-loop rather than open-loop policy, and is applicable to non-tabular dynamic models.
More recently,~\citet{fu2018vice} proposed a probabilistic inference method for solving the unknown time-step formulation, but required on-policy trajectory samples.
In contrast, we derive an off-policy method by introducing a variational distribution $q_T$ over the time when the outcome is reached.

Lastly, the closely related problem of finding control laws that allow agents to move from an initial state to some desired goal state while incurring minimal cost has been studied in the stochastic control literature~\citep{chen2018optimal,grigoriadis1997minimum,hotz1987covariance,xu1992}.
\citet{rawlik2010approximate} consider a continuous-time setting and propose an expectation maximization algorithm that assumes linear Gaussian dynamics, while more recent work has explored finding control laws for non-linear stochastic system dynamics~\citep{ridderhof2019nonlinear,yi2019nonlinear}.
In contrast to this strand of research, our framework considers a discrete-time setting and does not make assumptions or assumes knowledge of the system dynamics but only requires the ability to interact with the environment to learn an outcome-driven policy.

\begin{figure*}[t!]
\centering
    \subfloat[Box 2D]{
    \label{fig:2d-env-picture}%
      \includegraphics[height=2cm]{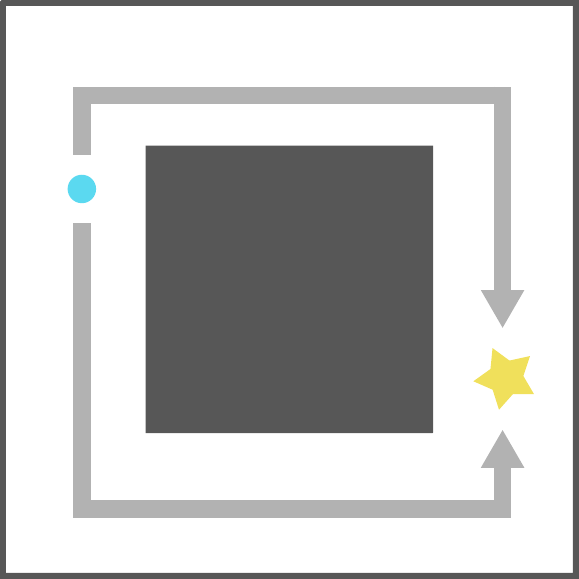}}%
    \hfill
    \subfloat[Ant]{
    \label{fig:sawyer-push-picture}%
      \includegraphics[height=2cm]{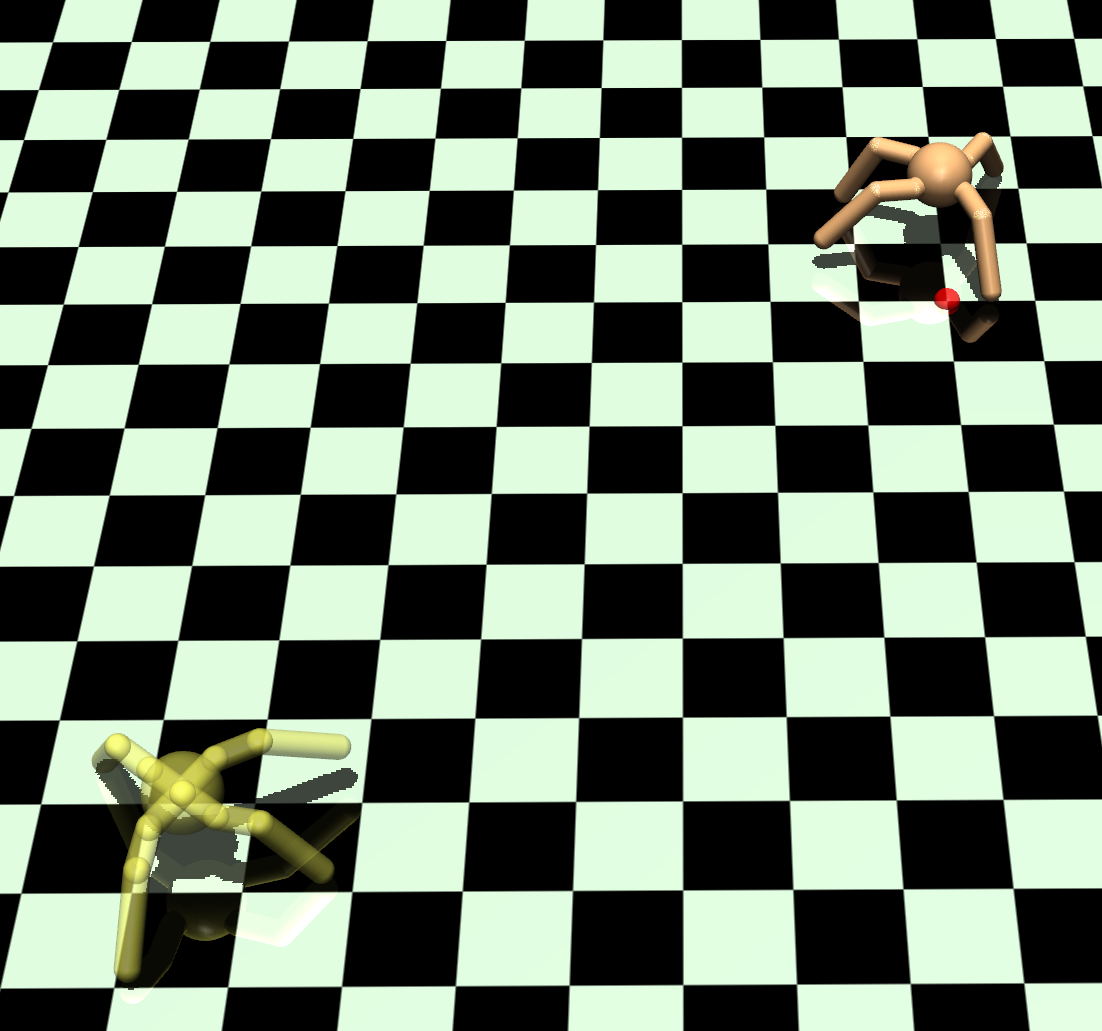}}%
    \hfill
    \subfloat[Sawyer Push]{
    \label{fig:fetch-push-picture}%
      \includegraphics[height=2cm]{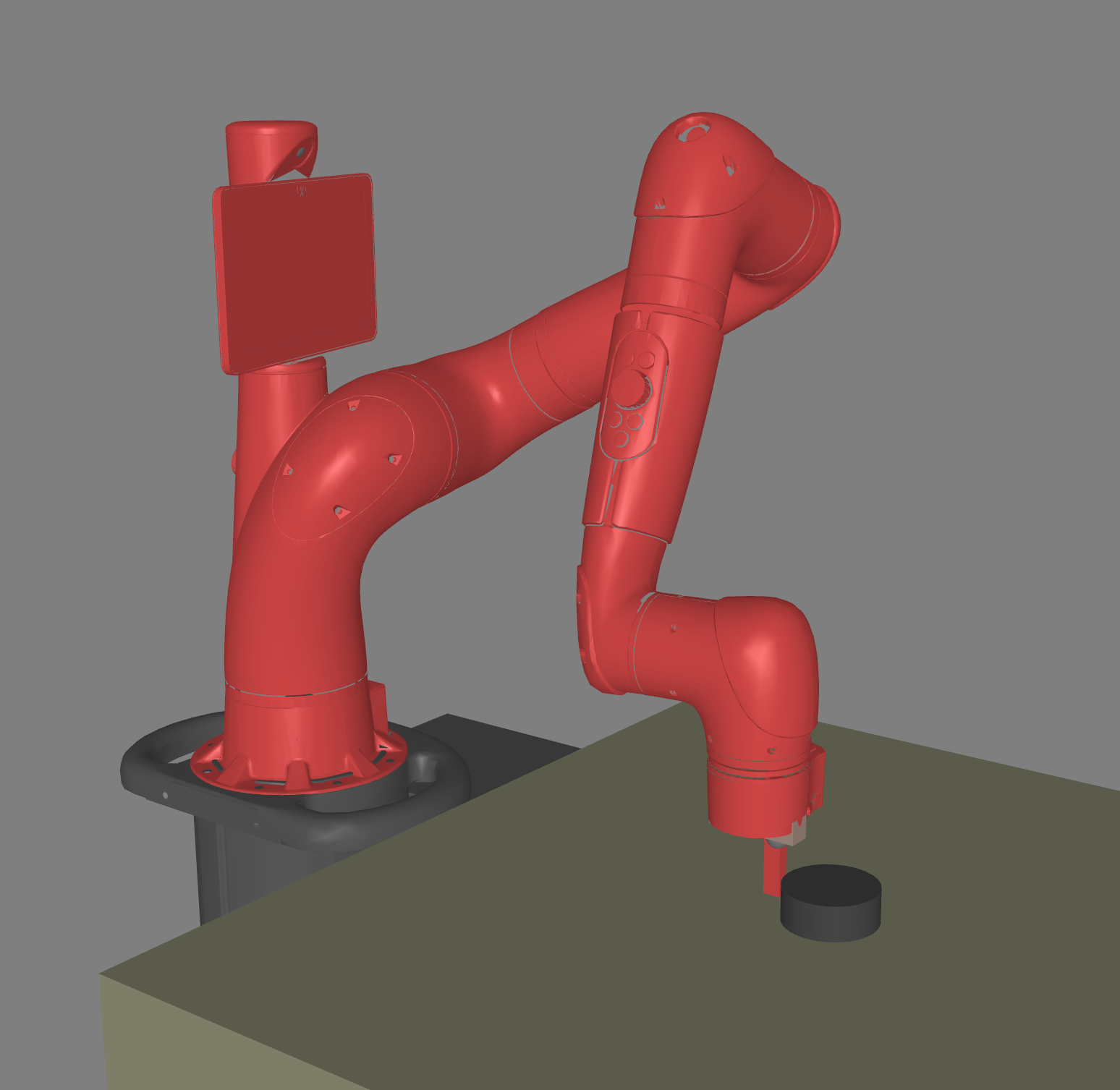}}%
    \hfill
    \subfloat[Fetch Push]{
    \label{fig:illustrative_two_moons_small_var}%
      \includegraphics[height=2cm]{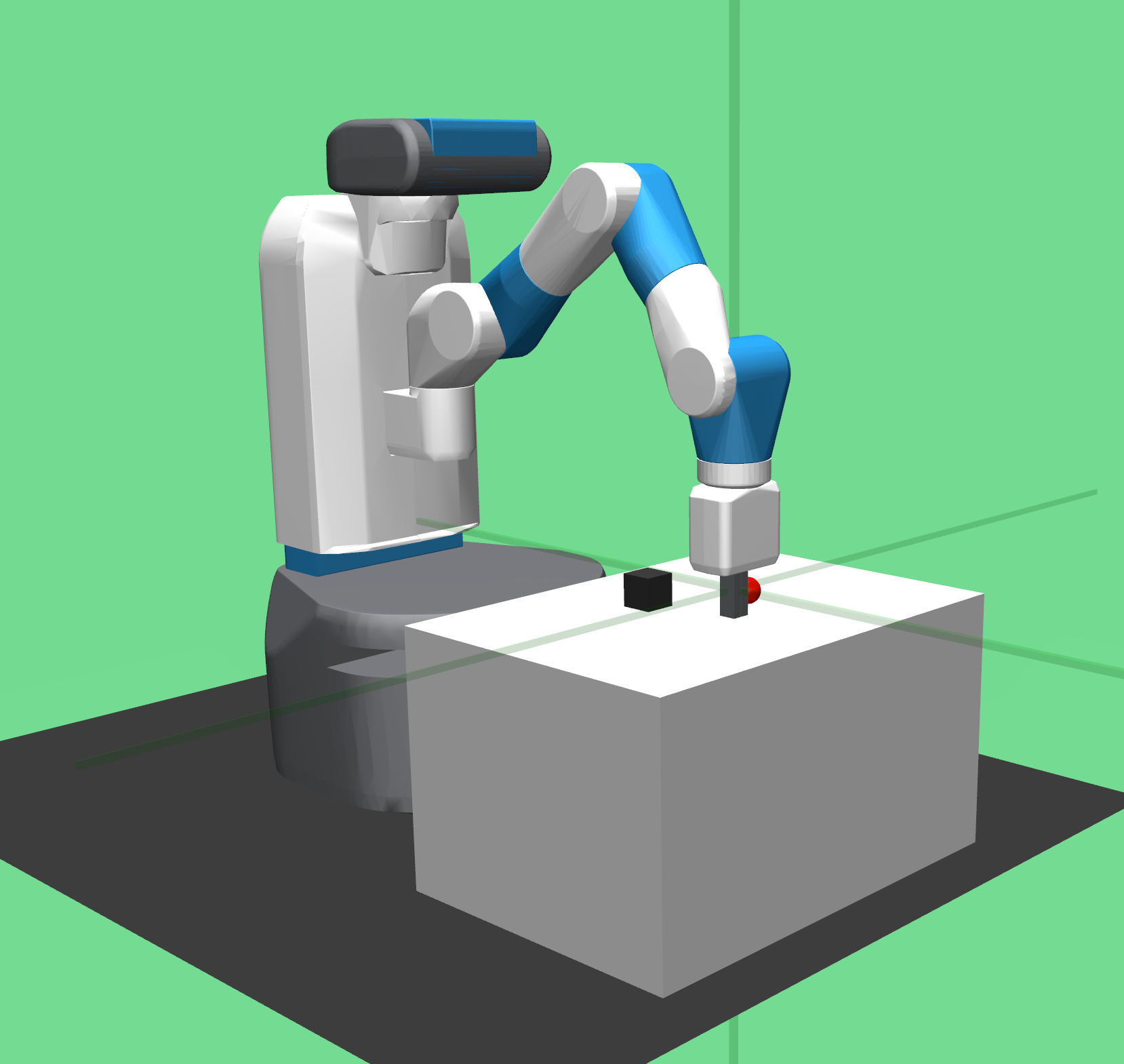}}%
    \hfill
    \subfloat[Faucet]{
    \label{fig:sawyer-faucet-picture}%
      \includegraphics[height=2cm]{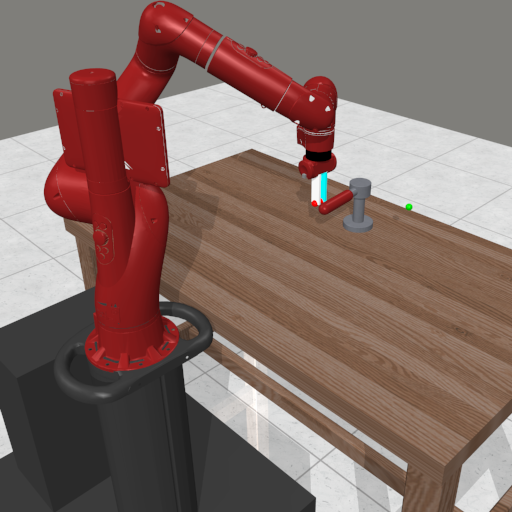}}%
    \hfill
    \subfloat[Window]{
    \label{fig:sawyer-window-picture}%
      \includegraphics[height=2cm]{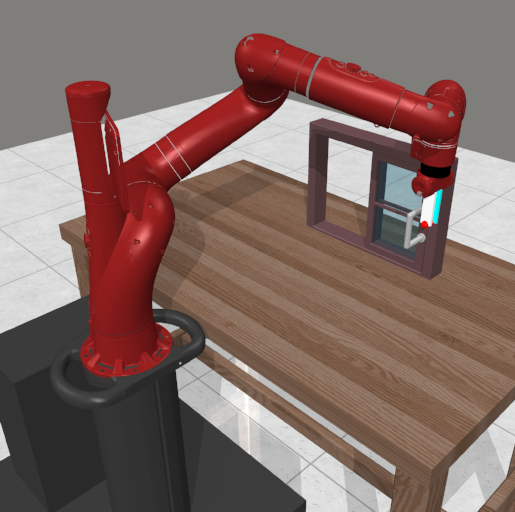}}%
    \caption{
      From left to right, we evaluate on: a 2D environment in which an agent must move around a box, a locomotion task in which a quadruped robot must match a location and pose (yellow), and four manipulation tasks in which the robot must push objects, rotate faucet valve, or open a window.
      }
      \label{fig:exp_illustrations}
    \vspace*{-10pt}
\end{figure*}

\section{Empirical Evaluation}
\label{sec:experiments}

Our experiments compare \odac to prior methods for learning goal-conditioned policies and evaluate how the components of our variational objective impact the performance of the method.
Specifically, we compare \odac to prior methods on a wide range of manipulation and locomotion tasks that require achieving a desired outcome.
To answer the second question, we conduct several ablation studies and visualize the behavior of $\qdeltatzero{t+1}$.
In our experiments, we use a uniform action prior $p(\ba)$ and the time prior $p_{T}$ is geometric with parameter $0.01$, that is, \mbox{$p_{\Delta_{t+1}}(\Delta_{t+1} = 0) = 0.99$}.
We begin by describing prior methods and environments used for the experiments.

\subsection{Learning to Achieve Desired Outcomes}

To avoid over-fitting to any one setting, we compare to these methods
across several different robot morphologies and tasks, all illustrated in~\Cref{fig:exp_illustrations} of the appendix.

\paragraph{Environments.}
We compare \odac to prior work on
a simple 2D navigation task, in which an agent must take non-greedy actions to move around a box,
as well as the \textit{Ant}, \textit{Sawyer Push}, and \textit{Fetch Push} simulated robot domains, which have each been studied in prior work on reinforcement learning for reaching goals~\citep{andrychowicz2017her,nachum2018hiro,nair2018rig,pong2019skewfit,schroecker2020universal}.
For the Ant and Sawyer tasks, desired outcomes correspond to full states (that is, desired positions and joints).
For the Fetch task, we use the same goal representation as in prior work~\citep{andrychowicz2017her} and only represent $\bg$ with the position of the object.
Lastly, we demonstrate the feasibility of replacing manually designed rewards with our outcome-driven paradigm by evaluating the methods on the \textit{Sawyer Window} and \textit{Sawyer Faucet} tasks from the MetaWorld benchmark~\citep{yu2020meta}.
These tasks come with manually designed reward functions, which we replace by simply specify a desired outcome $\bg$.
We plot the mean and standard deviation of the final Euclidean distance to the desired outcome across four random seeds.
We normalize the distance to be $1$ at the start of training.
For further details, see~\Cref{appsec:env_details}.

\paragraph{Goal Sampling.}
In all tasks, rather than assuming that we are able to perform oracle goal sampling (which can be difficult or even impossible in complex real-world environments), a fixed desired outcome is commanded as the exploration goal in each episode.
During training, the goals are relabeled using the future-style relabeling scheme from~\citet{andrychowicz2017her}.
Unlike oracle goal sampling, this approach does not assume knowledge of the set of admissible states in the environment, but is more realistic and presents a more challenging exploration problem.
To challenge the methods, we choose the desired goal to be far from the starting state. 

\paragraph{Baselines and Prior Work.}
We compare our method to hindsight experience replay (HER)~\citep{andrychowicz2017her}, a goal-conditioned method, where the learner receives a reward of $-1$ if it is within an $\epsilon$ distance from the goal, and $0$ otherwise,
universal value density estimation (UVD)~\citep{schroecker2020universal}, which also uses sparse rewards as well as a generative model of the future occupancy measure to estimate the $Q$-values, and DISCERN~\citep{wardefarley2019discern}, which learns a reward function by training a discriminator and using a clipped log-likelihood as the reward.
Lastly, we include an oracle Soft Actor--Critic (SAC) baseline that uses a manually designed reward.
For the MetaWorld tasks, this baseline uses the benchmark reward for each task.
For the remaining environments, this baseline uses the Euclidean distance between the agent's current and the desired outcome for the reward.

\paragraph{Results.}
In~\Cref{fig:sawyer-and-ant}, we see that \odac outperforms virtually every method on all tasks, consistently learning faster and often reaching a final distance that is orders of magnitude closer to the desired outcome.
The only exception is that the hand-crafted reward learns slightly faster on the 2D task, but this gap is closed within a few ten thousand steps.

\begin{figure*}[t]
     \centering
     \hspace*{-8pt}\includegraphics[width=0.345\textwidth]{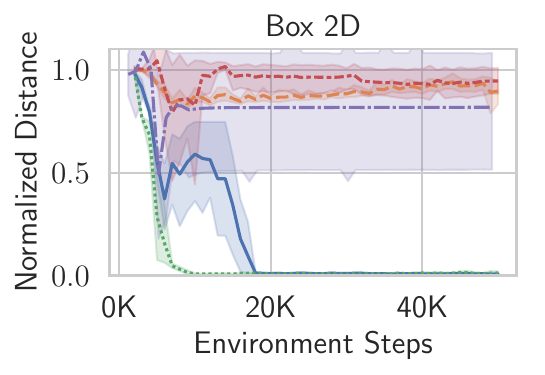}
     \hspace*{-5pt}\includegraphics[width=0.345\textwidth]{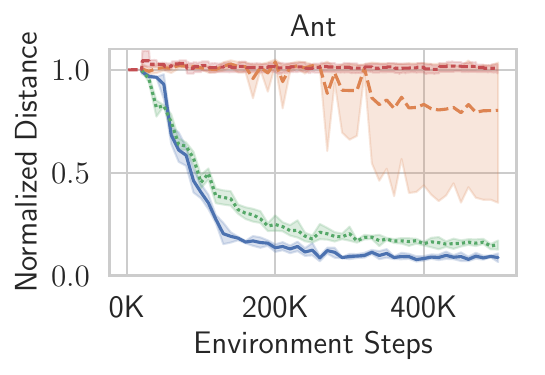}
     \hspace*{-5pt}\includegraphics[width=0.345\textwidth]{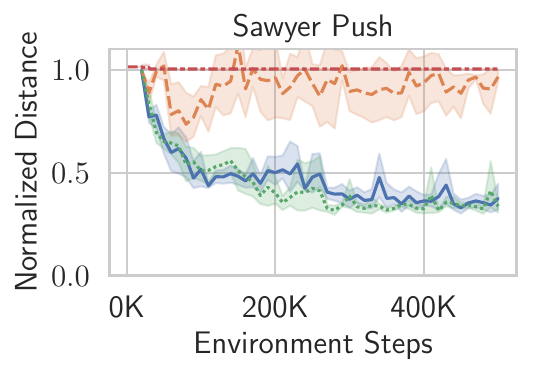}\hspace*{-10pt}
     \\
     \hspace*{-8pt}\includegraphics[width=0.345\textwidth]{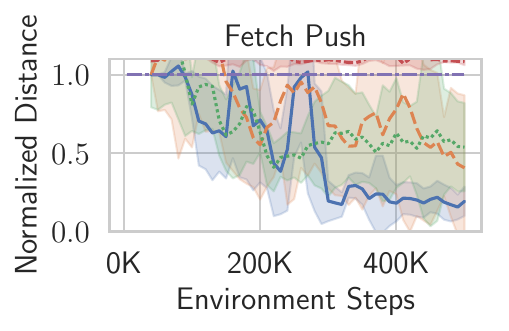}
     \hspace*{-5pt}\includegraphics[width=0.345\textwidth]{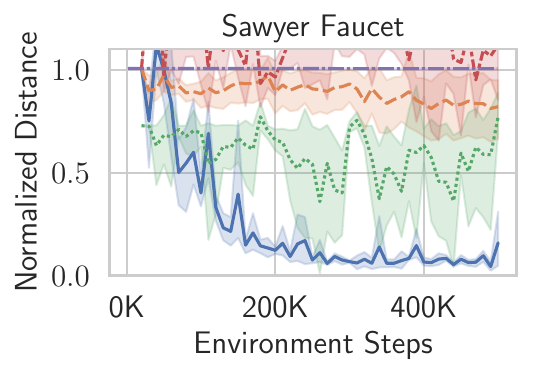}
     \hspace*{-5pt}\includegraphics[width=0.345\textwidth]{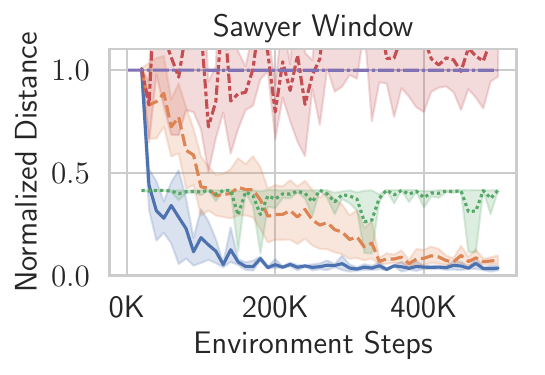}\hspace*{-10pt}
    \\
    \vspace{1pt}
     \includegraphics[width=0.87\textwidth,trim={0 0 0.4cm 0.45cm},clip]{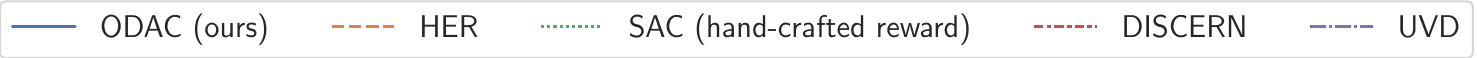}
     \caption{
     Learning curves showing final distance vs. environment steps across all six environments.
     Only \odac consistently performs well on all six tasks.
     Prior methods struggle to learn, especially in the absence of uniform goal sampling. See text for details.
     }
     \label{fig:sawyer-and-ant}
     \vspace*{-5pt}
\end{figure*}

\subsection{Ablation Study on the Effect of a Variational Discount Factor}

\setlength{\tabcolsep}{4.0pt}
\begin{wraptable}{R}{0.615\textwidth}
\vspace{-12pt}
    \centering
    \caption{
        Ablation results, showing mean final normalized distance ($\times 100$) at the end of training across 4 seeds.
        Best mean is in bold and standard error in parentheses.
        \odac is not sensitive to the dynamics models $\hat{p}_d$ but benefits from the dynamic $q_{T}$ variant.
    }
    \vspace{-6pt}
\begin{tabular}{l | c | c | c | c}
\toprule
Env & \odac & fixed $\hat{p}_d$ & fixed $q_{T}$ & fixed $q_{T}$, $\hat{p}_d$ \\
\hline
2D & 1.7 (1.20) & 1.2 (0.14) & {1.0} (0.24) & 1.3 (0.29) \\
Ant & {9} (0.48) & 11 (0.57) & 12 (0.41) & 13 (0.20) \\
Push & 35 (2.7) & {34} (1.5) & 37 (1.5) & 38 (3.1) \\
Fetch  & 19 (6) & {15} (3) & 53 (13) & 66 (15) \\
Window & 5.4 (0.62) & {5.0} (0.62) & 7.9 (0.71)  & 6.0 (0.12) \\
Faucet & {13} (4.2) & 15 (3.3) & 37 (8.3) & 38 (7.2)
\Bstrut\\
\hline
\hline
\end{tabular}
\label{tab:ablations}
\vspace*{-12pt}
\end{wraptable}

Next, we study the importance of the dynamic discount factor \mbox{$\qdeltatzero{t+1}$} and the sensitivity of our method to the dynamics model.
On all tasks, we evaluate the performance when the posterior exactly matches the prior, that is, \mbox{$\qdeltatzero{t+1} = 0.99$} (labeled ``fixed $q_{T}$'' in~\Cref{tab:ablations}).
Our analysis in~\Cref{appsec:alternative-posterior} suggests that this setting is sub-optimal, and this ablation empirically evaluate its benefits.
We also measure how the algorithm's performance depends on the accuracy of the learned dynamics model used for the reward in \odac.
To do this, we evaluate \odac with the dynamics model fixed to a multivariate Laplace distribution with a fixed variance, centered at the previous state (labeled ``fixed $\hat{p}_d$'' in~\Cref{tab:ablations}).
This ablation represents an extremely crude model, and good performance with such a model would indicate that our method does not depend on obtaining an particularly accurate model.

In~\Cref{tab:ablations}, we see that
fixing the distribution $q_{T}$ to the prior $p_{T}$ as described in Corollary~3 deteriorates performance, and that using a learned or fixed model both perform relatively well.
These results suggest that the derived optimal variational distribution \mbox{$q_{\Delta_{t+1}}^{\star}(\Delta_{t+1} = 0)$} given in~\Cref{prop:gc_optimal_variational_distribution_T} is better not only in theory but also in practice, and that \odac is not sensitive to the accuracy of the dynamics model.
Moreover, we note that a more expressive dynamics model---which could lead to a tighter variational bound---may not necessarily lead to a better variational policy if the functional form of the log-density under the dynamics model does not also provide favorable shaping.

In~\Cref{appsec:additional-exps}, we provide the full learning curves for this ablation study and present further experiments.
For example, we compare \odac to a variant in which we use the learned dynamics model for model-based planning.
We find that using the dynamics model only to compute rewards significantly outperforms the variant where it is used for both computing rewards and model-based planning.
This result suggests that \odac does not require learning a dynamics model that is accurate enough for planning, and that the derived Bellman updates are sufficient for obtaining policies that can achieve desired outcomes.
In~\Cref{fig:example-rollouts}, we also visualize $q_T$ and find that as the policy reaches an irrecoverable state, $\qdeltatzero{t+1}$ drops in value, suggesting that \odac \emph{automatically} learns a dynamic discount factor that terminates an episode when an irrecoverable state is reached.

\section{Conclusion}
\label{sec:conclusions}

We proposed a probabilistic approach for achieving desired outcomes in settings where no reward function and no termination condition are given.
We showed that by framing the problem of achieving desired outcomes as variational inference, we can derive an off-policy temporal-difference algorithm, a reward function learnable from environment interactions, and a novel Bellman backup that contains a state--action dependent dynamic discount factor for the reward and bootstrap term.

Our experimental results demonstrated that the resulting algorithm, \odac, leads to efficient outcome-driven approaches to RL.
While \odac requires choosing a dynamics model, we found that it works well even for simple dynamics models and believe that the use of more sophisticated dynamics models that incorporate epistemic uncertainty~\citep{chua18probabilisticdynamics} or domain-specific structure~\citep{ebert2017videoprediction,rudner2021fsvi,veerapaneni2020entity,zhang2019solar} is a promising avenue for future research.

\begin{ack}
We thank Marvin Zhang, Michael Janner, Abhishek Gupta, and various RAIL and OATML students for their discussions and feedback on early drafts of this paper.
Tim G. J. Rudner is funded by the Rhodes Trust, by a Qualcomm Innovation Fellowship, and by the Engineering and Physical Sciences Research Council (EPSRC).
This research was further supported by the Alan Turing Institute, the National Science Foundation, the DARPA Assured Autonomy Program, and ARL DCIST CRA W911NF-17-2-0181.
\end{ack}

\clearpage

\bibliography{references}
\bibliographystyle{plainnat}

\clearpage

\onecolumn

\begin{appendices}

\crefalias{section}{appsec}
\crefalias{subsection}{appsec}
\crefalias{subsubsection}{appsec}

\setcounter{equation}{0}
\renewcommand{\theequation}{\thesection.\arabic{equation}}

\small

\section*{\LARGE Supplementary Material}
\label{sec:appendix}

\section*{Table of Contents}
\vspace*{-10pt}
\startcontents[sections]
\printcontents[sections]{l}{1}{\setcounter{tocdepth}{2}}

\section{Proofs \& Derivations}
\label{appsec:proof_variational_objective}

\subsection{Finite- and Infinite-Horizon Variational Objectives}
\label{appsec:posterior_trajectory_inference}

In this section, we present detailed derivations and proofs for the results in Sections~\ref{sec:gc_probabilistic_framework_finite} and~\Cref{sec:gc_approximate_inference}.

\begin{customproposition}{1}{(Fixed-Time Outcome-Driven Variational Objective)}
\label{prop-app:maxent_objective_fixed_time}
Let $q_{\Trajnoso_{0:t} | \bS_{0}}(\trajnoso_{0:t} \vbar \bs_{0})$ be as defined in~\Cref{eq:warm-up-vi-family}.
Then, given any initial state $\bs_{0}$, termination time $t^\star$, and outcome $\bg$,
\begin{align}
\begin{split}
    \mathbb{D}_{\emph{\textrm{KL}}}( q_{\Trajnoso_{0:t} | \bS_{0}}(\cdot \vbar \bs_{0}) \,\|\, p_{\Trajnoso_{0:t} | \bS_{0}, \bS_{t^\star}}(\cdot \vbar \bs_{0}, \bg)
    )
    =
    \log p(\bg \vbar \bs_{0}) - \bar{\calF}(\pi, \bs_{0}, \bg),
\end{split}
\end{align}
where
\begin{align}
\begin{split}
\SwapAboveDisplaySkip
\label{eq-app:maxent_objective_fixed_time}
    \bar{\calF}(\pi, \bs_{0}, \bg)
    \defines
    \mathbb{E}_{q_{\Trajnoso_{0:t} | \bS_{0}}(\trajnoso_{0:t} \vbar \bs_{0})} \bigg[
    &\log \goaltransitiont{t} - \sum_{t'=0}^{t-1}
    \mathbb{D}_{\emph{\textrm{KL}}}(\policytdot{t'} \,||\, \policypriortdot{t'}) \bigg],
\end{split}
\end{align}
and since $\log p(\bg | \bs_{0})$ is constant in $\pi$,
\begin{align}
    \argmin_{\pi \in \Pi} \mathbb{D}_{\emph{\textrm{KL}}}( q_{\Trajnoso_{0:t} | \bS_{0}}(\cdot \vbar \bs_{0}) \,\|\, p_{\Trajnoso_{0:t} | \bS_{0}, \bS_{t^\star}}(\cdot \vbar \bs_{0}, \bg)
    )
    =
    \argmax_{\pi \in \Pi} \bar{\calF}(\pi, \bs_{0}, \bg).
\end{align}
\end{customproposition}
\begin{proof}
To find an approximation to the posterior $p_{\Trajnoso_{0:t} | \bS_{0}, \bS_{t^\star}}(\cdot \vbar \bs_{0}, \bg)$, we can use variational inference.
To do so, we consider the trajectory distribution under $p_{\Trajnoso_{0:t} | \bS_{0}, \bS_{t^\star}}(\cdot \vbar \bs_{0}, \bg)$, which by Bayes' Theorem is given by
\begin{align}
    p_{\Trajnoso_{0:t} | \bS_{0}, \bS_{t^\star}}(\trajnoso_{0:t} \vbar \bs_{0}, \bg)
    =
    \frac{\likelihoodt{t} p(\ba_{t} \vbar \bs_{t}) \prod_{t'=0}^{t-1} \dyn(\bs_{t'+1} \vbar \bs_{t}, \ba_{t}) p(\ba_{t'} \vbar \bs_{t'})}{p(\bg \vbar \bs_{0})} ,
\end{align}
where \mbox{$t=t^\star-1$}, and we denote the state--action trajectory realization from action $\ba_{0}$ to \mbox{$\ba_{t}$} by \mbox{$\trajnoso_{0:t} \defines \{ \ba_{0}, \bs_{1}, \ba_{1}, ..., \bs_{t}, \ba_{t} \}$}.
Inferring an approximation to the posterior distribution \mbox{$p_{\Trajnoso_{0:t} | \bS_{0}, \bS_{t^\star}}(\cdot \vbar \bs_{0}, \bg)$}
then becomes equivalent to finding a variational distribution $\qtrajnosodot$, which induces a trajectory distribution $q_{\Trajnoso_{0:t} | \bS_{0}}(\cdot \vbar \bs_{0})$ that minimizes the \kld from $q_{\Trajnoso_{0:t} | \bS_{0}}(\cdot \vbar \bs_{0})$ to $p_{\Trajnoso_{0:t} | \bS_{0}, \bS_{t^\star}}(\cdot \vbar \bs_{0}, \bg)$:
\begin{align}
\label{eq-app:kl_trajectories_fixed_t}
    \min_{q \in \bar{\calQ}} \DKL{q_{\Trajnoso_{0:t} | \bS_{0}}(\cdot \vbar \bs_{0})}{p_{\Trajnoso_{0:t} | \bS_{0}, \bS_{t^\star}}(\cdot \vbar \bs_{0}, \bg)}.
\end{align}
If we find a distribution $q_{\Trajnoso_{0:t} | \bS_{0}}(\cdot \vbar \bs_{0})$ for which the resulting \kld is zero, then $q_{\Trajnoso_{0:t} | \bS_{0}}(\cdot \vbar \bs_{0})$ is the exact posterior.
If the \kld is positive, then $q_{\Trajnoso_{0:t} | \bS_{0}}(\cdot \vbar \bs_{0})$ is an approximate posterior.
To solve the variational problem in~\Cref{eq-app:kl_trajectories_fixed_t}, we can define a factorized variational family
\begin{align}
    q_{\Trajnoso_{0:t} | \bS_{0}}(\trajnoso_{0:t} \vbar \bs_{0})
    \defines
    \policyt{t} \prod_{t'=0}^{t-1} q_{\bS_{t'+1} | \bS_{t'}, \bA_{t'}}(\bs_{t'+1 } \vbar \bs_{t'}, \ba_{t'}) \policyt{t'},
\end{align}
where $\bA_{0:t}$ and $\bS_{1:t}$ are latent variables over which to infer an approximate posterior distribution, and the product is from $t=0$ to $t=t^\star - 1$ to exclude the conditional distribution over the (observed) state $\bS_{t+1} = \bg$ from the variational distribution.

Returning to the variational problem in~\Cref{eq-app:kl_trajectories_fixed_t}, we can now write
\begin{align}
\begin{split}
    &\DKL{q_{\Trajnoso_{0:t} | \bS_{0}}(\cdot \vbar \bs_{0})}{p_{\Trajnoso_{0:t} | \bS_{0}, \bS_{t^\star}}(\cdot \vbar \bs_{0}, \bg)}
    \\
    &=
    \int_{\calA^{t+1}} \int_{\calS^{t}} q_{\Trajnoso_{0:t} | \bS_{0}}(\trajnoso_{0:t} \vbar \bs_{0}) \log \frac{q_{\Trajnoso_{0:t} | \bS_{0}}(\trajnoso_{0:t} \vbar \bs_{0})}{p_{\Trajnoso_{0:t} | \bS_{0}, \bS_{t^\star}}(\trajnoso_{0:t} | \bs_{0}, \bg )} d\bs_{1:t} d\ba_{0:t}
    \\
    &=
    -\bar{\calF}(\pi, \bs_{0}, \bg) + \log p(\bg | \bs_{0}),
\end{split}
\end{align}
where
\begin{align}
\begin{split}
    &\bar{\calF}(\pi, \bs_{0}, \bg)
    \\
    &\defines
    \mathbb{E}_{q_{\Trajnoso_{0:t} | \bS_{0}}(\trajnoso_{0:t} \vbar \bs_{0})}
    \Bigg[ \log \likelihoodt{t} + \log p(\ba_{t} \vbar \bs_{t}) - \log \policyt{t}
    \\
    &\qquad\quad
    + \sum_{t'=0}^{t-1} \log \policypriort{t'} + \log \transitiont{t'} - \log \policyt{t'} - \log q_{\bS_{t'+1} | \bS_{t'}, \bA_{t'}}(\bs_{t'+1} \vbar \bs_{t'}, \ba_{t'})
    \Bigg]
\end{split}
\end{align}
and
\begin{align}
    \log p(\bg | \bs_{0}) = \log \int_{\calA^{t+1}} \int_{\calS^{t}} \likelihoodt{t} p_{\Trajnoso_{0:t} | \bS_{0}}(\trajnoso_{0:t} \vbar \bs_{0}) d\bs_{1:t} d\ba_{0:t}
\end{align}
is a log-marginal likelihood.
Following~\citet{haarnoja2018sac}, we define the variational distribution over next states to be the the true transition dynamics, that is, \mbox{$q_{\bS_{t+1} | \bS_{t}, \bA_{t}}(\bs_{t+1} \vbar \bs_{t}, \ba_{t}) = \transitiont{t}$}, so that
\begin{align}
\label{eq-app:gc_variational_distribution_fixed_t}
    q_{\Trajnoso_{0:t} | \bS_{0}}(\trajnoso_{0:t} \vbar \bs_{0})
    \defines
    \policyt{t} \prod_{t'=0}^{t-1} \transitiont{t'} \policyt{t'}.
\end{align}
We can then simplify $\bar{\calF}(\pi, \bs_{0}, \bg)$ to
\begin{align}
    \bar{\calF}(\pi, \bs_{0}, \bg)
    =
    \mathbb{E}_{q_{\Trajnoso_{0:t} | \bS_{0}}(\trajnoso_{0:t} \vbar \bs_{0})} \left[
    \log \goaltransitiont{t} + \sum_{t'=0}^{t}
    \DKL{\policytdot{t'}}{\policypriortdot{t'}} \right].
\end{align}
Since $\log p(\bg | \bs_{0})$ is constant in $\pi$, solving the variational optimization problem in~\Cref{eq-app:kl_trajectories_fixed_t} is equivalent to maximizing the variational objective with respect to $\pi \in \Pi$, where $\Pi$ is a family of policy distributions.
\end{proof}

\begin{customcorollary}{1}{(Fixed-Time Outcome-Driven Reward Function)}
    \label{cor-app:maxent_finite_time}
    The objective in~\Cref{eq:maxent_objective_fixed_time} corresponds to KL-regularized reinforcement learning with a time-varying reward function given by
    \begin{align*}
        r(\bs_{t'}, \ba_{t'}, \bg, t')
        \defines
        \mathbb{I}\{ t' = t \} \log \goaltransitiont{t'} .
    \end{align*}
\end{customcorollary}
\begin{proof}
Let
\begin{align}
    r(\bs_{t'}, \ba_{t'}, \bg, t')
    \defines
    \mathbb{I}\{ t' = t \} \log \goaltransitiont{t'}
\end{align}
and note that the objective
\begin{align}
    \bar{\calF}(\pi, \bs_{0}, \bg)
    =
    \mathbb{E}_{\qtrajnosodot} \left[
    \log \goaltransitiont{t} + \sum_{t=0}^{t}
    \mathbb{D}_{\textrm{KL}}(\policytdot{t} \,\|\, \policypriortdot{t} ) \right]
\end{align}
can equivalently written as
\begin{align}
    \bar{\calF}(\pi, \bs_{0}, \bg)
    &=
    \mathbb{E}_{\qtrajnosodot} \left[
    \sum_{t'=0}^{t} r(\bs_{t'}, \ba_{t'}, \bg, t') + \sum_{t'=0}^{t}
    \DKL{\policytdot{t'}}{\policypriortdot{t'}} \right]
    \\
    &= \mathbb{E}_{\qtrajnosodot} \left[
    \sum_{t'=0}^{t} r(\bs_{t'}, \ba_{t'}, \bg, t') +
    \DKL{\policytdot{t'}}{\policypriortdot{t'}} \right],
\end{align}
which, as shown in~\citet{haarnoja2018sac}, can be written in the form of~\Cref{eq:recursive-q-def}.
\end{proof}

\begin{customproposition}{3}{(Unknown-time Outcome-Driven Variational Objective)}
\label{prop-app:objective_unknown_time_sum}
Let $q_{\Trajnoso_{0:T}, T | \bS_{0}}( \trajnoso_{0:t}, t \vbar \bs_{0}) = \qtrajnosotvar q_{T}(t)$, let $q_{T}(t)$ be a variational distribution defined on $t \in \mathbb{N}_{0}$, and let $\qtrajnosotvar$ be as defined in~\Cref{eq:warm-up-vi-family}.
Then, given any initial state $\bs_{0}$ and outcome $\bg$, we have that
\begin{align}
\begin{split}
\label{app-eq:gc_elbo_upper}
    \mathbb{D}_{\emph{\textrm{KL}}}( q_{\Trajnoso_{0:T}, T | \bS_{0}}( \cdot \vbar \bs_{0})
    \,\|\,
    p_{\Trajnoso_{0:T}, T | \bS_{0}, \bS_{T^\star}}( \cdot \vbar \bs_{0}, \bg)
    )
    =
    \log p(\bg | \bs_{0})
    -
    \calF(\pi, q_{T}, \bs_{0}, \bg),
\end{split}
\end{align}
where
\begin{align}
\begin{split}
\label{eq-app:F-definition}
    &\calF(\pi, q_{T}, \bs_{0}, \bg)
    \\
    &\defines\,
    \sum_{t=0}^\infty \qtvar \E_{\qtrajnosotvar} \Big[ \log \likelihoodt{t}
    - \mathbb{D}_{\emph{\textrm{KL}}}(q_{\Trajnoso_{0:T}, T | \bS_{0}}( \cdot \vbar \bs_{0}) \,||\, p_{\Trajnoso_{0:T}, T | \bS_{0}}( \cdot \vbar \bs_{0}) ) \Big]
\end{split}
\end{align}
and $\log p(\bg \vbar \bs_{0})$ is constant in $\pi$ and $q_{T}$.
\end{customproposition}
\begin{proof}
In general, solving the variational problem
\begin{align}
\label{eq-app:kl_2trajectories_unknown_t}
    \min_{q \in \calQ} \DKL{
    q_{\Trajnoso_{0:T}, T | \bS_{0}}( \cdot \vbar \bs_{0})
    }{
    p_{\Trajnoso_{0:T}, T | \bS_{0}, \bS_{T^\star}}( \cdot \vbar \bs_{0}, \bg)
    }
\end{align}
from~\Cref{sec:gc_approximate_inference} in closed form is challenging, but as in the fixed-time setting, we can take advantage of the fact that, by choosing a variational family parameterized by
\begin{align}
\label{eq-app:gc_variational_trajectory}
    \qtrajnosotvar \defines \policyt{t} \prod_{t'=0}^{t-1} \transitiont{t'} \, \policyt{t'},
\end{align}
with $\pi \in \Pi$, we can follow the same steps as in the proof for~\Cref{prop:maxent_objective_fixed_time} and show that given any initial state $\bs_{0}$ and outcome $\bg$,
\begin{align}
\begin{split}
    \DKL{ q_{\Trajnoso_{0:T}, T | \bS_{0}}( \cdot \vbar \bs_{0})
    }{
    p_{\Trajnoso_{0:T}, T | \bS_{0}, \bS_{T^\star}}( \cdot \vbar \bs_{0}, \bg)}
    )
    =
    \log p(\bg \vbar \bs_{0})
    -
    \calF(\pi, q_{T}, \bs_{0}, \bg)
    ,
\end{split}
\end{align}
where
\begin{align}
\begin{split}
\label{eq-app:gc_elbo_upper}
    &\calF(\pi, q_{T}, \bs_0, \bg)
    \\
    &\defines\,
    \sum_{t=0}^\infty q_{T}(t) \E_{\qtrajnosotvar} \Big[ \log \likelihoodt{t} - \DKL{q_{\Trajnoso_{0:T}, T | \bS_{0}}( \cdot \vbar \bs_{0})}{p_{\Trajnoso_{0:T}, T | \bS_{0}}( \cdot \vbar \bs_{0}) } \Big],
\end{split}
\end{align}
where $q_{\Trajnoso_{0:T}, T | \bS_{0}}(\trajnoso_{0:t}, t \vbar \bs_{0}) \defines \qtrajnosotvar q_{T}(t)$, and hence, solving the variational problem in~\Cref{eq:gc_kl_minimization} is equivalent to maximizing $ \calF(\pi, q_{T}, \bs_0, \bg)$ with respect to $\pi$ and $q_{T}$.
\end{proof}

\subsection{Recursive Variational Objective \& Outcome-Driven Bellman Backup Operator}
\label{appsec:variational_approximation}

\begin{customproposition}{4}{(Factorized Unknown-Time Outcome-Driven Variational Objective)}
\label{prop-app:objective_unknown_time_factorized}
Let $q_{\Trajnoso_{0:T}, T}( \trajnoso_{0:t}, t \vbar \bs_{0}) = \qtrajnosotvar q_{T}(t)$, let \mbox{$q_T(t) = \qdeltatone{t+1}  \prod_{t'=1}^{t} \qdeltatzero{t'}$} be a variational distribution defined on $t \in \mathbb{N}_{0}$, and let $\qtrajnosotvar$ be as defined in~\Cref{eq:warm-up-vi-family}.
Then, given any initial state $\bs_{0}$ and outcome $\bg$,~\Cref{eq-app:F-definition} can be rewritten as
\begin{align}
\begin{split}
\label{eq-app:objective_unknown_time_factorized}
    \calF(\pi, q_{T}, \bs_0, \bg)
    &=
    \E_{q_{\Trajnoso_{0} \vbar \bS_{0}}(\trajnoso_{0} \vbar \bs_{0})} \Bigg[ \sum_{{t} = 0}^\infty \left( \prod_{t'=1}^{t} \qdeltatzero{t'} \right) \Big(\rewardt{t} - \mathbb{D}_{\emph{\textrm{KL}}}(\policytdot{t} \,\|\,\policypriortdot{t}) \Big) \Bigg]
\end{split}
\end{align}
where
\begin{align}
    \rewardt{t} \defines \qdeltatone{t+1} \log \likelihoodt{t} - 
    \mathbb{D}_{\emph{\textrm{KL}}}(\qdeltatdot{t+1} \,\|\,\pdeltatdot{t+1}),
\end{align}
\end{customproposition}
\begin{proof}
Consider the variational objective $\calF(\pi, q_{T}, \bs_0, \bg)$ in~\Cref{eq-app:F-definition}:
\begin{align}
\begin{split}
    &\calF(\pi, q_{T}, \bs_0, \bg)
    \\
    &=
    \sum_{t=0}^\infty \qtvar
    \E_{\qtrajnosotvar} \Big[ \log \likelihoodt{t}
    - \DKL{q_{\Trajnoso_{0:T}, T | \bS_{0}}( \cdot \vbar \bs_{0})}{p_{\Trajnoso_{0:T}, T | \bS_{0}}( \cdot \vbar \bs_{0}) } \Big]
    \end{split}
    \\
    \begin{split}
    &=
    \sum_{{t} = 0}^\infty
    \qtvar
    \E_{\qtrajnosotvar} \left[ \log \likelihoodt{t}
    - \log \frac{\qtrajnosotvar \qtvar}{\ptrajnosotvar \ptvar}  d \trajnoso_{0:{{t}}}
    \right]
    \end{split}
    \\
    \begin{split}
    &=
    \sum_{{t} = 0}^\infty
    \qtvar
    \E_{\qtrajnosotvar} \left[ \log \likelihoodt{t}
    -  \log \frac{\qtrajnosotvar}{\ptrajnosotvar}
    \right] - \sum_{{t} = 0}^\infty \qtvar \log \frac{\qtvar}{\qtvar}.
    \end{split}
    \end{align}
Noting that $\sum_{{t} = 0}^\infty \qtvar \log \frac{\qtvar}{\qtvar} = \DKL{\qtdot}{\ptdot}$, we can write
\begin{align}
\begin{split}
    &\calF(\pi, q_{T}, \bs_0, \bg)
    \\
    &=
    \sum_{{t} = 0}^\infty
    \qtvar
    \E_{\qtrajnosotvar} \left[ \log \likelihoodt{t}
    -  \log \frac{\qtrajnosotvar}{\ptrajnosotvar}
    \right] - \DKL{\qtdot}{\ptdot}
    \end{split}
    \\
    \begin{split}
    &=
    \sum_{{t} = 0}^\infty
    \qtvar
    \E_{\qtrajnosotvar} \Big[
    \log \likelihoodt{t}
    \Big]
    \\
    &\qquad
    -
    \sum_{{t} = 0}^\infty
    \qtvar
    \E_{\qtrajnosotvar} \left[
    \log \frac{\qtrajnosotvar}{\ptrajnosotvar}
    \right] - \DKL{\qtdot}{\ptdot}.
    \label{eq-app:F-transdimensional}
    \end{split}
    \end{align}
Further noting that for an infinite-horizon trajectory distribution
\begin{align}
q_{\Trajnoso_{t'} | \bS_{t'}}(\trajnoso_{t'} \vbar \bs_{t'}) \defines \prod_{t=t'}^\infty \transitiont{t} \policyt{t},
\end{align}
trajectory realization \mbox{$\trajnoso_{t+1} \defines \{ \btau_{t'} \}_{t'=t+1}^\infty$}, and any joint probability density $f(\bs_{t}, \ba_{t})$,
\begin{align}
    &\sum_{{t} = 0}^\infty \qtvar \E_{\qtrajnosotvar} \Big[
    f(\bs_{t}, \ba_{t})
    \Big]
    \\
    &=
    \sum_{{t} = 0}^\infty \left( \int q_{\Trajnoso_{T+1} | \bS_{0}}(\trajnoso_{t+1} \vbar \bs_{0}) \left( \int_{\calS^{t} \times \calA^{t+1}} q_{\Trajnoso_{0:t} | \bS_{0}}(\trajnoso_{0:t} \vbar \bs_{0}) \qtvar f(\bs_{t}, \ba_{t}) d \trajnoso_{0:t} \right) d \trajnoso_{t+1} \right),
    \\
    &=
    \sum_{{t} = 0}^\infty \bigg( \E_{\qtrajnosotvar} \Big[
    \qtvar f(\bs_{t}, \ba_{t})
    \Big] \cdot \underbrace{ \left( \int q_{\Trajnoso_{T+1} | \bS_{0}}(\trajnoso_{t+1} \vbar \bs_{0}) d \trajnoso_{t+1} \right)}_{=1} \bigg)
    \\
    &=
    \sum_{{t} = 0}^\infty \bigg( \left(\int_{\calS^{t} \times \calA^{t+1}} q(\trajnoso_{0:t} \vbar \bs_{0}) \qtvar f(\bs_{t}, \ba_{t}) d \trajnoso_{0:t} \right) \cdot \underbrace{ \left( \int q_{\Trajnoso_{T+1} | \bS_{0}}(\trajnoso_{t+1} \vbar \bs_{0}) d \trajnoso_{t+1} \right)}_{=1} \bigg)
    \\
    &=
    \sum_{{t} = 0}^\infty \int q_{\Trajnoso_{0} | \bS_{0}}(\trajnoso_{0} \vbar \bs_{0}) \qtvar f(\bs_{t}, \ba_{t}) d \trajnoso_{0}
    \\
    &=
    \int q_{\Trajnoso_{0} | \bS_{0}}(\trajnoso_{0} \vbar \bs_{0}) \sum_{{t} = 0}^\infty \qtvar f(\bs_{t}, \ba_{t}) d \trajnoso_{0},
\end{align}
we can express~\Cref{eq-app:F-transdimensional} in terms of the infinite-horizon state--action trajectory \mbox{$q_{\Trajnoso_{0} | \bS_{0}}(\trajnoso_{0} \vbar \bs_{0}) \defines \prod_{t=0}^\infty \transitiont{t} \policyt{t}$} as
\begin{align}
\begin{split}
    &\calF(\pi, q_{T}, \bs_0, \bg)
    \\
    &=
    \int q_{\Trajnoso_{0} | \bS_{0}}(\trajnoso_{0} \vbar \bs_{0}) \sum_{{t} = 0}^\infty \qtvar \log p(\bg \vbar \bs_{{{t}}}, \ba_{{{t}}}) d \trajnoso
    \\
    &\qquad
    - \sum_{{t} = 0}^\infty \qtvar \DKL{\qtrajnosotdot}{\ptrajnosotdot} - \DKL{\qtdot}{\ptdot}
    \end{split}
    \\
    \begin{split}
    &=
    \E_{q_{\Trajnoso_{0} | \bS_{0}}(\trajnoso_{0} \vbar \bs_{0})} \bigg[ \sum_{{t} = 0}^\infty \qtvar \Big( \log p(\bg \vbar \bs_{{{t}}}, \ba_{{{t}}})
    \\
    &\qquad
    - \DKL{\qtrajnosotdot}{\ptrajnosotdot} \Big) \bigg] - \DKL{\qtdot}{\ptdot}.
    \end{split}
\end{align}
Using~\Cref{lemma-app:kl-decomp-delta} and the definition of $\qtvar$ in~\Cref{eq:qT-delta-form}, we can rewrite this objective as
\begin{align}
    \begin{split}
    &\calF(\pi, q_{T}, \bs_0, \bg)
    \\
    &=
    \E_{q_{\Trajnoso_{0} | \bS_{0}}(\trajnoso_{0} \vbar \bs_{0})} \bigg[ \sum_{{t} = 0}^\infty
    \Big(
    \prod_{t'=1}^{t} \qdeltatzero{t'}
    \Big)
    \qdeltatone{t'}
    \Big( \log p(\bg \vbar \bs_{{{t}}}, \ba_{{{t}}})
    \\
    &\qquad
    - \DKL{\qtrajnosotdot}{\ptrajnosotdot} \Big) \bigg] -
    \sum_{t=0}^\infty 
    \Big(
    \prod_{t'=1}^{t} \qdeltatzero{t'}
    \Big)
    \mathbb{D}_{{\textrm{KL}}}(\qdeltatdot{t+1} \,||\, \pdeltatdot{t+1})
    \end{split}
    \\
    \begin{split}
    &=
    \E_{q_{\Trajnoso_{0} | \bS_{0}}(\trajnoso_{0} \vbar \bs_{0})} \Bigg[ \sum_{{t} = 0}^\infty \left( \prod_{t'=1}^{t} q(\Delta_{t'}=0) \right)
    \\
    \label{eq-app:expanded-qT}
    &\qquad
    \cdot \Big(q(\Delta_{t+1}=1) \big( \log p(\bg \vbar \bs_{{{t}}}, \ba_{{{t}}}) - \DKL{\qtrajnosotdot}{\ptrajnosotdot} \big)
    \\
    &\qquad\qquad
    - \DKL{\qdeltatdot{t+1}}{\pdeltatdot{t+1}} \Big) \Bigg],
    \end{split}
\end{align}
with
\begin{align}
\begin{split}
    &\DKL{\qdeltatdot{t+1}}{\pdeltatdot{t+1}}
    \\
    &= \qdeltatzero{t+1} \log \frac{\qdeltatzero{t+1}}{\pdeltatzero{t+1}} + (1 - \qdeltatzero{t+1}) \log \frac{1 - \qdeltatzero{t+1}}{1 - \pdeltatzero{t+1}}.
\end{split}
\end{align}
Next, to re-express \mbox{$\DKL{\qtrajnosotdot}{\ptrajnosotdot}$} as a sum over Kullback-Leibler divergences between distributions over single action random variables, we note that
\begin{align}
\label{eq-app:traj-kl-given-t}
\begin{split}
    &\DKL{\qtrajnosotdot}{\ptrajnosotdot}
    \\
    &=
    \int_{\calS^{t} \times \calA^{t+1}} \qtrajnosotvar \log \frac{\qtrajnosotvar}{\ptrajnosotvar} d \trajnoso_{0:{{t}}}
    \end{split}
    \\
    \begin{split}
    &=
    \int_{\calS^{t} \times \calA^{t+1}} \qtrajnosotvar \log \frac{\prod_{t'=1}^{t} \policyt{t'}}{\prod_{t'=1}^{t} \policypriort{t'}} d \trajnoso_{0:{{t}}}
    \end{split}
    \\
    \begin{split}
    &=
    \int_{\calS^{t} \times \calA^{t+1}} \qtrajnosotvar \sum_{t'=0}^{t} \log \frac{ \policyt{t'}}{ \policypriort{t'}} d \trajnoso_{0:{{t}}}
    \end{split}
    \\
    \begin{split}
    &=
    \E_{q_{\Trajnoso_{0} | \bS_{0}}(\trajnoso_{0} \vbar \bs_{0} )} \left[ \sum_{t'=0}^{t} \int_{\calA}  \policyt{t'} \log \frac{ \policyt{t'}}{ \policypriort{t'}} d \ba_{t'} \right]
    \end{split}
    \\
    \begin{split}
    &=
    \E_{q_{\Trajnoso_{0} | \bS_{0}}(\trajnoso_{0} \vbar \bs_{0} )} \left[ \sum_{t'=0}^{t} \DKL{ \policytdot{t'}}{ \policypriortdot{t'}} \right] ,
\end{split}
\end{align}
where we have used the same marginalization trick as above to express the expression in terms of an infinite-horizon trajectory distribution, which allows us to express~\Cref{eq-app:expanded-qT} as
\begin{align}
\begin{split}
    &\calF(\pi, q_{T}, \bs_0, \bg)
    \\
    &=
    \E_{q_{\Trajnoso_{0} | \bS_{0}}(\trajnoso_{0} \vbar \bs_{0})} \Bigg[ \sum_{{t} = 0}^\infty \left( \prod_{t'=1}^{t} q(\Delta_{t'}=0) \right)
    \\
    &\qquad
    \cdot \bigg(q_{\Delta_{t+1}}(\Delta_{t+1}=1) \left( \log p(\bg \vbar \bs_{{{t}}}, \ba_{{{t}}}) - \E_{q_{\Trajnoso_{0} | \bS_{0}}(\trajnoso_{0} \vbar \bs_{0} )} \left[ \sum_{t'=0}^{t} \DKL{ \policytdot{t'}}{ \policypriortdot{t'}} \right] \right)
    \\
    &\qquad\qquad
    - \DKL{\qdeltatdot{t+1}}{\pdeltatdot{t+1}} \bigg) \Bigg].
\end{split}
\end{align}
Rearranging and dropping redundant expectation operators, we can now express the objective as
\begin{align}
    \begin{split}
    &\calF(\pi, q_{T}, \bs_0, \bg)
    \\
    &=
    \E_{q_{\Trajnoso_{0} | \bS_{0}}(\trajnoso_{0} \vbar \bs_{0})} \Bigg[ \sum_{{t} = 0}^\infty \left( \prod_{t'=1}^{t} q_{\Delta_{t+1}}(\Delta_{t'}=0) \right)
    \\
    &\qquad
    \cdot \bigg(q_{\Delta_{t+1}}(\Delta_{t+1}=1) \left( \log p(\bg \vbar \bs_{{{t}}}, \ba_{{{t}}}) - \E_{q_{\Trajnoso_{0} | \bS_{0}}(\trajnoso_{0} \vbar \bs_{0} )} \left[ \sum_{t'=0}^{t} \DKL{ \policytdot{t'}}{ \policypriortdot{t'}} \right] \right)
    \\
    &\qquad\qquad
    - \DKL{\qdeltatdot{t+1}}{\pdeltatdot{t+1}} \bigg) \Bigg].
    \end{split}
    \\
    \begin{split}
    &=
    \E_{q_{\Trajnoso_{0} | \bS_{0}}(\trajnoso_{0} \vbar \bs_{0})} \Bigg[ \sum_{{t} = 0}^\infty \left( \prod_{t'=1}^{t} q_{\Delta_{t'}}(\Delta_{t'}=0) \right)
    \\
    &\qquad
    \cdot \Big( q_{\Delta_{t+1}}(\Delta_{t+1}=1) \log p(\bg \vbar \bs_{{{t}}}, \ba_{{{t}}}) - \DKL{\qdeltatdot{t+1}}{\pdeltatdot{t+1}} \Big) \Bigg]
    \\
    &\qquad\qquad
    - \sum_{{t} = 0}^\infty \underbrace{\left( \prod_{t'=1}^{t} q_{\Delta_{t'}}(\Delta_{t'}=0) q_{\Delta_{t+1}}(\Delta_{t+1}=1) \right)}_{=\qtvar} \E_{q_{\Trajnoso_{0} | \bS_{0}}(\trajnoso_{0} \vbar \bs_{0} )} \left[ \sum_{t'=0}^{t} \DKL{ \policytdot{t'}}{ \policypriortdot{t'}} \right],
    \label{eq-app:expanded-qT-twice}
    \end{split}
\end{align}
whereupon we note that the negative term can be expressed as
\begin{align}
\begin{split}
    &\sum_{t=0}^\infty \qtvar \E_{q_{\Trajnoso_{0} | \bS_{0}}(\trajnoso_{0} \vbar \bs_{0} )} \left[ \sum_{t'=0}^{t} \DKL{\policytdot{t'}}{ \policypriortdot{t'}} \right]
    \\
    &=
    \E_{q_{\Trajnoso_{0} | \bS_{0}}(\trajnoso_{0} \vbar \bs_{0} )} \left[ \sum_{t=0}^\infty \sum_{t'=0}^{t} \qtvar \DKL{\policytdot{t'}}{ \policypriortdot{t'}} \right]
    \end{split}
    \\
    \begin{split}
    &=
    \E_{q_{\Trajnoso_{0} | \bS_{0}}(\trajnoso_{0} \vbar \bs_{0} )} \left[ \sum_{t=0}^{\infty} q(T \geq t) \DKL{\policytdot{t}}{ \policypriortdot{t}} \right]
    \end{split}
    \\
    \label{eq-app:simplified-action-kls}
    \begin{split}
    &=
    \E_{q_{\Trajnoso_{0} | \bS_{0}}(\trajnoso_{0} \vbar \bs_{0} )} \Bigg[ \sum_{t=0}^{\infty} \underbrace{ \left( \prod_{t'=1}^{t} \qdeltatzero{t'} \right) }_{\text{(by~\Cref{lemma-app:survival_q_delta})}} \DKL{\policytdot{t}}{ \policypriortdot{t}} \Bigg],
\end{split}
\end{align}
where the second line follows from expanding the sums and regrouping terms.
By substituting the expression in~\Cref{eq-app:simplified-action-kls} into~\Cref{eq-app:expanded-qT-twice}, we obtain an objective expressed entirely in terms of distributions over single-index random variables:
\begin{align}
    \begin{split}
    &\calF(\pi, q_{T}, \bs_{0}, \bg)
    \\
    &=
    \E_{q_{\Trajnoso_{0} | \bS_{0}}(\trajnoso_{0} \vbar \bs_{0})} \Bigg[ \sum_{{t} = 0}^\infty \left( \prod_{t'=1}^{t} \qdeltatzero{t'} \right)
    \\
    &\qquad
    \cdot \left(\qdeltatone{t+1} \log \likelihoodt{t} - \DKL{\qdeltatdot{t+1}}{\pdeltatdot{t+1}} - \DKL{\policytdot{t}}{ \policypriortdot{t}} \right) \Bigg]
    \end{split}
    \\
    &=
    \E_{q_{\Trajnoso_{0} | \bS_{0}}(\trajnoso_{0} \vbar \bs_{0})} \Bigg[ \sum_{{t} = 0}^\infty \left( \prod_{t'=1}^{t} \qdeltatzero{t'} \right) \Big(\rewardt{t} - \DKL{\policytdot{t}}{ \policypriortdot{t}} \Big) \Bigg],
\end{align}
where we defined
\begin{align}
    \rewardt{t} \defines \qdeltatone{t+1} \log \likelihoodt{t} - \DKL{\qdeltatdot{t+1}}{\pdeltatdot{t+1}},
\end{align}
which concludes the proof.
\end{proof}

\begin{customtheorem}{1}{(Outcome-Driven Variational Inference)}
\label{thm-app:gc_variational_problem}
    Let $q_{T}(t)$ and \mbox{$q_{\Trajnoso_{0:t} | T}(\trajnoso_{0:t} \vbar t, \bs_{0})$} be as defined in~\Cref{eq:warm-up-vi-family} and~\Cref{eq:qT-delta-form},
    and define
    \begin{align}
        V_{\text{{}}}^\pi(\bs_{t}, \bg; q_{T})
        &\defines \E_{\policyt{t}} \hspace{-1pt} \left[ Q_{\text{{}}}^{\pi}(\bs_{t}, \ba_{t}, \bg ; q_{T}) \right] - \mathbb{D}_{\emph{\textrm{KL}}}(\policytdot{t} \,\|\,\policypriortdot{t}),
        \label{eq-app:thm-v-defn}
        \\
        Q_{\text{{}}}^{\pi}(\bs_t, \ba_t, \bg; q_{T})
        &\defines
        \rewardt{t} +
        q(\Delta_{t+1} = 0)
        \E_{\dyn(\bs_{t+1} \vbar \bs_t, \ba_t)} \hspace{-1pt}\left[ V^\pi(\bs_{t+1}, \bg ; \pi, q_{T})
        \right],
        \label{eq-app:thm-q-defn}
        \\
        \rewardt{t}
        &\defines
        \qdeltatone{t+1} \log \goaltransitiont{t}
        - \mathbb{D}_{\emph{\textrm{KL}}}(\qdeltatdot{t+1} \,\|\, \pdeltatdot{t+1}).
    \end{align}
    Then given any initial state $\bs_{0}$ and outcome $\bg$,
    \begin{align*}
        \mathbb{D}_{\emph{\textrm{KL}}}( q_{\Trajnoso_{0:T}, T | \bS_{0}}( \cdot \vbar \bs_{0})
        \,\|\,
        p_{\Trajnoso_{0:T}, T | \bS_{0}, \bS_{T^\star}}( \cdot \vbar \bs_{0}, \bg)
        )
        =
        -\calF(\pi, q_{T}, \bs_0, \bg) + C
        =
        -V^\pi(\bs_0, \bg; q_{T}) + C,
    \end{align*}
    where $C \defines \log p(\bg \vbar \bs_{0})$ is independent of $\pi$ and $q_{T}$, and hence maximizing $V^\pi(\bs_{0}, \bg ; \pi, q_{T})$ is equivalent to
    minimizing~\Cref{eq:gc_kl_minimization}.
    In other words,
    \begin{align*}
        &\argmin_{\pi \in \Pi, q_{T} \in \calQ_{T}} \{
        \mathbb{D}_{\emph{\textrm{KL}}}( q_{\Trajnoso_{0:T}, T | \bs_{0}}( \cdot \vbar \bs_{0})
        \,\|\,
        p_{\Trajnoso_{0:T}, T | \bS_{0}, \bS_{T^\star}}( \cdot \vbar \bs_{0}, \bg)
        )
        \}
        \\
        =
        &\argmax_{\pi \in \Pi, q_{T} \in \calQ_{T}}\calF(\pi, q_{T}, \bs_0, \bg)
        \\
        =
        &\argmax_{\pi \in \Pi, q_{T} \in \calQ_{T}} V^\pi(\bs_0, \bg; q_{T}).
    \end{align*}
\end{customtheorem}
\begin{proof}
Consider the objective derived in Proposition~4,
\begin{align}
\begin{split}
    &\calF(\pi, q_{T}, \bs_{0}, \bg)
    \\
    &=
    \E_{q_{\Trajnoso_{0} | \bS_{0}}(\trajnoso_{0} \vbar \bs_{0})} \Bigg[ \sum_{{t} = 0}^\infty \left( \prod_{t'=1}^{t} \qdeltatzero{t'} \right)
    \\
    &\qquad
    \cdot \underbrace{\left(\qdeltatone{t+1} \log \likelihoodt{t} - \DKL{\qdeltatdot{t+1}}{\pdeltatdot{t+1}} \right)}_{\defines \, \rewardt{t}} - \DKL{ \policyt{t}}{ p(\ba_{{t}} | \bs_{{t}})} \Bigg],
\end{split}
\end{align}
and recall that, by~\Cref{prop:gc_optimal_variational_distribution_T},
\begin{align}
    \mathbb{D}_{{\textrm{KL}}}( q_{\Trajnoso_{0:T}, T | \bS_{0}}( \cdot \vbar \bs_{0})
        \,\|\,
        p_{\Trajnoso_{0:T}, T | \bS_{0}, \bS_{T^\star}}( \cdot \vbar \bs_{0}, \bg)
        )
        =
        -\calF(\pi, q_{T}, \bs_0, \bg) + \log p(\bg | \bs_{0}).
\end{align}
Therefore, to prove the result that
\begin{align*}
    \mathbb{D}_{{\textrm{KL}}}( q_{\Trajnoso_{0:T}, T | \bS_{0}}( \cdot \vbar \bs_{0})
    \,\|\,
    p_{\Trajnoso_{0:T}, T | \bS_{0}, \bS_{T^\star}}( \cdot \vbar \bs_{0}, \bg)
    )
    =
    -V^\pi(\bs_0, \bg; q_{T}) + \log p(\bg | \bs_{0}),
\end{align*}
we just need to show that \mbox{$\calF(\pi, q_{T}, \bs_0, \bg) = V_{\text{{}}}^\pi(\bs_{0}, \bg; q_{T})$} for $V_{\text{{}}}^\pi(\bs_{0}, \bg; q_{T})$ as defined in the theorem.
To do so, we start from the objective $\calF(\pi, q_{T}, \bs_0, \bg)$ and  and unroll it for $t=0$:
\begin{align}
\begin{split}
    &\calF(\pi, q_{T}, \bs_{0}, \bg)
    \\
    &=
    \E_{q_{\Trajnoso_{0} | \bS_{0}}(\trajnoso_{0} \vbar \bs_{0})} \Bigg[ \sum_{{t} = 0}^\infty \left( \prod_{t'=1}^{t} \qdeltatzero{t'} \right) \rewardt{t} - \DKL{ \policyt{t}}{ p(\ba_{{t}} | \bs_{{t}})} \Bigg]
    \end{split}
    \\
    \begin{split}
    &=
    \E_{\policyt{0}} \Bigg[\rewardt{0} +  \E_{q(\btau_{1} \vbar \bs_{0}, \ba_{0})} \Bigg[ \sum_{t = 1}^\infty  \prod_{t'=1}^{t} \qdeltatzero{t'} \Big( \rewardt{t}
    \\
    &\qquad
    - \DKL{ \policytdot{t}}{\policypriortdot{t}} \Big)
    \Bigg] \Bigg] - \DKL{ \policytdot{0}}{\policypriortdot{0}}.
    \end{split}
\end{align}
With this expression at hand, we now define
\begin{align}
\begin{split}
\label{eq-app:q-function_sum}
    &Q_{\text{{sum}}}^{\pi}(\bs_{{{0}}}, \ba_{{{0}}}, \bg; q_{T})
    \\
    &
    \defines
    r(\bs_{{{0}}}, \ba_{{{0}}}, \bg ; q_{\Delta}) +   \E_{q(\btau_{1} | \bs_{0}, \ba_{0})} \Bigg[ \sum_{t = {{1}}}^\infty \prod_{t'=1}^{t} \qdeltatzero{t'} \Big( \rewardt{t} - \DKL{\policytdot{t}}{\policypriortdot{t}} \Big)
    \Bigg],
\end{split}
\end{align}
and note that
\mbox{$\calF(\pi, q_{T}, \bs_{0}, \bg) = \E_{\policyt{0}} [Q_{\text{{sum}}}^{\pi}(\bs_{0}, \ba_{0}, \bg; q_{T})] - \mathbb{D}_{{\textrm{KL}}}(\policytdot{0} \,\|\,\policypriortdot{0}) = V_{\text{{}}}^\pi(\bs_{0}, \bg; q_{T})$}, as per the definition of $V_{\text{{}}}^\pi(\bs_{0}, \bg; q_{T})$.
To prove the theorem from this intermediate result, we now have to show that $Q_{\text{{sum}}}^{\pi}(\bs_{0}, \ba_{0}, \bg; q_{T})$ as defined in~\Cref{eq-app:q-function_sum} can in fact be expressed recursively as \mbox{$Q_{\text{{sum}}}^{\pi}(\bs_t, \ba_t, \bg; q_{T}) = Q_{\text{{}}}^{\pi}(\bs_{{{0}}}, \ba_{{{0}}}, \bg; q_{T})$} with
\begin{align}
    Q_{\text{{}}}^{\pi}(\bs_{{{0}}}, \ba_{{{0}}}, \bg; q_{T})
        &=
        \rewardt{t} +
        q(\Delta_{t+1} = 0)
        \E_{\dyn(\bs_{t+1} \vbar \bs_t, \ba_t)} \hspace{-1pt}\left[ V^\pi(\bs_{t+1}, \bg ; \pi, q_{T})
        \right].
\end{align}
To see that this is the case, first, unroll $Q_{\text{{}}}^{\pi}(\bs_{{{0}}}, \ba_{{{0}}}, \bg; q_{T})$ for $t=1$,
\begin{align}
    \begin{split}
        &Q_{\text{{sum}}}^{\pi}(\bs_{{{0}}}, \ba_{{{0}}}, \bg; q_{T})
        \\
        &~~
        =
        r(\bs_{{{0}}}, \ba_{{{0}}}, \bg ; q_{\Delta}) + \E_{q(\btau_{1} | \bs_{0}, \ba_{0})} \Bigg[ \sum_{t = {{1}}}^\infty \prod_{t'=1}^{t} \qdeltatzero{t'} \Big( \rewardt{t} - \DKL{\policytdot{t}}{\policypriortdot{t}} \Big)
        \Bigg]
        \end{split}
        \\
        \begin{split}
        &~~
        =
        r(\bs_{{{0}}}, \ba_{{{0}}}, \bg ; q_{\Delta}) + \E_{p_{d}(\bs_{1} | \ba_{0}, \ba_{0})} \Bigg[ \E_{q(\btau_{1} | \bs_{0}, \ba_{0})} \Bigg[ \sum_{t = {{1}}}^\infty \prod_{t'=1}^{t} \qdeltatzero{t'} \Big( \rewardt{t}
        \\
        &\qquad
        - \DKL{\policytdot{t}}{\policypriortdot{t}} \Big)
        \Bigg] \Bigg]
        \end{split}
        \\
        \begin{split}
        &~~
        =
        r(\bs_{{{0}}}, \ba_{{{0}}}, \bg ; q_{\Delta}) + \E_{p_{d}(\bs_{1} | \ba_{0}, \ba_{0})} \Bigg[ \E_{\policyt{1}} \Bigg[ \qdeltatzero{1} \left( \rewardt{1} - \DKL{\policytdot{1}}{\policypriortdot{1}} \right)
        \\
        &~~\qquad
        + \E_{q(\btau_{2} | \bs_{1}, \ba_{1})} \Bigg[ \sum_{t = {{2}}}^\infty \prod_{t'=2}^{t} \qdeltatzero{t'} \Big( \rewardt{t} - \DKL{\policytdot{t}}{\policypriortdot{t}} \Big) \Bigg]
        \Bigg] \Bigg] ,
    \end{split}
    \end{align}
    and note that we can rearrange this expression to obtain the recursive relationship
    \begin{align}
    \begin{split}
        &Q_{\text{{sum}}}^{\pi}(\bs_{{{0}}}, \ba_{{{0}}}, \bg; q_{T})
        \\
        &~~
        =
        r(\bs_{{{0}}}, \ba_{{{0}}}, \bg ; q_{\Delta}) + \qdeltatzero{1} \E_{\transitiont{0}} \Bigg[ - \DKL{\policytdot{1}}{\policypriortdot{1}}
        \\
        &~~\qquad
        + \E_{\policyt{1}} \Bigg[ \rewardt{1} + \E \Bigg[ \sum_{t = {{2}}}^\infty \left( \prod_{t'=2}^t \qdeltatzero{t'} \right) \Big( \rewardt{t}
        \\
        &\qquad\qquad
        - \DKL{\policytdot{t}}{\policypriortdot{}} \Big) \Bigg] 
    \Bigg] \Bigg],
    \end{split}
    \end{align}
    where the innermost expectation is taken with respect to $q(\btau_{2} | \bs_{1}, \ba_{1})$.
    With this result and noting that
    \begin{align}
    \begin{split}
        &Q_{\text{{sum}}}^{\pi}(\bs_{{{1}}}, \ba_{{{1}}}, \bg; q_{T})
        \\
        &=
        \rewardt{1} + \E \Bigg[ \sum_{t = {{2}}}^\infty \left( \prod_{t'=2}^t \qdeltatzero{t'} \right) \Big( \rewardt{t} - \DKL{\policytdot{t}}{\policypriortdot{}} \Big) \Bigg],
    \end{split}
    \end{align}
    where the expectation is again taken with respect to $q(\btau_{2} | \bs_{1}, \ba_{1})$, we see that
    \begin{align}
    \begin{split}
        &Q_{\text{{sum}}}^{\pi}(\bs_{{{0}}}, \ba_{{{0}}}, \bg; q_{T})
        \\
        &~~=
        r(\bs_{{{0}}}, \ba_{{{0}}}, \bg ; q_{\Delta}) + \qdeltatzero{1}   \E_{\transitiont{0}} \Big[ \E_{\pi(\ba_{{{1}}} | \bs_{{{1}}})} \left[ Q_{\text{{sum}}}^{\pi}(\bs_{1}, \ba_{{{1}}}, \bg; q_{T}) \right]
        \\
        &\qquad
        - \DKL{\policytdot{1}}{\policypriortdot{1}} \Big]
    \end{split}
    \\
    \begin{split}
    &~~=
    r(\bs_{{{0}}}, \ba_{{{0}}}, \bg ; q_{\Delta}) + \qdeltatzero{1} \E_{\dyn(\bs_{1} | \bs, \ba)} \Big[ V^\pi_{\text{{}}}(\bs_{{{1}}} , \bg ; q_{T}) \Big] ,
    \end{split}
\end{align}
for $V_{\text{{}}}(\bs_{t+1}, \bg ; q_{T})$ as defined above, as desired.
In other words, we have that
\begin{align}
    \calF(\pi, q_{T}, \bs_{0}, \bg) = \E_{\policyt{0}} [Q_{\text{{sum}}}^{\pi}(\bs_{0}, \ba_{0}, \bg; q_{T})] - \mathbb{D}_{{\textrm{KL}}}(\policytdot{0} \,\|\,\policypriortdot{0}) = V_{\text{{}}}^\pi(\bs_{0}, \bg; q_{T}).
\end{align}

Combining this result with~\Cref{prop:gc_optimal_variational_distribution_T} and Proposition~4, we finally conclude that
    \begin{align}
        \mathbb{D}_{{\textrm{KL}}}( q_{\Trajnoso_{0:T}, T | \bS_{0}}( \cdot \vbar \bs_{0})
        \,\|\,
        p_{\Trajnoso_{0:T}, T | \bS_{0}, \bS_{T^\star}}( \cdot \vbar \bs_{0}, \bg)
        )
        =
        -\calF(\pi, q_{T}, \bs_0, \bg) + C
        =
        -V^\pi(\bs_0, \bg; q_{T}) + C,
    \end{align}
    where $C \defines \log p(\bg \vbar \bs_{0})$ is independent of $\pi$ and $q_{T}$.
    Hence, maximizing $V^\pi(\bs_{0}, \bg ; \pi, q_{T})$ is equivalent to
    minimizing the objective in~\Cref{eq:gc_kl_minimization}.
    In other words,
    \begin{align}
    \begin{split}
        &\argmin_{\pi \in \Pi, q_{T} \in \calQ_{T}} \{
        \mathbb{D}_{{\textrm{KL}}}( q_{\Trajnoso_{0:T}, T | \bS_{0}}( \cdot \vbar \bs_{0})
        \,\|\,
        p_{\Trajnoso_{0:T}, T | \bS_{0}, \bS_{T^\star}}( \cdot \vbar \bs_{0}, \bg)
        )
        \}
        \\
        =
        &\argmax_{\pi \in \Pi, q_{T} \in \calQ_{T}}\calF(\pi, q_{T}, \bs_0, \bg)
        =
        \argmax_{\pi \in \Pi, q_{T} \in \calQ_{T}} V^\pi(\bs_0, \bg; q_{T}).
    \end{split}
    \end{align}
    This concludes the proof.
\end{proof}

\begin{customcorollary}{2}{(Fixed-Discount Outcome-Driven Variational Inference)}
\label{cor-app:gc_variational_problem_naive}
    Let \mbox{$q_{T} = p_{T}$}, assume that $p_{T}$ is a Geometric distribution with parameter $\gamma \in (0,1)$.
    Then the inference problem in~\Cref{eq:gc_kl_minimization} of finding a goal-directed variational trajectory distribution simplifies to maximizing the following recursively defined variational objective with respect to $\pi$:
    \begin{align}
    \begin{split}
        \bar{V}^\pi_{\text{{}}}(\bs_{0} , \bg ; \gamma)
        &\defines
        \E_{\policyt{0}} \left[ Q_{\text{{}}}(\bs_{0}, \ba_{0}, \bg ; \gamma ) \right] - \mathbb{D}_{\emph{\textrm{KL}}}(\policytdot{0} \,\|\, \policypriortdot{}{0})),
    \end{split}
    \end{align}
    where
    \begin{align}
    \begin{split}
        &\bar{Q}^\pi_{\text{{}}}(\bs_{0}, \ba_{0}, \bg ; \gamma) \defines\,
        (1 - \gamma) \log \goaltransitiont{0} + \gamma \E_{\dyn(\bs_{1} | \bs_{0}, \ba_{0})} \big[ V_{\text{{}}}(\bs_{1}, \bg ; \gamma) \big].
    \end{split}
    \end{align}
\end{customcorollary}
\begin{proof}
    The result follows immediately when replacing $q_{\Delta}$ in~\Cref{thm:gc_variational_problem} by $p_{\Delta}$ and noting that \mbox{$\DKL{p_{\Delta}}{p_{\Delta}}=0$}.
\end{proof}

\subsection{Optimal Variational Posterior over $T$}
\label{appsec:alternative-posterior}

\begin{customproposition}{2}{(Optimal Variational Distribution over $T$)}
\label{prop-app:gc_optimal_variational_distribution_T}
    The optimal variational distribution $q_{T}^\star$ with respect to~\Cref{eq:odac_variational_objective} is defined recursively in terms of $q_{\Delta_{t+1}}^\star({\Delta_{t+1}}=0) \forall t \in \mathbb{N}_{0}$ by
    \begin{align}
    \begin{split}
    \label{eq-app:optimal-qt}
        q_{\Delta_{t+1}}^\star(\Delta_{t+1} = 0; \pi, Q^\pi)
        =
        \sigmoid\left(\hspace*{-1pt} \Lambda(\bs_{t}, \pi, q_{T}, Q^\pi) + \sigmoid^{-1}\hspace{-1pt}\left(\pdeltatzero{t+1} \right)
        \hspace*{-1pt}\right)\hspace*{-2pt},
    \end{split}
    \end{align}
    where
    \begin{align*}
    \Lambda(\bs_{t}, \pi, q_{T}, Q^\pi)
    \defines
    \E_{\policyt{t+1} \transitiont{t} \policyt{t}}[ Q_{\text{{}}}^{\pi}(\bs_{t+1}, \ba_{t+1}, \bg; q_{T}) - \log \goaltransitiont{t}]
    \end{align*}
    and $\sigmoid(\cdot)$ is the sigmoid function, that is, \mbox{$\sigmoid(x) = \frac{1}{e^{-x}+1}$} and \mbox{$\sigmoid^{-1}(x) = \log \frac{x}{1-x}$}.
\end{customproposition}
\begin{proof}
Consider $\calF(\pi, q_{T}, \bs_{0}, \bg)$:
\begin{align}
\begin{split}
    \calF(\pi, q_{T}, \bs_{t}, \bg)
    &=
    \E_{\pi(\ba_{t} | \bs_{t})}[Q_{\text{{}}}^{\pi}(\bs_{{{t}}}, \ba_{{{t}}}, \bg; q_{T})]
    \\
    &=
    \E_{\pi(\ba_{t} | \bs_{t})}[ r(\bs_{t}, \ba_{t} , \bg ; q_{\Delta}) + \qdeltatzero{t+1} \E \big[ V_{\text{{}}}(\bs_{t+1}, \bg ; q_{T}) \big] ].
\end{split}
\end{align}
Since the variational objective $\calF(\pi, q_{T}, \bs_{t}, \bg)$ can be expressed recursively as
\begin{align*}
    V_{\text{{}}}^\pi(\bs_{t}, \bg; q_{T})
        &\defines
        \E_{\policyt{t}} \hspace{-1pt} \left[ Q_{\text{{}}}(\bs_{t}, \ba_{t}, \bg ; q_{T}) \right] - \DKL{\policytdot{t}}{\policypriortdot{t}},
\end{align*}
with
\begin{align*}
    Q_{\text{{}}}^{\pi}(\bs_t, \ba_t, \bg; q_{T})
    &=
    \rewardt{t}
    +
    \qdeltatzero{t+1}
    \E_{\transitiont{t}} \hspace{-1pt}\left[ V^\pi(\bs_{t+1}, \bg ; q_{T})
    \right],
    \\
    \rewardt{t}
    &=
    \qdeltatone{t+1} \log \goaltransitiont{t} - \DKL{\qdeltatdot{t+1}}{\pdeltatdot{t+1}},
\end{align*}
and since $\DKL{\qdeltatdot{t+1}}{\pdeltatdot{t+1}}$ is strictly convex in $\qdeltatzero{t+1}$, we can find the globally optimal Bernoulli distribution parameters $\qdeltatzero{t+1}$ for all $t \in \mathbb{N}_{0}$ recursively.
That is, it is sufficient to solve the problem
\begin{align}
    q_{\Delta_{t+1}}^\star(\Delta_{t+1} = 0) \defines \argmax_{\qdeltatzero{t+1} }\left\{ \calF(\pi, q_{T}, \bs_0, \bg) \right\} = \argmax_{\qdeltatzero{t+1} }\left\{ \calF(\pi, q_{\Delta_1}, \dots, \qdeltatdot{t+1}, \dots, \bs_0, \bg) \right\}
\end{align}
for a fixed $t+1$.
To do so, we take the derivative of $\calF(\pi, q_{\Delta_1}, \dots, \qdeltatdot{t+1}, \dots, \bs_0, \bg)$, which---defined recursively---is given by
\begin{align}
\begin{split}
    &\E_{\policyt{t}} \hspace{-1pt} \left[ Q_{\text{{}}}(\bs_{t}, \ba_{t}, \bg ; q_{T} \right] - \DKL{\policytdot{t}}{\policypriortdot{t}}
    \\
    =&
    \E_{\policyt{t}} \hspace{-1pt} \left[ \rewardt{t}
    +
    \qdeltatzero{t+1}
    \E_{\transitiont{t}} \hspace{-1pt}\left[ V^\pi(\bs_{t+1}, \bg ; q_{T})
    \right]
    \right]
    \\
    &\qquad
    \qquad
    - \DKL{\policytdot{t}}{\policypriortdot{t}}
    \end{split}
    \\
    \begin{split}
    =&
    \E_{\policyt{t}} \hspace{-1pt} \bigg[ \qdeltatone{t+1} \log \goaltransitiont{t} - \DKL{\qdeltatdot{t+1}}{\pdeltatdot{t+1}}
    \\
    &\qquad\qquad
    + \qdeltatzero{t+1}
    \E_{\transitiont{t}} \hspace{-1pt}\left[ V^\pi(\bs_{t+1}, \bg ; q_{T})
    \right]
    \bigg] - \DKL{\policytdot{t}}{\policypriortdot{t}}
    \end{split}
    \\
    \begin{split}
    =&
    \E_{\policyt{t}} \hspace{-1pt} \bigg[ ( 1 - \qdeltatzero{t+1} ) \log \goaltransitiont{t} - \DKL{\qdeltatdot{t+1}}{\pdeltatdot{t+1}}
    \\
    &\qquad\qquad
    + \qdeltatzero{t+1}
    \E_{\transitiont{t}} \hspace{-1pt}\left[ V^\pi(\bs_{t+1}, \bg ; q_{T})
    \right]
    \bigg] - \DKL{\policytdot{t}}{\policypriortdot{t}},
    \end{split}
\end{align}
with respect to $\qdeltatzero{t+1}$ and set it to zero, which yields
\begin{align}
\begin{split}
    0
    &=
    - \E_{\policyt{t}} \hspace{-1pt} \left[ \log \likelihoodt{t} + \E_{\policyt{t+1} \transitiont{t}} [Q_{\text{{}}}^{\pi}(\bs_{t+1}, \ba_{t+1}, \bg; q_{T})] \right]
    \\
    &\qquad\quad  + \log \frac{1-q_{\Delta_{t+1}}^\star(\Delta_{t+1} = 0)}{1-\pdeltatzero{t+1}} - \log \frac{q_{\Delta_{t+1}}^\star(\Delta_{t+1} = 0)}{\pdeltatzero{t+1}}.
\end{split}
\end{align}

Rearranging, we get
\begin{align}
\begin{split}
    \frac{q_{\Delta_{t+1}}^\star(\Delta_{t+1} = 0)}{1 - q_{\Delta_{t+1}}^\star(\Delta_{t+1} = 0)}
    =
    \exp\left( \E [Q_{\text{{}}}^{\pi}(\bs_{t+1}, \ba_{t+1}, \bg; q_{T}) - \log \likelihoodt{t} ] + \log \frac{\pdeltatzero{t+1}}{1 - \pdeltatzero{t+1}} \right),
\end{split}
\end{align}
where the expectation is taken with respect to $\policyt{t+1} \transitiont{t}\policyt{t}$ and the $Q$-function depends on $q(\Delta_{t'})$ with $t' > t$, but not on \mbox{$q_{\Delta_{t+1}}^\star(\Delta_{t+1} = 0)$}.
Solving for \mbox{$q_{\Delta_{t+1}}^\star(\Delta_{t+1} = 0)$}.
Solving for $q_{\Delta_{t+1}}^\star(\Delta_{t+1} = 0)$, we obtain
\begin{align}
    \begin{split}
    &q_{\Delta_{t+1}}^\star(\Delta_{t+1} = 0)
    \\
    &~~
    = \frac{\exp(\E_{p_{\pi p_{d}} \policyt{t}} [Q_{\text{{}}}^{\pi}(\bs_{t+1}, \ba_{t+1}, \bg; q_{T}) - \log \likelihoodt{t} ] + \log \frac{\pdeltatzero{t+1}}{1 - \pdeltatzero{t+1}})}{1 + \exp(\E_{p_{\pi p_{d}} \policyt{t}} [Q_{\text{{}}}^{\pi}(\bs_{t+1}, \ba_{t+1}, \bg; q_{T}) - \log \likelihoodt{t} ] + \log \frac{\pdeltatzero{t+1}}{1 - \pdeltatzero{t+1}})}
    \end{split}
    \\
    &~~
    =\sigmoid\Big(\E_{p_{\pi p_{d}}} [ Q_{\text{{}}}^{\pi}(\bs_{t+1}, \ba_{t+1}, \bg; q_{T})] - 
    \E_{\policyt{t}}[\log \likelihoodt{t}] + \sigmoid^{-1}\left(\pdeltatzero{t+1}\right)\Big),
\end{align}
where $p_{\pi p_{d}} \defines \policyt{t+1} \transitiont{t}$, $\sigmoid(\cdot)$ is the sigmoid function with \mbox{$\sigmoid(x) = \frac{1}{e^{-x}+1}$} and \mbox{$\sigmoid^{-1}(x) = \log \frac{x}{1-x}$}.
This concludes the proof.
\end{proof}
\begin{remark}
As can be seen from~\Cref{prop:gc_optimal_variational_distribution_T}, the optimal approximation to the posterior over $T$ trades off short-term rewards via $\E_{\policyt{t}} [r(\bs_{t}, \ba_{t}, \bg; q_{\Delta})]$, long-term rewards via $\E_{\policyt{t+1} \transitiont{t}} [Q_{\text{{}}}^{\pi}(\bs_{t+1}, \ba_{t+1}, \bg; q_{T})]$, and the prior log-odds of not achieving the outcome at a given point in time conditioned on the outcome not having been achieved yet, $\frac{\pdeltatzero{t+1}}{1-\pdeltatzero{t+1}}$.
\end{remark}

\subsection{Outcome-Driven Policy Iteration}
\label{appsec:proof_policy_iteration_theorem}

\begin{customtheorem}{2}{(Variational Outcome-Driven Policy Iteration)}
    Assume $|\calA| <\infty$ and that the MDP is ergodic.
    \begin{enumerate}[leftmargin=15pt]
    \item Outcome-Driven Policy Evaluation (ODPE): Given policy $\pi$ and a function $Q^{0}: \calS \times \calA \times \calS \rightarrow \R$, define $Q^{i+1}_{\text{{}}} = \calT^{\pi} Q^{i}_{\text{{}}}$. Then the sequence $Q^{i}_{\text{{}}}$ converges to the lower bound in~\Cref{thm:gc_variational_problem}.
    \item Outcome-Driven Policy Improvement (ODPI): The policy
    \begin{align}
        \pi^+
        = \argmax_{\pi' \in \Pi} \left\{ \E_{\pi'(\ba_{t} \vbar \bs_{t})} \left[ Q^{\pi}(\bs_{t}, \ba_{t}, \bg; q_{T}) \right] - \mathbb{D}_{\emph{\textrm{KL}}}(\pi'( \cdot \vbar \bs_{t}) \,||\, \policypriortdot{t}) \right\}
    \end{align}
    and the variational distribution over $T$ recursively defined in terms of
    \begin{align}
    \begin{split}
        &q^+(\Delta_{t+1} = 0 \vbar \bs_{0} ; \pi, Q^\pi)
        \\
        &~~=
        \sigmoid\Big( \E_{\policyt{t+1} \transitiont{t}}[ Q_{\text{{}}}^{\pi}(\bs_{t+1}, \ba_{t+1}, \bg; q_{T}) ] - \E_{\policyt{t}}[ \log \likelihoodt{t} ]
        \\
        &\qquad\qquad
        + \sigmoid^{-1}\left(\pdeltatzero{t+1}\right) \Big)
    \end{split}
    \end{align}
    improve the variational objective.
    In other words, $\calF(\pi^+, q_T, \bs_0) \geq \calF(\pi, q_T, \bs_0)$ and $\calF(\pi, q_T^+, \bs_0) \geq \calF(\pi, q_T, \bs_0)$ for all $\bs_0 \in \calS$.
    \item Alternating between ODPE and ODPI converges to a policy $\pi^\star$ and a variational distribution over $T$, $q_{T}^\star$, such that $\Qpgc^{\pi^\star}(\bs, \ba, \bg; q_{T}^\star) \geq \Qpgc^{\pi}(\bs, \ba, \bg; q_{T})$ for all $(\pi, q_{T}) \in \Pi \times \calQ_{T}$ and any \mbox{$(\bs, \ba) \in \calS \times \calA$}.
    \end{enumerate}
\end{customtheorem}
\begin{proof}
Parts of this proof are adapted from the proof given in~\citet{haarnoja2018sac}, modified for the Bellman operator proposed in~\Cref{def:gc_bellman_operator}.
\begin{enumerate}[leftmargin=15pt]
\item Outcome-Driven Policy Evaluation (ODPE):
Instead of absorbing the entropy term into the $Q$-function, we can define an entropy-augmented reward as
\begin{align}
\begin{split}
    r^\pi(\bs_{t}, \ba_{t}, \bg; q_{\Delta}) &\defines \qdeltatone{t+1} \log \goaltransitiont{t} - \DKL{\qdeltatdot{t+1}}{\pdeltatdot{t+1}}
    \\
    &\qquad
    + \qdeltatzero{t+1} \E_{\transitiont{t}} [ \DKL{\policytdot{t+1}}{\policypriortdot{t+1}} ] .
\end{split}
\end{align}
We can then write an update rule according to~\Cref{def:gc_bellman_operator} as
\begin{align}
\begin{split}
    \tilde{Q}(\bs_{t}, \ba_{t}, \bg; q_{T}) &\leftarrow r^\pi(\bs_{t}, \ba_{t}, \bg; q_{\Delta}) \\
    &\qquad
    + \qdeltatzero{t+1} \E_{\policyt{t+1} \transitiont{t}} [\tilde{Q}(\bs_{t+1}, \ba_{t+1}, \bg; q_{T})],
\end{split}
\end{align}
where $\qdeltatzero{t+1} \leq 1$.
This update is similar to a Bellman update~\citep{sutton1998rl}, but with a discount factor given by $\qdeltatzero{t+1}$.
In general, this discount factor $\qdeltatzero{t+1}$ can be computed dynamically based on the current state and action, such as in~\Cref{eq:optimal-qt}.
As discussed in~\citet{white2017unifying}, this Bellman operator is still a contraction mapping so long as the Markov chain induced by the current policy is ergodic and there exists a state such that $\qdeltatzero{t+1} < 1$. 
The first condition is true by assumption.
The second condition is true since $\qdeltatzero{t+1}$ is given by~\Cref{eq:optimal-qt}, which is always strictly between $0$ and $1$.
Therefore, we apply convergence results for policy evaluation with transition-dependent discount factors~\citep{white2017unifying} to this contraction mapping, and the result immediately follows.

\item
Outcome-Driven Policy Improvement (ODPI):
Let $\pi_{\text{old}} \in \Pi$ and let $Q^{\pi_{\text{old}}}$ and $V^{\pi_{\text{old}}}$ be the outcome-driven state and state-action value functions from~\Cref{def:gc_bellman_operator}, let $q_{T}$ be some variational distribution over $T$, and let $\pi_{\text{new}}$ be given by
\begin{align}
    \pi_{\text{new}}(\ba_{t} | \bs_{t})
    &= \argmax_{\pi' \in \Pi} \left\{ \E_{\pi'(\ba_{t} \vbar \bs_{t})} \left[ Q^{\pi_{\text{old}}}(\bs_{t}, \ba_{t}, \bg; q_{T}) \right] - \DKL{\pi'( \cdot \vbar \bs_{t})}{\policypriortdot{t}} \right\}
    \\
    &= \argmax_{\pi' \in \Pi} \calJ_{\pi_{\text{old}}}(\pi'(\ba_{t}, \bs_{t}), q_{T}).
\end{align}
Then, it must be true that \mbox{$\calJ_{\pi_{\text{old}}}(\pi_{\text{old}}(\ba_{t} | \bs_{t}); q_{T}) \leq \calJ_{\pi_{\text{old}}}(\pi_{\text{new}}(\ba_{t} | \bs_{t}); q_{T})$}, since one could set \mbox{$\pi_{\text{new}} = \pi_{\text{old}} \in \Pi$}.
Thus,
\begin{align}
\begin{split}
    &\E_{\pi_{\text{new}}(\ba_{t} | \bs_{t})} \left[ Q^{\pi_{\text{old}}}(\bs_{t}, \ba_{t}, \bg; q_{T}) \right] - \DKL{\pi_{\text{new}}( \cdot \vbar \bs_{t})}{\policypriortdot{t}}
    \\
    &\qquad \geq \E_{\pi_{\text{old}}(\ba_{t} | \bs_{t})} \left[ Q^{\pi_{\text{old}}}(\bs_{t}, \ba_{t}, \bg; q_{T}) \right] - \DKL{\pi_{\text{old}}( \cdot \vbar \bs_{t})}{\policypriortdot{t}},
\end{split}
\end{align}
and since
\begin{align}
    V^{\pi_{\text{old}}}(\bs_{t}, \bg; q_{T}) = \E_{\pi_{\text{old}}(\ba_{t} | \bs_{t})} \left[ Q^{\pi_{\text{old}}}(\bs_{t}, \ba_{t}, \bg; q_{T}) \right] - \DKL{\pi_{\text{old}}( \cdot \vbar \bs_{t})}{\policypriortdot{t}},
\end{align}
we get
\begin{align}
\label{eq-app:bound_Q}
    \E_{\pi_{\text{new}}(\ba_{t} | \bs_{t})} \left[ Q^{\pi_{\text{old}}}(\bs_{t}, \ba_{t}, \bg; q_{T}) \right] - \DKL{\pi_{\text{new}}( \cdot \vbar \bs_{t})}{\policypriortdot{t}} \geq V^{\pi_{\text{old}}}(\bs_{t}, \bg; q_{T}).
\end{align}
We can now write the Bellman equation as
\begin{align}
    \begin{split}
    Q^{\pi_{\text{old}}}&(\bs_{t}, \ba_{t}, \bg; q_{T})
    \\
    &= \qdeltatone{t+1} \log \goaltransitiont{t} + \qdeltatzero{t+1} \E_{\transitiont{t}} [ V^{\pi_{\text{old}}}(\bs_{t+1}, \bg; q_{T}) ]
    \end{split}
    \\
    \begin{split}
    &\leq \qdeltatone{t+1} \log \goaltransitiont{t}
    \\
    &\qquad
    + \qdeltatzero{t+1} \E_{p(\bs_{t'} | \bs_{t}, \ba_{t})} [  \E_{\pi_{\text{new}}(\ba_{t'} | \bs_{t'})} \left[ Q^{\pi_{\text{old}}}(\bs_{t'}, \ba_{t'}, \bg; q_{T}) \right]
    \\
    &\qquad\qquad
    - \DKL{\pi_{\text{new}}( \cdot \vbar \bs_{t'})}{\policypriortdot{t'}} ],
    \end{split}
    \\
    &\,\,\,\vdots \nonumber
    \\
    &\leq Q^{\pi_{\text{new}}}(\bs_{t}, \ba_{t}, \bg; q_{T})
\end{align}
where we defined $t' \defines t+1$, repeatedly applied the Bellman backup operator defined in~\Cref{def:gc_bellman_operator} and used the bound in~\Cref{eq-app:bound_Q}.
Convergence follows from Outcome-Driven Policy Evaluation above.

\item
Locally Optimal Variational Outcome-Driven Policy Iteration:
Define $\pi^{i}$ to be a policy at iteration $i$.
By ODPI for a given $q_{T}$, the sequence of state-action value functions $\{ Q^{\pi^{i}}(q_{T}) \}_{i=1}^\infty$ is monotonically increasing in $i$.
Since the reward is finite and the negative \kld
is upper bounded by zero, $Q^{\pi}(q_{T})$ is upper bounded for $\pi \in \Pi$ and the sequence $\{ \pi^{i} \}_{i=1}^\infty$ converges to some $\pi^\star$.
To see that $\pi^\star$ is an optimal policy, note that it must be the case that $\calJ_{\pi^\star}(\pi^\star(\ba_{t} | \bs_{t}); q_{T}) > \calJ_{\pi^\star}(\policyt{t}; q_{T})$ for any $\pi \in \Pi$ with $\pi \neq \pi^\star$.
By the argument used in ODPI above, it must be the case that the outcome-driven state-action value of the converged policy is higher than that of any other non-converged policy in $\Pi$, that is, $Q^{\pi^\star}(\bs_{t}, \ba_{t}; q_{T}) > Q^{\pi}(\bs_{t}, \ba_{t}; q_{T})$ for all \mbox{$\pi \in \Pi$} and any $q_{T}^i \in \calQ_{T}$ and \mbox{$(\bs, \ba) \in \calS \times \calA$}.
Therefore, given $q_{T}$, $\pi^\star$ must be optimal in $\Pi$, which concludes the proof.

\item
Globally Optimal Variational Outcome-Driven Policy Iteration:
Let $\pi^{i}$ be a policy and let $q_T^i$ be variational distributions over $T$ at iteration $i$.
By Locally Optimal Variational Outcome-Driven Policy Iteration, for a \emph{fixed} $q_{T}^i$ with \mbox{$q_{T}^i = q_{T}^j \forall i,j\in \mathbb{N}_{0}$}, the sequence of $\{(\pi^i, q_T^i)\}_{i=1}^\infty$ increases the objective~\Cref{app-eq:gc_elbo_upper} at each iteration and converges to a stationary point in $\pi^i$, where $Q^{\pi^\star}(\bs_{t}, \ba_{t}; q_{T}^i) > Q^{\pi}(\bs_{t}, \ba_{t}; q_{T}^i)$ for all \mbox{$\pi \in \Pi$} and any $q_{T}^i \in \calQ_{T}$ and \mbox{$(\bs, \ba) \in \calS \times \calA$}.
Since the objective in~\Cref{app-eq:gc_elbo_upper} is concave in $q_{T}$, it must be the case that for, $q_{T}^{\star^i} \in \calQ_{T}$, the optimal variational distribution over $T$ at iteration $i$, defined recursively by
\begin{align*}
    q^{\star^i}(\Delta_{t+1}
    = 0 ; \pi^i, Q^{\pi^i})
    =
    \sigmoid \Big(
    & \E_{\policyt{t+1} \transitiont{t}}[ Q_{\text{{}}}^{\pi^i}(\bs_{t+1}, \ba_{t+1}, \bg; q_{T}(\pi^i, Q^{\pi^i})) ]
    \\
    &\qquad
    - \E_{\policyt{t}}[ \log \likelihoodt{t} ]
    + \sigmoid^{-1}(\pdeltatzero{t+1}
    \Big),
\end{align*}
for $t \in \mathbb{N}_{0}$, $Q^{\pi}(\bs_{t}, \ba_{t}; q_{T}^\star) > Q^{\pi}(\bs_{t}, \ba_{t}; q_{T})$ for all \mbox{$\pi \in \Pi$} and any \mbox{$(\bs, \ba) \in \calS \times \calA$}.
Note that $q_{T}$ is defined implicitly in terms of $\pi^i$ and $Q^{\pi^i}$, that is, the optimal variational distribution over $T$ at iteration $i$ is defined as a function of the policy and $Q$-function at iteration $i$.
Hence, it must then be true that for $Q^{\pi^\star}(\bs_{t}, \ba_{t}; q_{T}^\star) > Q^{\pi^\star}(\bs_{t}, \ba_{t}; q_{T})$ for all $q_{T}^\star(\pi^\star, Q^{\pi^\star}) \in \calQ_{T}$ and for any \mbox{$\pi^\star \in \Pi$} and \mbox{$(\bs, \ba) \in \calS \times \calA$}.
In other words, for an optimal policy and corresponding $Q$-function, there exists an optimal variational distribution over $T$ that maximizes the $Q$-function, given the optimal policy.
Repeating locally optimal variational outcome-driven policy iteration under the new variational distribution $q_{T}^\star(\pi^\star, Q^{\pi^\star})$ will yield an optimal policy $\pi^{\star\star}$ and computing the corresponding optimal variational distribution, $q_{T}^{\star\star}(\pi^{\star\star}, Q^{\pi^{\star\star}})$ will further increase the variational objective such that for $\pi^{\star\star}) \in \Pi$ and $q_{T}^{\star\star}(\pi^{\star\star}, Q^{\pi^{\star\star}}) \in \calQ_{T}$, we have that
\begin{align}
     Q^{\pi^{\star\star}}(\bs_{t}, \ba_{t}; q_{T}^{\star\star}) > Q^{\pi^{\star\star}}(\bs_{t}, \ba_{t}; q_{T}^{\star}) > Q^{\pi^\star}(\bs_{t}, \ba_{t}; q_{T}^\star) > Q^{\pi^\star}(\bs_{t}, \ba_{t}; q_{T})
\end{align}
for any \mbox{$\pi^\star \in \Pi$} and \mbox{$(\bs, \ba) \in \calS \times \calA$}.
Hence, global optimal variational outcome-driven policy iteration increases the variational objective at every step.
Since the objective is upper bounded (by virtue of the rewards being finite and the negative \kld being upper bounded by zero) and the sequence of $\{(\pi^i, q_T^i)\}_{i=1}^\infty$ increases the objective~\Cref{app-eq:gc_elbo_upper} at each iteration, by the monotone convergence theorem, the objective value converges to a supremum and since the objective function is concave the supremum is unique.
Hence, since the supremum is unique and obtained via global optimal variational outcome--driven policy iteration on $(\pi, q_{T}) \in \Pi \times \calQ_{T}$, the sequence of $\{(\pi^i, q_T^i)\}_{i=1}^\infty$ converges to a unique stationary point $(\pi^\star, q_{T}^\star) \in \Pi \times \calQ_{T}$, where $Q^{\pi^\star}(\bs_{t}, \ba_{t}; q_{T}^\star) > Q^{\pi}(\bs_{t}, \ba_{t}; q_{T}^i)$ for all \mbox{$\pi \in \Pi$} and any $q_{T}^i \in \calQ_{T}$ and \mbox{$(\bs, \ba) \in \calS \times \calA$}.
\end{enumerate}
\end{proof}

\begin{customcorollary}{3}{(Optimality of Variational Outcome Driven Policy Iteration)}
Variational Outcome-Driven Policy Iteration on $(\pi, q_{T}) \in \Pi \times \calQ_{T}$ results in an optimal policy at least as good or better than any optimal policy attainable from policy iteration on $\pi \in \Pi$ alone.
\end{customcorollary}

\begin{remark}
The convergence proof of ODPE assumes a transition-dependent discount factor~\citep{white2017unifying}, because the variational distribution used in~\Cref{eq:optimal-qt} depends on the next state and action as well as on the desired outcome.
\end{remark}

\subsection{Lemmas}

\begin{lemma}
\label{lemma-app:survival}
Let $q(T=t) \defines q(T = t | T \geq t) \prod_{i=1}^{t} q(T \neq i-1 | T \geq i-1)$ be a discrete probability distribution with support $\mathbb{N}_{0}$.
Then for any $t \in \mathbb{N}_{0}$, we have that 
\begin{align}
    q(T \geq t)
    =
    \sum_{i=t}^{\infty} q(T = i | T \geq i) \prod_{j=1}^{i} q(T \neq j-1 | T \geq j-1)
    =
    \prod_{i=1}^{t} q(T \neq i-1 | T \geq i-1).
\end{align}
\end{lemma}
\begin{proof}
We proof the statement by induction on $t$.

\underline{Base case}: For $t=0$, $q(T \geq 0) = 1$ by definition of the empty product.

\underline{Inductive case}:
Note that $q(T \leq t) = \prod_{i=1}^{t} q(T = i-1 | T \geq i-1)$.
Show that
\begin{align}
    q(T \geq t) = \prod_{i=1}^{t} q(T \neq i-1 | T \geq i-1) \Longrightarrow q(T \geq t+1) = \prod_{i=1}^{t+1} q(T \neq i-1 | T \geq i-1).
\end{align}
Consider $q(T \geq t+1) = \sum_{i=t+1}^{\infty} q(T = i | T \geq i) \prod_{j=1}^{i} q(T \neq j-1 | T \geq j-1)$.
To proof the inductive hypothesis, we need to show that the following equality is true:
\begin{align}
    &\sum_{i=t+1}^{\infty} q(T = i | T \geq i) \prod_{j=1}^{i} q(T \neq j-1 | T \geq j-1) =  \prod_{i=1}^{t+1} q(T \neq i-1 | T \geq i-1)
    \\
    \begin{split}
    \Longleftrightarrow
    &\sum_{i=t}^\infty q(T = i | T \geq i) \prod_{j=1}^{i} q(T \neq j-1 | T \geq j-1) - q(T = t | T \geq t) \prod_{j=1}^{t} q(T \neq j-1 | T \geq j-1)
    \\
    &= q(T \neq t | T \geq t) \prod_{i=1}^{t} q(T \neq i-1 | T \geq i-1).
    \label{eq:survival_inductive_case_part1}
    \end{split}
    \end{align}
    By the inductive hypothesis, \begin{align}
        q(T \geq t) = \sum_{i=t}^\infty q(T = i | T \geq i) \prod_{j=1}^{i} q(T \neq j-1 | T \geq j-1) = \prod_{i=1}^{t} q(T \neq i-1 | T \geq i-1),
    \end{align}    
    and so
    \begin{align}
    \textrm{\Cref{eq:survival_inductive_case_part1}} \Longleftrightarrow
    &\prod_{j=1}^{t} q(T \neq j | T \geq j) - q(T \neq t+1 | T \geq t+1) \prod_{j=1}^{t} q(T = j | T \geq j)
    \\
    \qquad &=
    q(T \neq t | T \geq t) \prod_{i=1}^{t} q(T \neq i-1 | T \geq i-1).
    \end{align}
    Factoring out $\prod_{i=1}^{t} q(T \neq i-1 | T \geq i-1)$, we get
    \begin{align}
    \Longleftrightarrow
    \prod_{j=1}^{t} q(T \neq j-1 | T \geq j-1) \underbrace{\left( 1 - q(T = t | T \geq t) \right)}_{=q(T \neq t | T \geq t)} &= q(T \neq t | T \geq t) \prod_{j=1}^{t} q(T = j-1 | T \geq j-1)
    \\
    \Longleftrightarrow
    q(T \neq t | T \geq t) \prod_{j=1}^{t} q(T \neq j-1 | T \geq j-1) &= q(T \neq t | T \geq t) \prod_{j=1}^{t} q(T \neq j-1 | T \geq j-1),
\end{align}
which proves the inductive hypothesis.
\end{proof}

\begin{lemma}
\label{lemma-app:survival_q_delta}
Let $q_T(t)$ and $p_T(t)$ be discrete probability distributions with support $\mathbb{N}_{0}$, let $\Delta_{t}$ be a Bernoulli random variable, with success defined as $T=t+1$ given that $T \geq t$, and let $q_{\Delta_{t}}$ be a discrete probability distribution over $\Delta_{t}$ for $t \in \mathbb{N}\backslash \{ 0 \}$, so that
\begin{align}
\label{eq-app:q-delta-formal-definition_lemma1}
\begin{split}
    \qdeltatzero{t+1}
    &\defines
    q(T \neq t \vbar T \geq t)
    \\
    \qdeltatone{t+1}
    &\defines
    q(T = t \vbar T \geq t).
\end{split}
\end{align}
Then we can write $q(T=t) = \qdeltatone{t+1} \prod_{i=1}^{t} \qdeltatzero{i}$ for any $t \in \mathbb{N}_{0}$ and have that 
\begin{align}
    q(T \geq t)
    =
    \sum_{i=t}^{\infty} \qdeltatone{i+1} \prod_{j=1}^{i} \qdeltatzero{j}
    =
    \prod_{i=1}^{t} \qdeltatzero{i}.
\end{align}
\end{lemma}
\begin{proof}
By~\Cref{lemma-app:survival}, we have that for any $t \in \mathbb{N}_{0}$
\begin{align}
    q(T \geq t)
    =
    \sum_{i=t}^{\infty} q(T = i | T \geq i) \prod_{j=1}^{i} q(T \neq j-1 | T \geq j-1)
    =
    \prod_{i=1}^{t} q(T \neq i-1 | T \geq i-1).
\end{align}
The result follows by replacing $q(T = i | T \geq i)$ by $\qdeltatone{i+1}$, $q(T \neq j-1 | T \geq j-1)$ by $\qdeltatzero{j}$, and $q(T \neq i-1 | T \geq i-1)$ by $\qdeltatzero{i}$.
\end{proof}

\newcommand{\onem}[1]{(1 - #1)}
\begin{lemma}
\label{lemma-app:kl-decomp1}
Let $q_T(t)$ and $p_T(t)$ be discrete probability distributions with support $\mathbb{N}_{0}$.
Then for any $k \in \mathbb{N}_{0}$,
\begin{align}
\begin{split}
\label{eq-app:recursive-kl-decomp-any-k}
    &\E_{t \sim q(T \vbar T \geq k)}
    \left[
        \log \frac{q(T=t \vbar T \geq k)}{p(T=t \vbar T \geq k)}
    \right]
    \\
    &\qquad
    =
    f(q, p, k)
    +
    q(T \neq k \vbar T \geq k)
    \E_{t \sim q(T \vbar T \geq k+1)}
    \left[
        \log \frac{q(T=t \vbar T \geq k+1)}{p(T=t \vbar T \geq k+1)}
    \right].
\end{split}
\end{align}
\end{lemma}
\begin{proof}
Consider $\E_{t \sim q(T \vbar T \geq k)}
    \left[
        \log \frac{q(T=t \vbar T \geq k)}{p(T=t \vbar T \geq k)}
    \right]$ and note that by the law of total expectation we can rewrite it as
\begin{align}
    \nonumber
    &\E_{t \sim q(T \vbar T \geq k)}
    \left[
        \log \frac{q(T=t \vbar T \geq k)}{p(T=t \vbar T \geq k)}
    \right]
    \\
    \begin{split}
    &~
    =
    q(T=k \vbar T \geq k)
    \E_{t \sim q(T \vbar T = k)}
    \left[
        \log \frac{q(T=t \vbar T \geq k)}{p(T=t \vbar T \geq k)}
    \right]
    \\
    &\qquad
    + 
    q(T \neq k \vbar T \geq k)
    \E_{t \sim q(T \vbar T \geq k+1)}
    \left[
        \log \frac{q(T=t \vbar T \geq k)}{p(T=t \vbar T \geq k)}
    \right]
    \end{split}
    \\
    &~
    =
    \label{eq-app:kl-decomp-one-e}
    q(T=k \vbar T \geq k)
        \log \frac{q(T=k \vbar T \geq k)}{p(T=k \vbar T \geq k)}
    + 
    q(T \neq k \vbar T \geq k)
    \E_{t \sim q(T | T \geq k+1)}
    \left[
        \log \frac{q(T=t \vbar T \geq k)}{p(T=t \vbar T \geq k)}
    \right].
\end{align}
For all values of $T \geq k+1$, we have that
\begin{align}
    q(T=t \vbar T \geq k)
        &= q(T=t \vbar T \geq k+1) q(T \neq k \vbar T \geq k)
    \\
    p(T=t \vbar T \geq k)
        &= p(T=t \vbar T \geq k+1) p(T \neq k \vbar T \geq k)
\end{align}
and so we can rewrite the expectation in~\Cref{eq-app:kl-decomp-one-e} as
\begin{align}
    \E_{t \sim q(T \vbar T \geq k+1)}
    \left[
        \log \frac{q(T=t \vbar T \geq k)}{p(T=t \vbar T \geq k)}
    \right]
    &=
    \E_{t \sim q(T \vbar T \geq k+1)}
    \left[
        \log \frac{q(T=t \vbar T \geq k)}{p(T=t \vbar T \geq k)}
        + \log \frac{q(T \neq k \vbar T \geq k)}{p(T \neq k \vbar T \geq k)}
    \right]
    \\
    &=\label{eq-app:kl-decomp-tmp1}
    \E_{t \sim q(T \vbar T \geq k+1)}
    \left[
        \log \frac{q(T=t \vbar T \geq k)}{p(T=t \vbar T \geq k)}
    \right]
    + \log \frac{q(T \neq k \vbar T \geq k)}{p(T \neq k \vbar T \geq k)}
\end{align}
Combining~\Cref{eq-app:kl-decomp-tmp1} with~\Cref{eq-app:kl-decomp-one-e}, we have
\begin{align}
\begin{split}
    &\E_{t \sim q(T \vbar T \geq k)}
    \left[
        \log \frac{q(T=t \vbar T \geq k)}{p(T=t \vbar T \geq k)}
    \right]
    \\
    &\qquad
    =
    \underbrace{
    q(T=k \vbar T \geq k)
        \log \frac{q(T=k \vbar T \geq k)}{p(T=k \vbar T \geq k)}
    +
    q(T \neq k \vbar T \geq k)
        \log \frac{q(T \neq k \vbar T \geq k)}{p(T \neq k \vbar T \geq k)}
    }_{
        \defines f(q, p, k)
    }
    \\
    &\qquad\qquad
    +
    q(T \neq k \vbar T \geq k)
    \E_{t \sim q(T \vbar T \geq k+1)}
    \left[
        \log \frac{q(T=t \vbar T \geq k+1)}{p(T=t \vbar T \geq k+1)}
    \right],
\end{split}
\end{align}
which concludes the proof.
\end{proof}

\begin{lemma}
\label{lemma-app:kl-decomp2}
Let $q_T(t)$ and $p_T(t)$ be discrete probability distributions with support $\mathbb{N}_{0}$.
Then the \kld from $q_{T}$ to $p_{T}$ can be written as
\begin{align}
\label{eq-app:kl-decomp-general}
    \mathbb{D}_{\emph{\textrm{KL}}}(\qtdot \,||\, \ptdot) = \sum_{t=0}^\infty
        q(T \geq t)
        f(q_{T}, p_{T}, t)
\end{align}
where $f(q_{T}, p_{T}, t)$ is shorthand for
\begin{align}
\begin{split}
    f(q_{T}, p_{T}, t)
    &=
    q(T=t \vbar T \geq t)
    \log \frac{q(T=t \vbar T \geq t)}{p(T=t \vbar T \geq t)}
    +
    q(T \neq t \vbar T \geq t)
    \log \frac{q(T \neq t \vbar T \geq t)}{p(T \neq t \vbar T \geq t)}.
\end{split}
\end{align}
\end{lemma}
\begin{proof}
Note that $q(T=k)$ denotes the probability that the distribution $q$ assigns to the event $T=k$ and $q(T \geq m)$ denotes the tail probability, that is, $q(T \geq m) = \sum_{t=m}^\infty q(T = t)$.
We will write $q(T | T \geq m)$ to denote the conditional distribution of $q$ given $T \geq m$, that is, $q(T = k | T \geq m) = \mathbbm{1}[k \geq m] q(T = k) / q(T \geq m)$.
We will use analogous notation for $p$.

By the definition of the \kld and using the fact that, since the support is lowerbounded by $T=0$, \mbox{$q(T=0) = q(T=0 \vbar T \geq 0)$}, we have
\begin{align}
    \DKL{\qtdot}{\ptdot}
    &=
    \E_{t \sim q(T)}
    \left[
        \log \frac{q(T=t)}{p(T=t)}
    \right]
    =
    \E_{t \sim q(T \vbar T \geq 0)}
    \left[
        \log \frac{q(T=t \vbar T \geq 0)}{p(T=t \vbar T \geq 0)}
    \right].
\end{align}
Using~\Cref{lemma-app:kl-decomp1} with $k=0, 1, 2, 3, \dots$, we can expand the above expression to get
\begin{align}
    &\DKL{\qtdot}{\ptdot}
    \\
    &=
    f(q_{T}, p_{T}, 0)
    +
    q(T \neq 0 \vbar T \geq 0)
    \E_{t \sim q(T \vbar T \geq 1)}
    \left[
        \log \frac{q(T=t \vbar T \geq 1)}{p(T=t \vbar T \geq 1)}
    \right]
    \\
    \begin{split}
    &=
    f(q, p, 0) +
        q(T \neq 0 \vbar T \geq 1)
        f(q_{T}, p_{T}, 1)
    \\
    &\qquad+
        q(T \neq 0 \vbar T \geq 0)
        q(T \neq 1 \vbar T \geq 1)
        \E_{t \sim q(T \vbar T \geq 2)}
        \left[
            \log \frac{q(T=t \vbar T \geq 2)}{p(T=t \vbar T \geq 2)}
        \right]
    \end{split}
    \\
    \begin{split}
    &=
    \underbrace{
        1
    }_{= q(T \geq 0)}
    \cdot
    f(q, p, 0)
    \\
    &\qquad+
    \underbrace{
        q(T \neq 0 \vbar T \geq 0)
    }_{= q(T \geq 1)}
    f(q, p, 1)
    \\
    &\qquad+
    \underbrace{
        q(T \neq 0 \vbar T \geq 0)
        q(T \neq 1 \vbar T \geq 1)
    }_{= q(T \geq 2)}
    f(q_{T}, p_{T}, 2)
    \\
    &\qquad+
    \underbrace{
        q(T \neq 0 \vbar T \geq 0)
        q(T \neq 1 \vbar T \geq 1)
        q(T \neq 2 \vbar T \geq 2)
    }_{= q(T \geq 3)}
    \E_{t \sim q(T \vbar T \geq 3)}
    \left[
        \log \frac{q(T=t \vbar T \geq 3)}{p(T=t \vbar T \geq 3)}
    \right]
    \end{split}
    \\
    &=
    \sum_{t=0}^\infty
    q(T \geq t)
    f(q_{T}, p_{T}, t),
\end{align}
where $f(q_{T}, p_{T}, t)$ is shorthand for
\begin{align}
\begin{split}
    f(q_{T}, p_{T}, t)
    &=
    q(T=t \vbar T \geq t)
    \log \frac{q(T=t \vbar T \geq t)}{p(T=t \vbar T \geq t)}
    +
    q(T \neq t \vbar T \geq t)
    \log \frac{q(T \neq t \vbar T \geq t)}{p(T \neq t \vbar T \geq t)}.
\end{split}
\end{align}
and we used the fact that, by~\Cref{lemma-app:survival},
\begin{align}
\label{eq-app:prod-of-delta}
    q(T \geq t)
    =
    \prod_{k=1}^{t}
    q(T \neq k-1 \vbar T \geq k-1).
\end{align}
This completes the proof.
\end{proof}

\begin{lemma}
\label{lemma-app:kl-decomp-delta}
Let $q_T(t)$ and $p_T(t)$ be discrete probability distributions with support $\mathbb{N}_{0}$, let $\Delta_{t}$ be a Bernoulli random variable, with success defined as $T=t$ given that $T \geq t$, and let $q_{\Delta_{t}}$ and $p_{\Delta_{t}}$ be discrete probability distributions over $\Delta_{t}$ for $t \in \mathbb{N}_{0} \backslash \{ 0 \}$, so that
\begin{align}
    \qdeltatzero{t+1} \defines q(T \neq t \vbar T \geq t)
    \qquad
    &\qdeltatone{t+1} \defines q(T = t \vbar T \geq t)
    \label{eq-app:q-delta-formal-definition}
    \\
    \pdeltatzero{t+1} \defines p(T \neq t \vbar T \geq t)
    \qquad
    &\pdeltatone{t+1} \defines p(T = t \vbar T \geq t).
\end{align}
Then the \kld from $\qtdot$ to $\ptdot$ can be written as
\begin{align}
\label{eq-app:kl-decomp-delta}
    \mathbb{D}_{\emph{\textrm{KL}}}(\qtdot \,||\, \ptdot)
    =
    \sum_{t=0}^\infty 
    \Big(
    \prod_{k=1}^{t} \qdeltatzero{t}
    \Big)
    \mathbb{D}_{\emph{\textrm{KL}}}(\qdeltatdot{t+1} \,||\, \pdeltatdot{t+1})
\end{align}
\end{lemma}
\begin{proof}
The result follows from~\Cref{lemma-app:kl-decomp2},
\Cref{eq-app:prod-of-delta},
\Cref{eq-app:q-delta-formal-definition}, and the definition of $f$.

In detail, from~\Cref{lemma-app:survival},
and~\Cref{eq-app:q-delta-formal-definition}
we have that
\begin{align}
    \label{eq-app:kl-decomp-delta-tmp1}
        q(T \geq t)
    =
    \prod_{k=1}^{t}
    q(T \neq k-1 \vbar T \geq k-1)
    =
    \prod_{k=1}^{t}
    \qdeltatzero{k}.
\end{align}
From the definition of $f(q_{T}, p_{T}, t)$, we have
\begin{align}
    f(q_{T}, p_{T}, t)
    &=
    q(T=t \vbar T \geq t)
    \log \frac{q(T=t \vbar T \geq t)}{p(T=t \vbar T \geq t)}
    +
    q(T \neq t \vbar T \geq t)
    \log \frac{q(T \neq t \vbar T \geq t)}{p(T \neq t \vbar T \geq t)}
    \\
    &=
    \qdeltatzero{t+1}
    \log \frac{\qdeltatzero{t+1}}{\pdeltatzero{t+1}}
    +
    q(\Delta_{t+1}=1)
    \log \frac{\qdeltatone{t+1}}{\pdeltatone{t+1}}
    \\
    \label{eq-app:kl-decomp-delta-tmp2}
    &=
    \DKL{\qdeltatdot{t+1}}{\pdeltatdot{t+1}}.
\end{align}
Combining
\Cref{eq-app:kl-decomp-delta-tmp1}, 
\Cref{eq-app:kl-decomp-delta-tmp2}, 
and
\Cref{eq-app:kl-decomp-general}
completes the proof.
\end{proof}

\section{Additional Experiments}\label{appsec:additional-exps}

\subsection{Further Ablation Study Results}

We show the full ablation learning curves in~\Cref{app-fig:ablation-learning-curves}.
We see that \odac consistently performs the best, and that \odac with a fixed model also performs well.
However, on a few tasks, and in particular the Fetch Push and Sawyer Faucet tasks, we see that using a fixed $q_T$ hurts the performance, suggesting that our derived formula in~\Cref{eq:optimal-qt} results in better empirical performance.

\begin{figure*}[h!]
     \vspace*{-10pt}
     \centering
     \includegraphics[height=3.5cm]{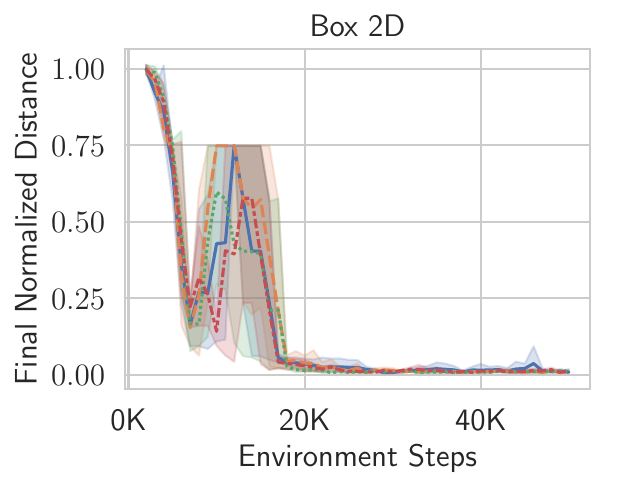}
     \includegraphics[height=3.5cm]{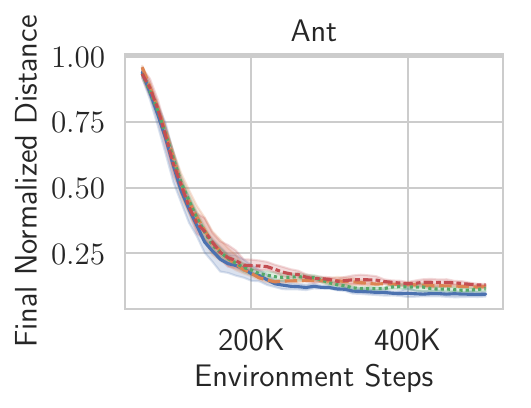}
     \includegraphics[height=3.5cm]{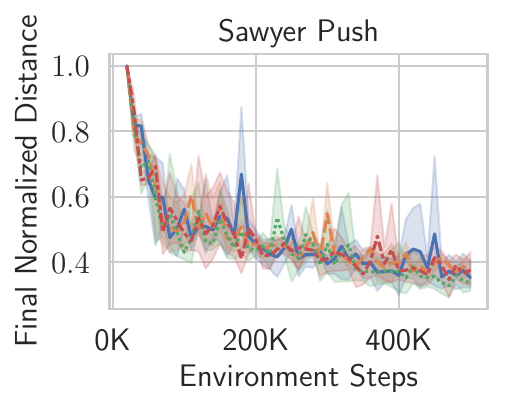}
    \\
     \includegraphics[height=3.5cm]{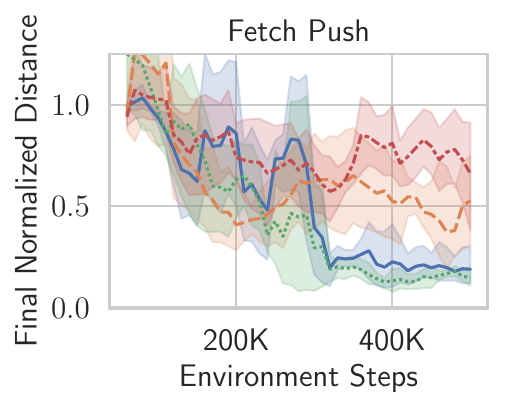}
     \includegraphics[height=3.5cm]{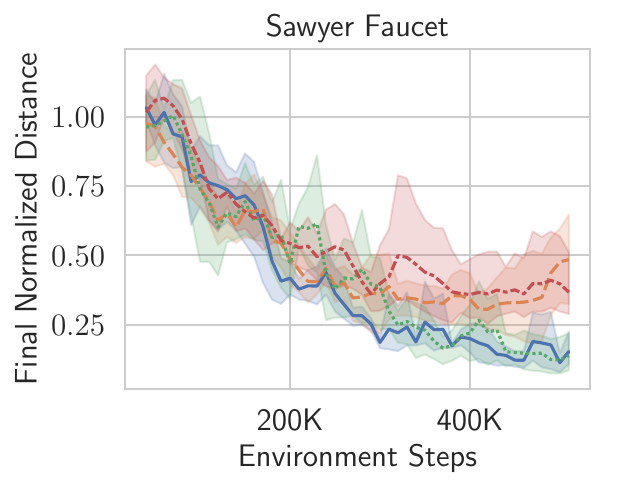}
     \includegraphics[height=3.5cm]{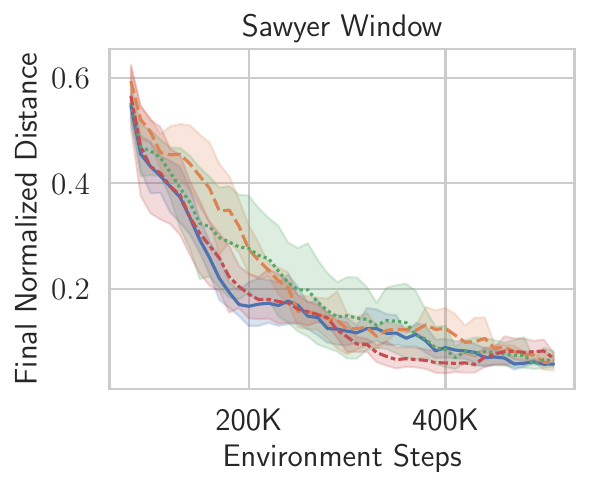}
     \\
    \vspace{1pt}
     \includegraphics[width=0.8\textwidth,trim={0 0 0.4cm 0.45cm},clip]{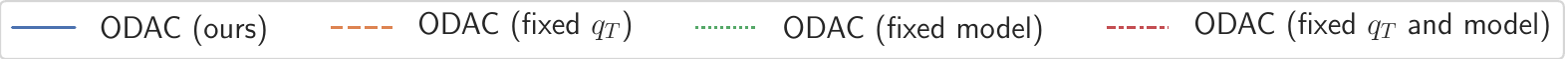}
     \caption{
     Ablation results across all six environments.
     We see that using our derived $q_T$ optimality equation is important for best performance across all six tasks and that \odac is not sensitive to the quality of the dynamics model.
     }
     \label{app-fig:ablation-learning-curves}
\end{figure*}

\begin{figure*}[h!]
  \centering
  \vspace*{-5pt}
    \begin{subfigure}[b]{\linewidth}
        \centering
        \captionsetup[subfigure]{labelformat=empty}
        \begin{subfigure}[b]{0.3\linewidth}
            \includegraphics[width=0.9\linewidth, trim=5 7 5 4,clip]{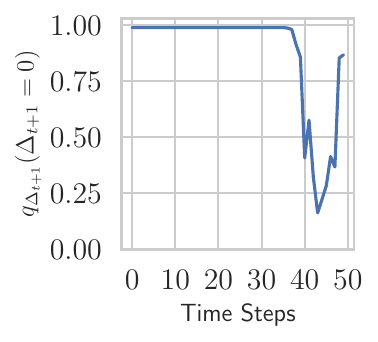}
        \end{subfigure}
        \hfill
        \begin{subfigure}[b]{0.2\linewidth}
            \includegraphics[width=\linewidth,trim=0 0 20 20,clip]{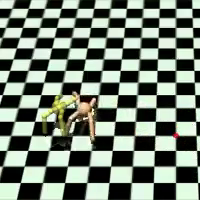}
            \caption*{$t=30$}
        \end{subfigure}
        \hfill
        \begin{subfigure}[b]{0.2\linewidth}
            \includegraphics[width=\linewidth,trim=0 0 20 20,clip]{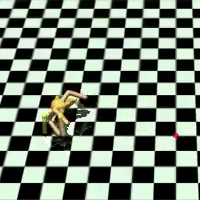}
            \caption*{$t=40$}
        \end{subfigure}
        \hfill
        \begin{subfigure}[b]{0.2\linewidth}
            \includegraphics[width=\linewidth,trim=0 0 20 20,clip]{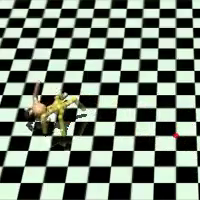}
            \caption*{$t=50$}
        \end{subfigure}
    \label{fig:ant-rollout}
    \end{subfigure}
    \caption{
        The inferred $\qdeltatzero{t+1}$ versus time during an example trajectory in the Ant environment.
        As the ant robot falls over,
        $\qdeltatzero{t+1}$ drops in value.
        We see that the optimal posterior $q_{\Delta_{t+1}}^{\star}(\Delta_{t+1} = 0)$ given in~\Cref{prop:gc_optimal_variational_distribution_T} automatically assigns a high likelihood of terminating when this irrecoverable state is first reached, effectively acting as a dynamic discount factor.
    }
    \label{fig:example-rollouts}
    \vspace*{-10pt}
\end{figure*}

\subsection{Comparisons under Oracle Goal Sampling}

\begin{wrapfigure}{R}{0.6\textwidth}
\vspace*{-16pt}
    \centering
    \hspace*{-4pt}\includegraphics[width=1.03\linewidth]{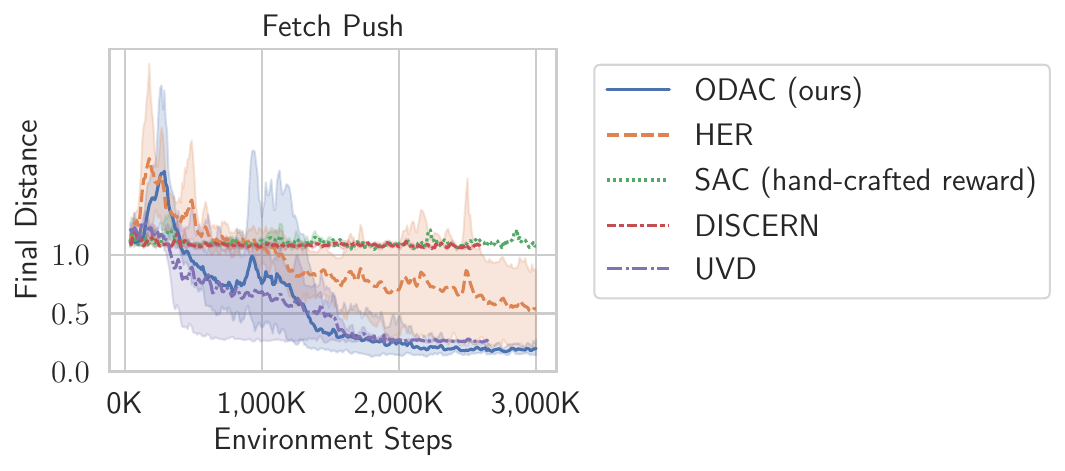}
    \vspace*{-15pt}
    \caption{
        Comparison of different methods for desired outcomes $\bg$ sampled uniformly from the set of admissible states.
        }
    \label{app-fig:oracle-goals-fetch}
\vspace*{-15pt}
\end{wrapfigure}

For exploration,~\citet{andrychowicz2017her} explore the benefits of HER either using a single, fixed goal during exploration (see Section 4.3 of~\citet{andrychowicz2017her}) or using oracle goal sampling, that is, during exploration, a new goal is sampled each episode from a uniform distribution over the set of all reachable goals in the environment.
As such, oracle goal sampling requires knowledge of the environment to sample several reachable goals.
For example, in the 2D box experiment (\Cref{fig:2d-env-picture}), points inside the grey block in the center are not reachable goal states, and this additional information must be available when performing oracle goal sampling.

To demonstrate the impact of sampling the desired outcome $\bg$ during exploration, we evaluate \odac and related methods on the Fetch task when using oracle goal sampling.
As shown in~\Cref{app-fig:oracle-goals-fetch}, the performances of UVD and \odac are similar and both outperform other methods.
These results suggest that UVD depends more heavily on sampling outcomes from the set of desired outcomes than \odac.
The significant decrease in performance when the desired outcome $\bg$ is fixed may be due to the fact that uniformly sampling $\bg$ implicitly provides a curriculum for learning.
For example, in the Box 2D environment, goal states sampled above the box can train the agent to move around the obstacle, making it easier to learn how to reach the other side of the box.
Without this guidance, prior methods often ``get stuck'' on the other side of the box.
In contrast, \odac consistently performs well in this more challenging setting, suggesting that the log-likelihood signal provides good guidance to the policy.

As shown in~\Cref{fig:sawyer-and-ant}, \odac performs well on both this setting and the harder setting where the desired outcome $\bg$ was fixed during exploration, suggesting that \odac does not rely as heavily on the uniform sampling of $\bg$ to learn a good policy than do other methods.

\subsection{Comparison to Model-Based Planning}

\odac learns a dynamics model but does not use it for planning and instead relies on the derived Bellman updates to obtain a policy.
However, a natural question is whether or not the method would benefit from using this model to perform model-based planning, as in~\citet{janner2019mbpo}.
We assess this by comparing \odac with model-based baseline that uses a 1-step look-ahead.
In particular, we follow the training procedure in~\citet{janner2019mbpo} with $k=1$.
To ensure a fair comparison, we use the exact same dynamics model architecture as in \odac and match the update-to-environment step ratio to be 4-to-1 for both methods.

\Cref{table:dyna} shows the final distance to the goal (best results in bold).
Using the same dynamics model, \odac, which does not use the dynamics model to perform planning and only uses it to compute rewards, outperforms the model-based planning method.
While a better model might lead to better performance for the model-based baseline, these results suggest that \odac is not sensitive to model quality to the same degree as model-based planning methods.

\setlength{\tabcolsep}{13.0pt}
\begin{table}[h!]
\centering
\vspace*{-5pt}
\captionof{table}{
Normalized final distances (lower is better) across four random seeds, multiplied by a factor of $100$.
}
\begin{tabular}{l|c|c}
\toprule
    Environment &       \odac (Mean + Standard Error)   &       Dyna (Mean + Standard Error) \\
\hline
 Box 2D &    0.74 (0.091) &  0.87 (0.058) \\
    Ant &         33 (27) &    102 (0.83) \\
 Sawyer Faucet &        14 (6.3) &       100 (5) \\
  Fetch Push &        12 (3.7) &      96 (3.8) \\
   Sawyer Push &        58 (8.7) &     96 (0.39) \\
 Sawyer Window &       4.4 (1.5) &      116 (14)
\Bstrut\\
\hline
\midrule
\end{tabular}
\label{table:dyna}
\end{table}

\subsection{Reward Visualization}

We visualize the reward for the Box 2D environment in~\Cref{app-fig:reward-vis}.
We see that over the course of training, the reward function initially flattens out near $\bg$, making learning easier by encouraging the policy to focus on moving just out of the top left corner of the environment.
Later in training (around 16,000 steps), the policy learns to move out of the top left corner, and we see that the reward changes to have a stronger reward gradient near $\bg$.
We also note that the reward are much more negative for being far $\bg$ at the end of training: the top left region changes from having a penalty of $-1.6$ to $-107$.
Overall, these visualizations show that the reward function automatically changes during training and provides a strong reward signal for different parts of the state space depending on the behavior of the policy.

\begin{figure}[t!]
    \centering
    \vspace*{-20pt}
    \begin{subfigure}{\textwidth}
      \centering
      \includegraphics[width=0.33\linewidth]{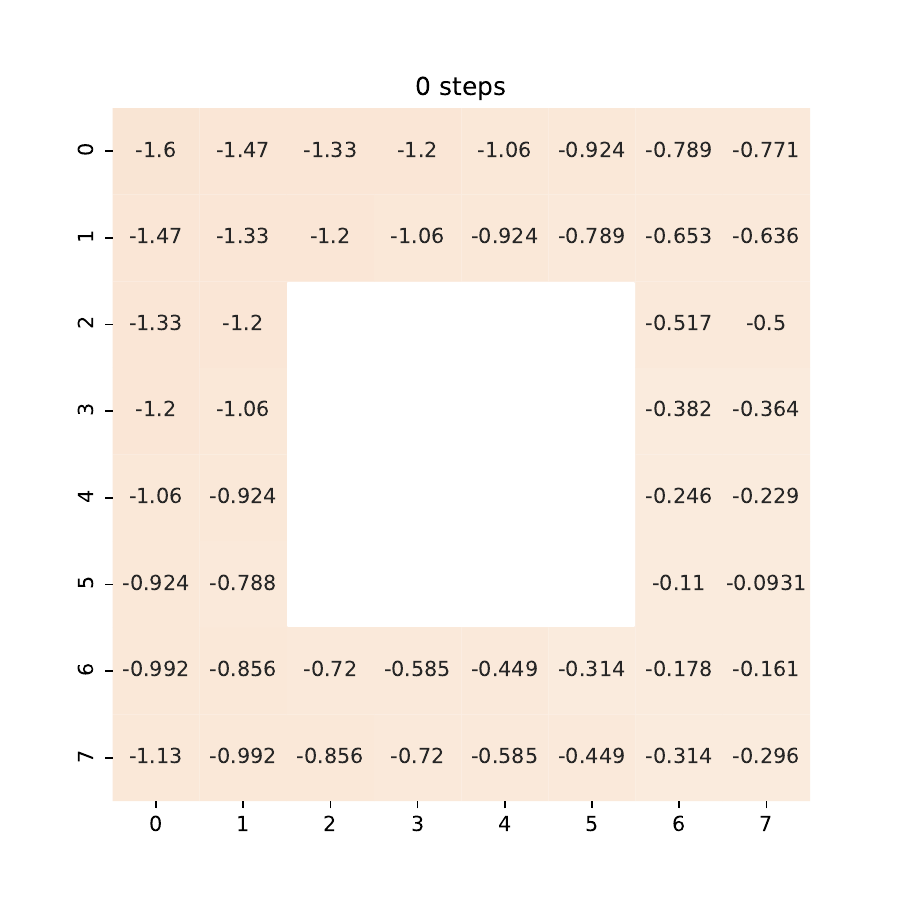}
      \includegraphics[width=0.33\linewidth]{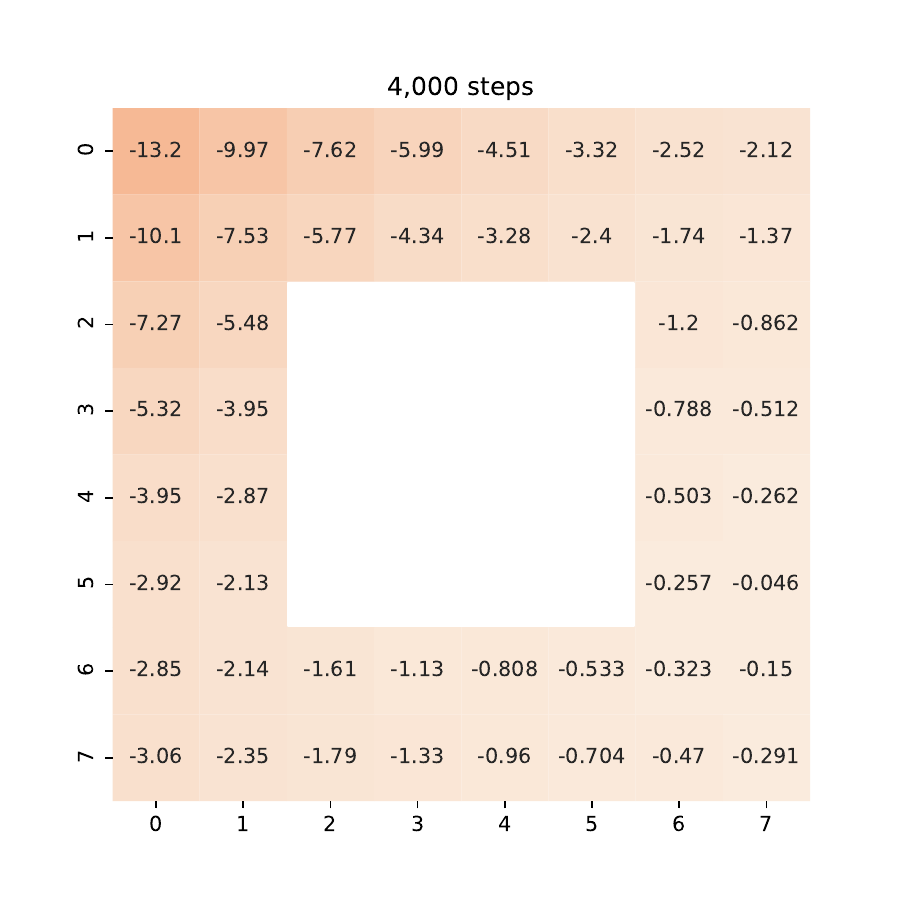}
      \includegraphics[width=0.33\linewidth]{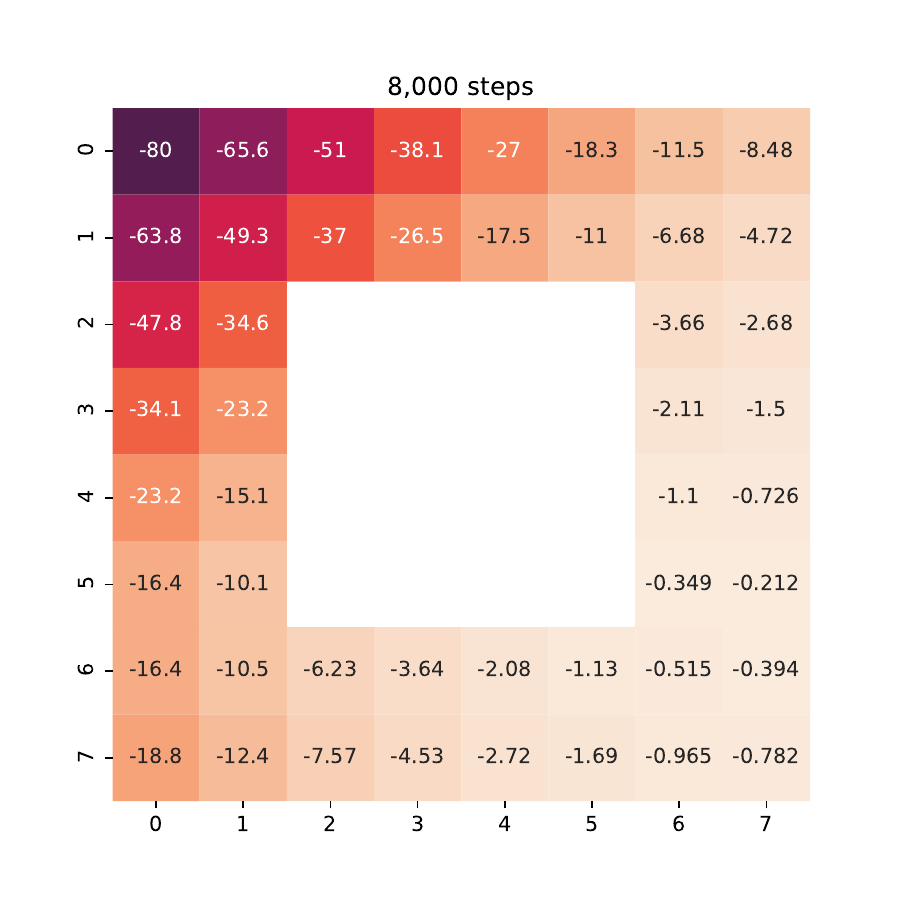}
      \includegraphics[width=0.33\linewidth]{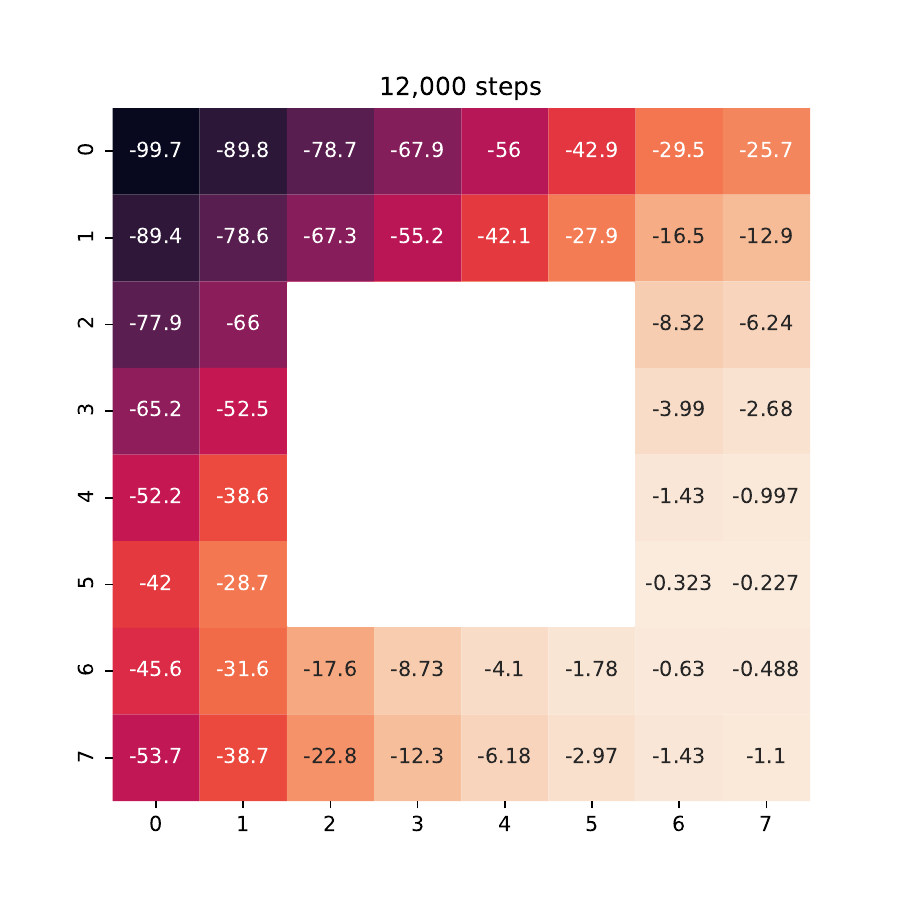}
      \includegraphics[width=0.33\linewidth]{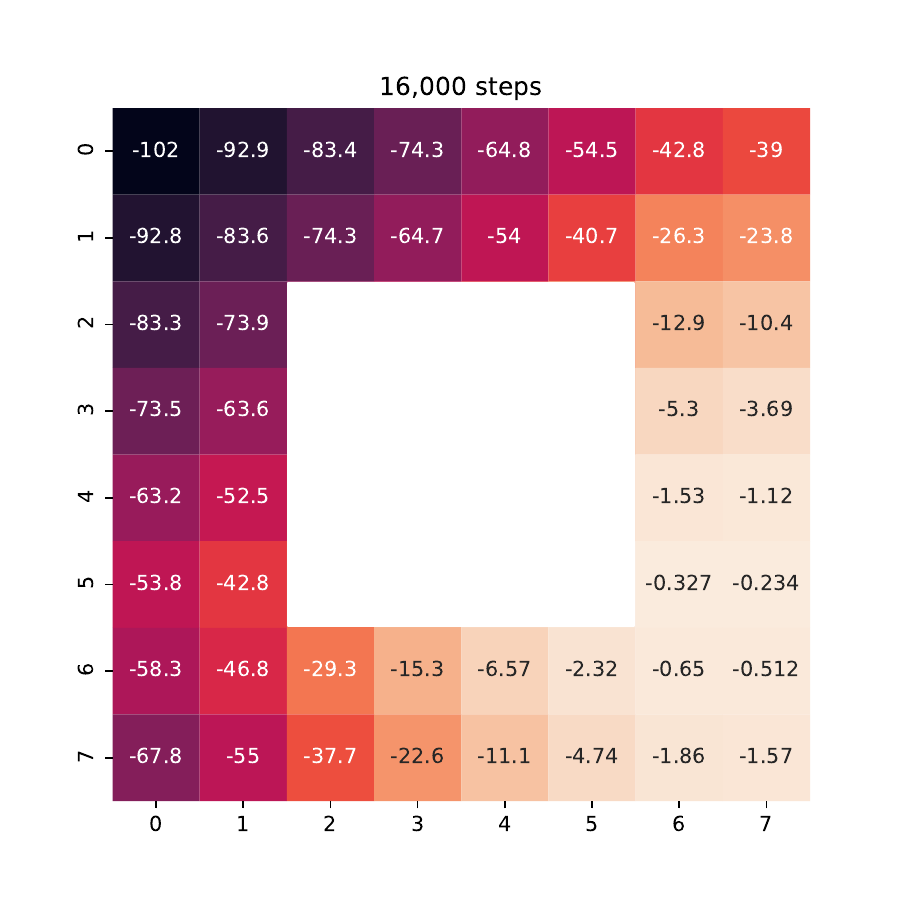}
      \includegraphics[width=0.33\linewidth]{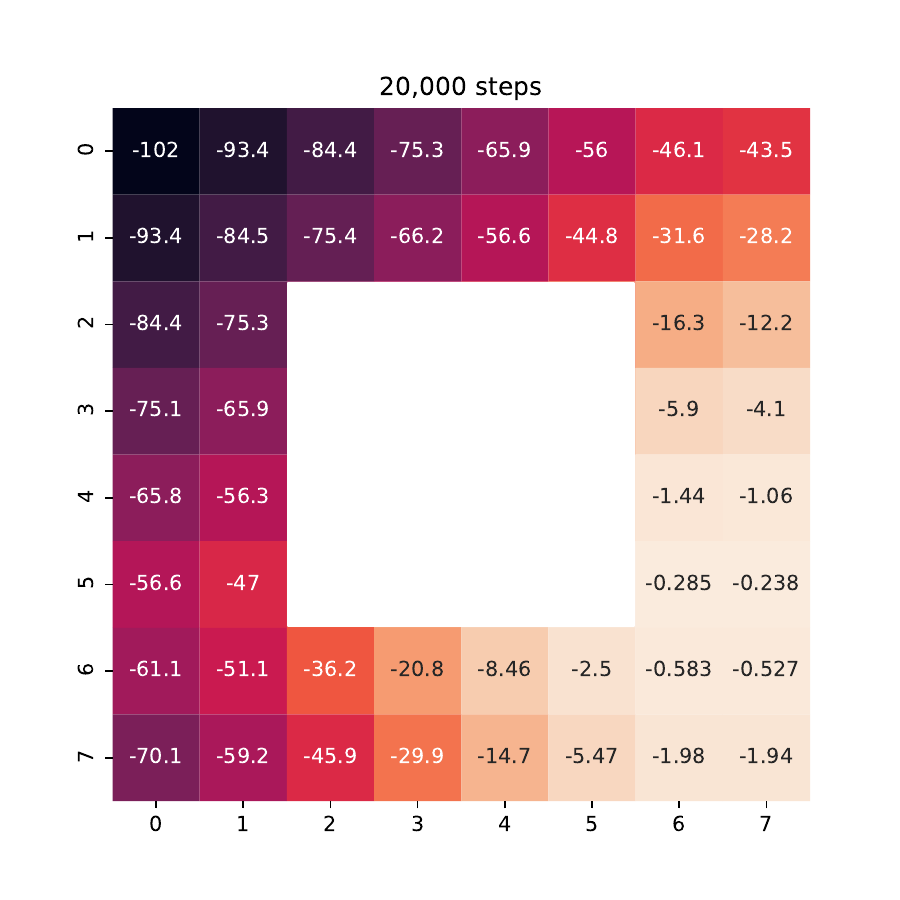}
      \includegraphics[width=0.5\linewidth]{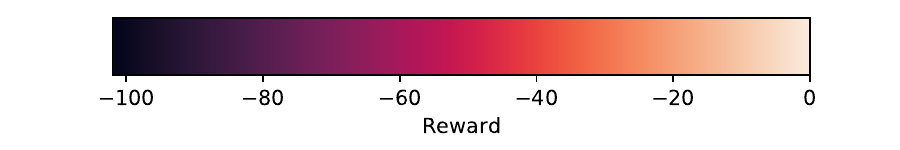}
      \caption{
        Reward Visualization
      }
    \end{subfigure}
    \\
    \begin{subfigure}{0.48\textwidth}
      \centering
      \includegraphics[height=5cm]{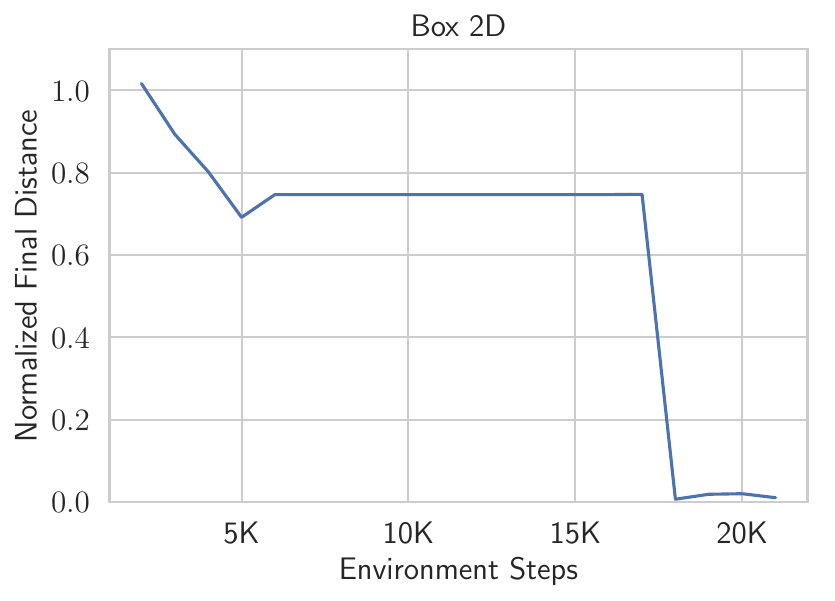}
      \caption{
        Learning Curve associated with Reward Visualized
      }
      \label{app-fig:box2d-learning-curve}
    \end{subfigure}
    \begin{subfigure}{0.48\textwidth}
      \centering
      \vspace*{13.5pt}
      \includegraphics[trim={0 0 0 0 0}, clip, height=4.5cm]{figures/box2d_two_ways.pdf}
      \caption{
        Environment Visualization
      }
      \label{app-fig:box2d-vis}
    \end{subfigure}
    \caption{
        We visualize the rewards over the course of training on a single random seed for the Box 2D environment.
        To visualize the reward, we discretize the continuous state space and evaluate
    $\rewardt{t}$
    for
    $\ba = \vec{0}$
    at different states.
    As shown in~\Cref{app-fig:box2d-vis}, the desired outcome $\bg$ is near the bottom right and the states in the center are invalid.
    After 4-8 thousand environment steps, the reward is more flat near $\bg$, and only provides a reward gradient far from $\bg$.
    After 20 thousand environment steps, the reward gradient is much larger again near the end, and the penalty for being in the top left corner has changed from $-1.6$ to $-107$.
    }
    \label{app-fig:reward-vis}
\end{figure}

\clearpage

\section{Experimental Details}\label{appsec:exp-details}

\subsection{Environment}\label{appsec:env_details}

\paragraph{Ant.}
This Ant domain is based on the ``Ant-V3'' OpenAI Gym~\citep{brockman2016openai} environment, with three modifications:
the gear ratio is reduced from $150$ to $120$,
the contact force sensors are removed from the state,
and there is no termination condition and the episode only terminates after a fixed amount of time.
In this environment, the state space is $23$ dimensional,
consistent of the XYZ coordinate of the center of the torso,
the orientation of the ant (in quaternion),
and the angle and angular velocity of all 8 joints.
The action space is 8-dimensional and corresponds to the torque to apply to each joint.
The desired outcome consists of the desired XYZ, orientation, and joint angles at a position that is 5 meters down and to the right of the initial position.
This desired pose is shown in~\Cref{fig:sawyer-and-ant}.

\paragraph{Sawyer Push.}
In this environment, the state and goal space is $4$ dimensional and the action space is $2$ dimension.
The state and goal consists of the XY end effector (EE) and the XY position of the puck.
The object is on a 40cm x 50cm table and starts 20 cm in front of the hand.
The goal puck position is fixed to 15 cm forward and 30 cm to the right of the initial hand position, while the goal hand position is 5cm behind and 20 cm to the right of the initial hand position.
The action is the change in position in each XY direction, with a maximum change of 3 cm per direction at each time step.
The episode horizon is 100.

\paragraph{Sawyer Window and Faucet.}
In this environment, the state and goal space is $6$ dimensional and the action space is $2$ dimension.
The state and goal consists of the XYZ end effector (EE) and the XYZ position of the window or faucet end endpoint.
The hand is initialized away from the window and faucet.
The EE goal XYZ position is set to the initial window or faucet position.
The action is the change in position in each XYZ direction.
For the window task, the goal positions is to close the window, and for the faucet task, the goal position is to rotate the faucet $90$ degrees counter-clockwise from above.

\paragraph{Box 2D.}
In this environment, the state is a $4x4$ with a $2x2$ box in the middle.
The policy is initialized to to $(-3.5, -2)$ and the desired outcome is $(3.5, 2)$.
The action is the XY velocity of the agent, with wall collisions taken into account and maximum velocity of $0.2$ in each direction.
To make the environment stochastic, we add Gaussian noise to actions with mean zero and standard deviation that's 10\% of the maximum action magnitude.

\paragraph{Tabular Box 2D (\Cref{fig:heatmaps}).}

We implemented a tabular version of \textsc{odac} and applied it to the 2D environment shown in~\Cref{fig:2d-env-picture}.
We discretize the environment into an $8 \times 8$ grid of states.
The action correspond to moving up, down, left, or right.
If probability $1-\epsilon$, this action is taken.
If the agent runs into a wall or boundary, the agent stays in its current state.
With probability $\epsilon = 0.1$, the commanded action is ignored and a neighboring state grid (including the current state) is uniformly sampled as the next state.
The policy and $Q$-function are represented with look-up tables and randomly initialized.
The entropy reward is weighted by $0.01$ and the time prior $p_{T}$ is geometric with parameter $0.5$.
The dynamics model, $p_d^{(0)}$ is initialized to give a uniform probability to each states for every state and action.
Each iteration, we simulate data collection by updating the dynamics model with the running average update
$
    p_d^{(t+1)} = 0.99 p_d^{(t)} + 0.01 p_d,
$
where $p_d$ is the true dynamics and update the policy and $Q$-function according to 
\Cref{eq:odpi} and
\Cref{eq:variational-q-t-bellman-update}, respectively.
\Cref{fig:heatmaps} shows that, in contrast to the binary-reward setting, the learned reward provides shaping for the policy, which solves the task within 100 iterations.

\subsection{Algorithm}\label{appsec:algo_details}
Pseudocode for the complete algorithm is shown in~\Cref{algo:odac}.

\begin{algorithm}
\caption{Outcome-Driven Actor Critic}
\label{algo:odac}
\begin{algorithmic}
\REQUIRE Policy $\pi_\ppi$, $Q$-function $Q_\pq$, dynamics
model $p_\pdyn$, replay buffer $\calR$, and map from state to achieved goal $f$.
\FOR{$n=0, \dots N-1$ episodes}
    \STATE Sample initial state $\bs_0$ from environment.
    \STATE Sample goal $\bg$ from environment.
    \FOR{$t=0, \dots, H-1$ steps}
        \STATE Get action $\ba_t \sim \pi_\ppi(\bs_t, \bg)$.
        \STATE Get next state $\bs_{t+1} \sim p(\cdot \vbar \bs_t, \ba_t)$.
        \STATE Store $(\bs_t, \ba_t, \bs_{t+1}, \bg)$ into replay buffer $\calR$.
        \STATE Sample transition $(\bs, \ba, \bs', \bg) \sim \calR$.
        \STATE Compute reward $r = \log p_\pdyn(\bg \vbar \bs, \ba) - \DKL{q_{\Delta}(\cdot | \bs_{t}, \ba_{t})}{p(\Delta)}$.
        \STATE Compute $q(\Delta_{t} = 0 | \bs, \ba)$ using~\Cref{eq:optimal-qt}.
        \STATE Update $Q_\pq$ using~\Cref{eq:q-loss} and data $(\bs, \ba, \bs', \bg, r)$.
        \STATE Update $\pi_\ppi$ using~\Cref{eq:pi-loss} and data $(\bs, \ba, \bg)$.
        \STATE Update $p_\pdyn$ using~\Cref{eq:dyn-loss} and data $(\bs, \ba, \bg)$.
    \ENDFOR
    \FOR{$t=0,...,H -1$ steps}
        \FOR{$i=0,...,k-1$ steps}
            \STATE Sample future state $\bs_{h_i}$, where $t < h_i \leq H-1$.
            \STATE Store $(\bs_t, \ba_t, \bs_{t+1}, f(\bs_{h_i}))$ into
            $\mathcal R$.
        \ENDFOR
    \ENDFOR
\ENDFOR
\end{algorithmic}
\end{algorithm}

\subsection{Implementation Details}
\label{appsec:implementation_details}

\paragraph{Dynamics model.}
For the Ant and Sawyer experiments, we train a neural network to output the
mean and standard deviation of a Laplace distribution.
This distribution is then used to model the distribution over the \textit{difference} between the current state and the next state, which we found to be more reliable than predicting the next state.
So, the overall distribution is given by a Laplace distribution with learned mean $\mu$ and fixed standard deviation $\sigma$ computed via
\begin{align*}
    p_\pdyn = \text{Laplace}(
        \mu = g_\pdyn(\bs, \ba) + f(\bs),
        \sigma = 0.00001
    )
\end{align*}
where $g$ is the output of a network and $f$ is a function that maps a state into a goal.

For the 2D Navigation experiment, we use a Gaussian distribution.
The dynamics neural network has hidden units of size $[64, 64]$ with a ReLU hidden activations.
For the Ant and Sawyer experiments, there is no output activation.
For the linear-Gaussiand and 2D Navigation experiments, we have a tanh output, so that the mean and standard
To bound the standard deviation outputted by the network, the standard-deviation tanh is multiplied by two
with the standard deviation be between limited to between

\paragraph{Reward normalization.}
Because the different experiments have rewards of very different scale, we
normalize the rewards by dividing by a running average of the maximum reward
magnitude.
Specifically, for every reward $r$ in the $i$th batch of data, we replace the reward with
\begin{align*}
    \hat{r} = r / C_i
\end{align*}
where we update the normalizing coefficient $C_i$ using each batch of reward $\{r_b\}_{b=1}^B$:
\begin{align*}
    C_{i+1} \leftarrow (1-\lambda) \times C_i + \lambda \max_{b \in [1, \dots, B]} |r_b|
\end{align*}
and $C_i$ is initialized to $1$.
In our experiments, we use $\lambda = 0.001$.

\paragraph{Target networks.}
To train our Q-function, we use the technique from~\citet{fujimoto2018td3} in which we train two separate Q-networks with target networks and take the minimum over two to compute the bootstrap value.
The target networks are updated using a slow, moving average of the parameters after every batch of data:
\begin{align*}
    \bar{\pq}_{i+1} = (1 - \tau) \bar{\pq}_i + \times \pq_i.
\end{align*}
In our experiments, we used $\tau=0.001$.

\paragraph{Automatic entropy tuning.}
We use the same technique as in~\citet{haarnoja2018applications} to weight the rewards against the policy entropy term.
Specifically, we pre-multiply the entropy term in
\begin{align*}
    \hat{V}(\bs', \bg)
    \approx
    Q_{\pqtarget}(\bs', \ba', \bg) - \log \pi(\ba' | \bs' ; \bg),
\end{align*}
by a parameter $\alpha$ that is updated to ensure that the policy entropy is above a minimum threshold.
The parameter $\alpha$ is updated by taking a gradient step on the following function with each batch of data:
\begin{align*}
    \calF_\alpha(\alpha)
    = - \alpha \left(
        \log \pi(\ba \vbar \bs, \bg)
        + \calH_\text{target}
    \right)
\end{align*}
and where $\calH_\text{target}$ is the target entropy of the policy.
We follow the procedure in~\citet{haarnoja2018applications} to choose $\calH_\text{target}$ and choose $\calH_\text{target} = -D_\text{action}$, where $D_\text{action}$ is the dimension of the action space.

\paragraph{Exploration policy.}
Because \odac is an off-policy algorithm, we are free to use any exploration
policy.
It may be beneficial to add
For the Ant and Sawyer tasks, we simply sample current policy.
For the 2D Navigation task, at each time step, the policy takes a random action with
probability 0.3 and repeats its

\paragraph{Evaluation policy.}
For evaluation, we use the mean of the learned policy for selecting actions.

\clearpage

\paragraph{Hyperparameters.}
\Cref{table:general-hyperparams} lists the hyperparameters that were shared across the experiments.
\Cref{table:env-hyperparams} lists hyper-parameters specific to each environment.

\setlength{\tabcolsep}{15.3pt}
\begin{table*}[h!]
    \centering
    \caption{Environment specific hyper-parameters.}
    \vspace*{-5pt}
    \begin{tabular}{l|c|c}
    \toprule
    Environment & horizon & $Q$-function and policy network sizes (hidden units)
    \\
    \hline
    Box 2D & 100 & [64, 64] \\
    Ant & 100 & [400, 300] \\
    Fetch Push & 50 & [64, 64] \\
    Sawyer Push & 100 & [400, 300] \\
    Sawyer Window & 100 & [400, 300] \\
    Sawyer Faucet & 100 & [400, 300]\\
    \hline
    \midrule
    \end{tabular}
\label{table:env-hyperparams}
\end{table*}

\setlength{\tabcolsep}{50.0pt}
\begin{table*}[h!]
    \centering
    \caption{General hyperparameters used for all experiments.}
    \vspace*{-5pt}
    \begin{tabular}{l|c}
    \toprule
    \textbf{Hyperparameter} & \textbf{Value}\\
    \hline
    \# training batches per environment step & $1$\\
    batch size & $256$\\
    discount Factor & $0.99$\\
    policy hidden activation & ReLU\\
    $Q$-function hidden activation & ReLU\\
    replay buffer size & $1$ million\\
    hindsight relabeling strategy & future\\
    hindsight relabeling probability & 80\%\\
    target network update speed $\tau$ & 0.001\\
    reward scale update speed $\lambda$ & 0.001\\
    \hline
    \midrule
    \end{tabular}
\label{table:general-hyperparams}
\end{table*}

\end{appendices}

\end{document}